\documentclass{article}

 \usepackage[preprint]{neurips_2026}

\usepackage[utf8]{inputenc} 
\usepackage[T1]{fontenc}    
\usepackage{hyperref}       
\usepackage{url}            
\usepackage{booktabs}       
\usepackage{amsfonts}       
\usepackage{nicefrac}       
\usepackage{microtype}      
\usepackage{xcolor}         

\usepackage{amsmath,amsfonts,bm}

\usepackage{enumitem}
\usepackage{graphicx}

\usepackage{amsthm} 

\theoremstyle{plain}
\newtheorem{theorem}{Theorem}[section]
\newtheorem{proposition}[theorem]{Proposition}
\newtheorem{lemma}[theorem]{Lemma}
\newtheorem{corollary}[theorem]{Corollary}
\theoremstyle{definition}
\newtheorem{definition}[theorem]{Definition}
\newtheorem{assumption}[theorem]{Assumption}
\theoremstyle{remark}
\newtheorem{remark}[theorem]{Remark}

\newcommand{\Law}{\operatorname{Law}}
\newcommand{\diag}{\text{diag}}
\newcommand{\norm}[1]{\|#1\|}

\newcommand{\inn}[2]{\left\langle#1, #2\right\rangle}

\newcommand{\Gaus}{\mathcal{N}}

\def\vtheta{{\bm{\theta}}}

\usepackage{xcolor}
\definecolor{purple}{RGB}{128, 0, 128}

\def\eqref#1{Eq.~(\ref{#1})}

\def\1{\bm{1}}

\def\vtheta{{\bm{\theta}}}
\def\va{{\bm{a}}}
\def\vb{{\bm{b}}}

\def\vg{{\bm{g}}}
\def\vh{{\bm{h}}}

\def\vv{{\bm{v}}}

\def\vx{{\bm{x}}}

\def\vz{{\bm{z}}}

\def\mG{{\bm{G}}}

\def\mI{{\bm{I}}}

\def\mK{{\bm{K}}}

\def\mP{{\bm{P}}}

\def\mU{{\bm{U}}}

\def\mW{{\bm{W}}}

\DeclareMathAlphabet{\mathsfit}{\encodingdefault}{\sfdefault}{m}{sl}
\SetMathAlphabet{\mathsfit}{bold}{\encodingdefault}{\sfdefault}{bx}{n}

\def\gB{{\mathcal{B}}}

\newcommand{\E}{\mathbb{E}}

\newcommand{\Cov}{\mathrm{Cov}}

\DeclareMathOperator{\Tr}{Tr}

\title{Feature Learning Dynamics in Infinite-Depth Neural Networks}

\author{%
  Zihan Yao \\
  School of Computing \\
  DePaul University \\
  \texttt{zyao8@depaul.edu}
  \And
  Ruoyu Wu \\
  Department of Mathematics \\
  Iowa State University \\
  \texttt{ruoyu@iastate.edu}
  \And
  Tianxiang Gao\thanks{Corresponding author.}\\
  School of Computing \\
  DePaul University \\
  \texttt{tgao9@depaul.edu}
}

\begin{document}

\maketitle

\begin{abstract}
Deep neural networks (DNNs) have achieved remarkable success in practice, yet a mechanistic understanding of how features evolve during training remains incomplete, especially in the large-depth limit. For ResNets under depth-$\mu$P scaling, prior work treats the layer index $\ell$ as a continuous time $t_\ell = \ ell/L$, yielding SDE descriptions of the training dynamics. A key unresolved issue is that backpropagation reuses each forward weight matrix $\mW_\ell$ through its transpose $\mW_\ell^\top$, creating correlations between forward features and backward gradients whose behavior and role in feature learning remain unclear. We study this reused-weight forward--backward coupling in one-layer ResNets under depth-$\mu$P scaling. Using conditional Gaussian representations, we explicitly separate the coupling terms induced by weight reuse from decoupled Gaussian fluctuations before taking any limit. At initialization, we prove that the coupling is a finite-width effect and vanishes at rate $O(n^{-1})$, uniformly over depth. During training, however, SGD induces a nontrivial forward--backward correlation term that survives the infinite-width limit. The key depth effect is that, under depth-$\mu$P scaling, this surviving term is higher order in depth and its accumulated contribution over layers becomes negligible as $L\to\infty$. This depth-induced suppression motivates \textit{neural feature dynamics} (NFD), a forward--backward SDE system with decoupled backward weights that retains the feature-gradient covariance structure generated during training. Under nondegeneracy assumptions, we prove that the finite-network training dynamics converge to NFD with an $O(L^{-1})$ depth-discretization error, while the reused-weight coupling term has a faster $O(L^{-2})$  decay. These results provide a rigorous infinite-depth limit for the feature-learning dynamics of one-layer ResNets under depth-$\mu$P.
\end{abstract}

\vspace{-0.5em}
\section{Introduction}
\vspace{-0.6em}
Deep neural networks (DNNs) have achieved remarkable success in practice across a wide range of domains \cite{achiam2023gpt,Dosovitskiy2021,LeCun2015}. Despite this success, a mechanistic understanding of how useful features are learned during training remains incomplete \cite{simon2026there}. Ideally, a theory of deep learning should not only establish optimization and generalization guarantees, but also characterize the training-time evolution of internal features, explain how their interaction and co-evolution with backward gradients shape feature learning, and determine how forward--backward coupling influences this learning process.

A fruitful approach toward such a theory is to study limiting training dynamics of simple DNNs optimized via stochastic gradient descent (SGD) in the nontrivial feature-learning regime \cite{yang2020feature,Mei2018,yang2024tensor,bordelon2024depthwise}. Under suitable large-network limits, this complex high-dimensional feature-gradient co-evolution may converge to tractable deterministic or stochastic dynamics. These limits are naturally structured along two axes: \textit{width} and \textit{depth}. Along the width direction, mean-field theory \cite{Mei2018} and the Tensor Program \cite{yang2020feature} establish principled frameworks for the infinite-width feature learning dynamics. In particular, maximal update parameterization ($\mu$P) ensures that feature updates remain $O(1)$ at infinite width, enabling hyperparameter (HP) transfer across width \cite{yang2021tuning}.

Along the depth direction, however, the corresponding characterization of feature-learning dynamics remains less complete. Many existing analysis studies this problem through one-layer residual networks \cite{he2016deep,yang2024tensor,hayou2023width,bordelon2024depthwise}. 
Recent work shows that scaling the residual branch by $1/\sqrt{L}$, where $L$ is the depth, is essential for stable large-depth training and enables depth-wise HP transfer; this scaling is known as depth-$\mu$P \cite{yang2024tensor,bordelon2024depthwise}. Under this regime, the layer index $\ell$ transforms into a continuous time variable $t_\ell=\ell/L\in[0,1]$, by which prior works rigorously prove that, in the infinite-width and infinite-depth limit, forward feature propagation at initialization converges to a mean-field stochastic differential equation (SDE) \cite{hayou2023width,peluchetti2020infinitely}. Dynamic mean-field theory (DMFT), as an insightful physics-inspired heuristic, later extends this SDE view to training time by formulating feature learning as self-consistent forward-backward stochastic dynamics coupled through correlation kernels and response functions \cite{bordelon2024depthwise}. These developments raise the following question:
\vspace{-0.3em}
\begin{quote}
    \textit{Why do we still lack rigorous optimization and generalization theories comparable to those in kernel regimes? What is the main obstacle?}
\end{quote}
\vspace{-0.3em}
From our perspective, a central bottleneck is the \textit{forward--backward coupling effect} introduced by backpropagation, which reuses the forward weights $\mW_{\ell}$ in the backward pass via their transpose $\mW_{\ell}^{\top}$. This coupling is already present at initialization and also appears when computing the NTK in the kernel regime \citep{yang2017mean,jacot2018neural}. Fortunately, the \textit{gradient-independence heuristic} becomes valid in the infinite-width limit \citep{yang2020ntk}, where the backward weights can be treated as decoupled from the forward weights. Together with lazy training \citep{Chizat2019}, where features are essentially fixed at initialization, this makes the coupling negligible throughout training and enables a closed kernel description. On the other hand, feature-learning regimes are more delicate: features and gradients co-evolve nontrivially, and SGD updates create new dependencies among weights, forward features, and backward gradients. Hence, such training-time correlation effects can persist in the infinite-width limit \citep{yang2020feature}. However, their behavior in the large-depth limit remains much less understood. Although DMFT \cite{bordelon2024depthwise} encodes related dependencies through response functions, it does not directly characterize these coupling effects and their depth scaling. Hence, understanding how this coupling behaves during training is a necessary step toward a rigorous theory of infinite-depth feature-learning dynamics.

To address this bottleneck, we characterize the reused-weight forward--backward coupling and its asymptotic behavior in the infinite-width and infinite-depth limits. Our first step is to make this coupling explicit before taking any limit. In particular, we use conditional Gaussian representations for the forward and backward increments, thereby separating reused-weight correlations from decoupled Gaussian fluctuations. At initialization, we provide a quantitative refinement of the \textit{gradient independence assumption} (GIA) \citep{yang2017mean} by proving a finite-width decay rate for the coupling effect, uniformly over depth. During training, the same Gaussian representation explicitly shows how SGD updates induce a nontrivial forward--backward correlation term that does not vanish in the large-width limit. One of our key findings is that this surviving term is higher order in depth and becomes negligible as the depth increases. This motivates \textit{neural feature dynamics (NFD)} as a rigorous limiting dynamics for describing infinite-depth feature learning. Our contributions are summarized as follows:

\vspace{-0.4em}
\begin{itemize}[leftmargin=*]
    \item \textbf{Gaussian characterization of reused-weight coupling.}
    We derive conditional Gaussian representations that explicitly separate the reused-weight forward--backward coupling from the decoupled Gaussian fluctuations. This identifies the precise obstruction to a GIA-style approximation in finite-width, finite-depth networks, both at initialization and during training.

\vspace{-0.1em}
    \item \textbf{Initialization-time GIA with depth-uniform rate.}
    At initialization, we prove that the true backward process and its decoupled auxiliary process differ by $O(1/n)$ in coordinate-wise mean square, uniformly over depth. Consequently, the joint forward--backward propagation admits a forward--backward SDE limit driven by independent forward and backward Brownian motions.

\vspace{-0.1em}
    \item \textbf{Training-time coupling and depth suppression.}
    We show that after SGD updates, a forward--backward correlation term survives in the infinite-width limit, so width alone does not justify a GIA-style description during training. Under depth-$\mu$P scaling, however, this term is higher order in depth. In particular, its per-layer contribution is $O(L^{-2})$, while the main SGD drift is $O(L^{-1})$. Hence, its accumulated effect becomes negligible as $L\to\infty$.

    \item \textbf{Neural feature dynamics.}
    Motivated by this depth-induced suppression, we define NFD as a candidate dynamics for infinite-depth feature learning, using decoupled backward weights while retaining the nontrivial feature and gradient covariance kernels generated during training. Under nondegeneracy assumptions, we prove convergence of the finite-network training dynamics to NFD with an $O(L^{-1})$ depth-discretization error, while the omitted coupling term has a faster $O(L^{-2})$ decay, providing a rigorous feature-learning dynamics for one-layer ResNets in large-depth limit.
\end{itemize}

\section{Related Work}\label{app sec: related works}
\vspace{-0.5em}
\paragraph{Width limits and feature learning.}
An important line of theoretical work studies the large-width limit of neural networks. In the Neural Tangent Kernel (NTK) regime \citep{jacot2018neural,lee2019wide,arora2019fine}, training is well approximated by kernel regression with an essentially fixed kernel, leading to strong optimization and generalization guarantees for overparameterized networks \citep{allen2019convergence,du2019gradient,zou2020gradient,gao2022a,gao2025global}. However, this lazy-training regime \citep{Chizat2019} does not capture substantial feature evolution. Mean-field parameterizations provide an infinite-width feature-learning alternative \citep{Mei2018,ChizatBach2018,SirignanoSpiliopoulos2018}, with convergence guarantees for shallow networks \citep{nitanda2022convex,pham2021global}, but standard mean-field scalings become degenerate in deep architectures, leading to vanishing signals or gradients \citep{fang2021modeling,nguyen2023rigorous}. Tensor Programs and $\mu$P \citep{yang2019wide,yang2020ntk,yang2020laws,yang2020feature,yang2021tuning} clarify which width scalings yield nontrivial feature learning and enable HP transfer.

\paragraph{Depth limits and training dynamics.}
Large-depth analyses study how signals and gradients behave as depth increases, often in conjunction with width. Classical signal-propagation work identified stability conditions such as the edge of chaos \citep{poole2016exponential,schoenholz2017deep}. For ResNets, scaling the residual branch by $1/\sqrt{L}$ restores stable propagation and yields an SDE view of forward feature propagation at initialization in joint width-depth limits \citep{hayou2023width}; related large-depth stability studies also appear in implicit models and Neural ODEs \citep{gao2023wide,gao2022a,gao2024mastering}. More recent work shows that depth-$\mu$P scaling enables depthwise HP transfer \citep{yang2024tensor,bordelon2024depthwise} and that feature learning can persist in infinite-depth networks \citep{yang2024tensor}. These results motivate an SDE view of training dynamics in the large-depth limit, but they do not explicitly characterize the forward--backward correlations created by weight reuse.


\paragraph{Relation to prior work.}
Our work focuses on this missing forward--backward coupling, which arises because the forward update uses $\mW_\ell$ while backpropagation uses $\mW_\ell^\top$. We rigorously characterize this coupling through conditional Gaussian representations and show how its role changes from initialization to training under depth-$\mu$P scaling, thereby motivating neural feature dynamics (NFDs) as a limiting description of one-layer ResNet training dynamics.
\section{Preliminaries}\label{sec:preliminary}
\paragraph{Model and Training Setup.}
In this paper, we consider a simple one-layer ResNet under depth-$\mu$P scaling \cite{yang2024tensor}, which maps an input $\vx\in\mathbb{R}^d$ through the scaled residual stream defined as follows
\begin{align}
    f(\vx;\vtheta) 
    &= \frac{1}{n} \, \vv^{\top} \vh_{L}, 
    \qquad
    \vh_{\ell} 
    = \vh_{\ell-1} 
    + \frac{1}{\sqrt{L n}} \, \mW_{\ell} \, \phi(\vh_{\ell-1}), 
    \qquad
    \vh_0 
    = \frac{1}{\sqrt{d}} \, \mU \vx, 
    \quad \forall\ell \in [L],
    \label{eq:resnet}
\end{align}
where $\phi$ is an activation function, and $\mU \in \mathbb{R}^{n\times d}$, $\mW_{\ell} \in \mathbb{R}^{n\times n}$, and $\vv \in \mathbb{R}^{n}$ are trainable parameters. We denote the collection of parameters by $\vtheta := (\vv,\{\mW_{\ell}\}_{\ell=1}^{L},\mU)$ and initialize all entries independently $\vtheta_i \sim \mathcal{N}(0,1)$. Here, the $1/\sqrt{n}$ width scaling and $1/\sqrt{L}$ depth scaling on the residual branches follow directly from the standard $\mu$P \citep{yang2020feature} and depth-$\mu$P \citep{yang2024tensor}, respectively.

Given a loss $\mathcal{L}$, we train the ResNet with stream SGD using one sample $(\vx^{(k)},y^{(k)})$ per iteration:
\begin{align}
    \vtheta^{(k+1)}
    = \vtheta^{(k)}
    - \eta \nabla_{\vtheta} \mathcal{L}(f^{(k)},y^{(k)})
    = \vtheta^{(k)}
    - \eta \chi^{(k)} \nabla_{\vtheta^{(k)}} f^{(k)},
\end{align}
where $f^{(k)} := f(\vx^{(k)};\vtheta^{(k)})$ and 
$\chi^{(k)} := \partial_f \mathcal{L}(f^{(k)},y^{(k)})$. 
To obtain nontrivial feature updates in the infinite-width and infinite-depth limits, following depth-$\mu$P, we use the learning rate $\eta=\eta_c n$, where $\eta_c>0$ is the fixed effective learning rate, independent of the network width and depth.

\paragraph{Pre- vs. Post-activation ResNets.}
While the original ResNet \citep{he2016deep} used the post-activation design \eqref{eq:post-act-style}, modern architectures commonly adopt the pre-activation form \eqref{eq:resnet} \citep{he2016identity}, including Transformers \citep{radford2019language,brown2020language}. Beyond empirical practice, the following proposition shows that post-activation ResNets can exhibit hidden-state divergence at large depth; see Appendix~\ref{app:post-act-style} for the proof.

\begin{proposition}
\label{prop:post-act}
Let $\phi$ satisfy the following \textbf{positive dominance} condition: there exist nonnegative constants $c_1,c_2$, not both zero, such that $\E[\phi(wx)] \geq c_1 |x| + c_2$, $\forall x \in \mathbb{R}$, where $w\sim\mathcal{N}(0,1)$. Then, in the post-activation ResNet \eqref{eq:post-act-style}, the expected hidden state obeys, for every $i\in[n]$,
\begin{equation*}
    \E[\vh_{L,i}]
    \geq
    c_1 \left(1+\frac{c_1}{\sqrt{Ln}}\right)^L
    \frac{\norm{\vx}}{\sqrt{d}}
    +
    c_2\sqrt{L}.
\end{equation*}
\end{proposition}
\vspace{-0.5em}

For activations such as ReLU, where $c_1=1/\sqrt{2\pi}$ and $c_2=0$, Proposition~\ref{prop:post-act} shows that the hidden states $\vh_\ell$ may diverge as depth grows, even under the residual scaling $1/\sqrt{L}$. Figure~\ref{fig:pre-act-vs-post-act} illustrates this behavior. The superior stability of the pre-activation design justifies focusing our subsequent analysis on \eqref{eq:resnet}; it also avoids additional normalization or carefully balanced depth-width scaling assumptions often needed in analyses of post-activation ResNets \citep{peluchetti2020infinitely,yang2024tensor,li2022neural}.
\section{Forward-Backward Coupling at Initialization}
\label{sec:at-initial}

\begin{figure*}[t]
    \centering
    \begin{minipage}{0.24\textwidth}
        \centering
        \includegraphics[width=\linewidth]{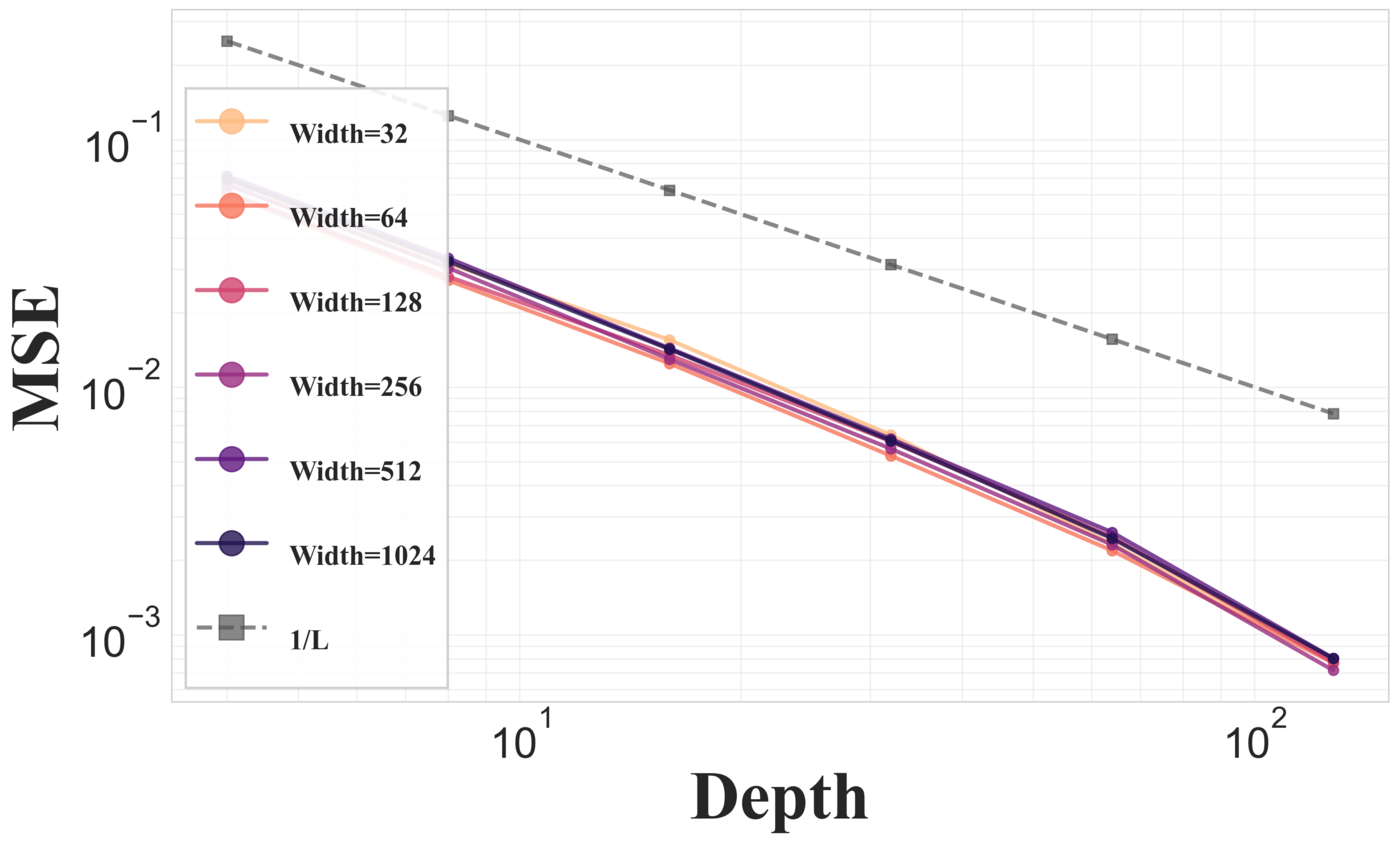}
    \end{minipage}
    \begin{minipage}{0.24\textwidth}
        \centering
        \includegraphics[width=\linewidth]{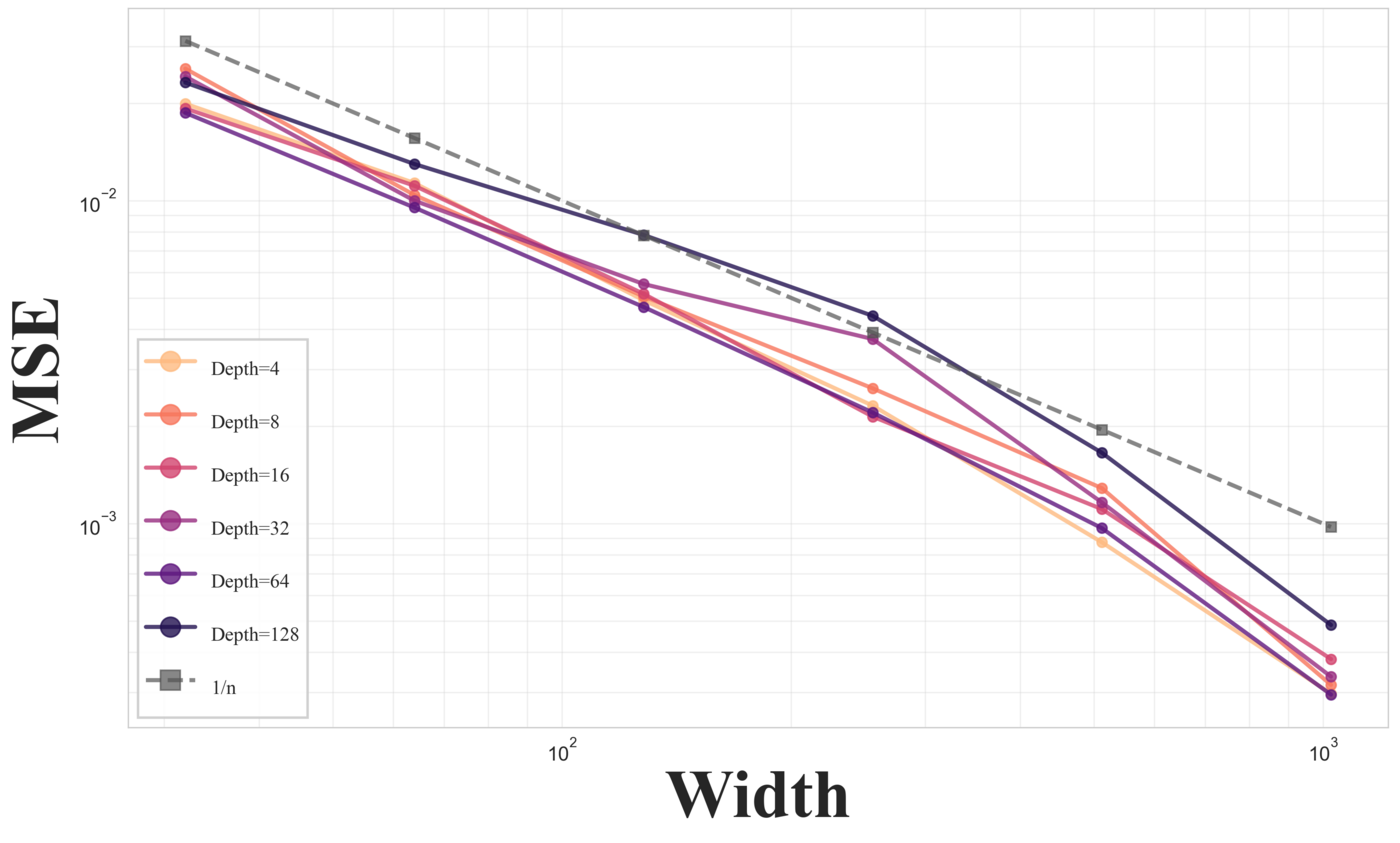}
    \end{minipage}
    \begin{minipage}{0.24\textwidth}
        \centering
        \includegraphics[width=\linewidth]{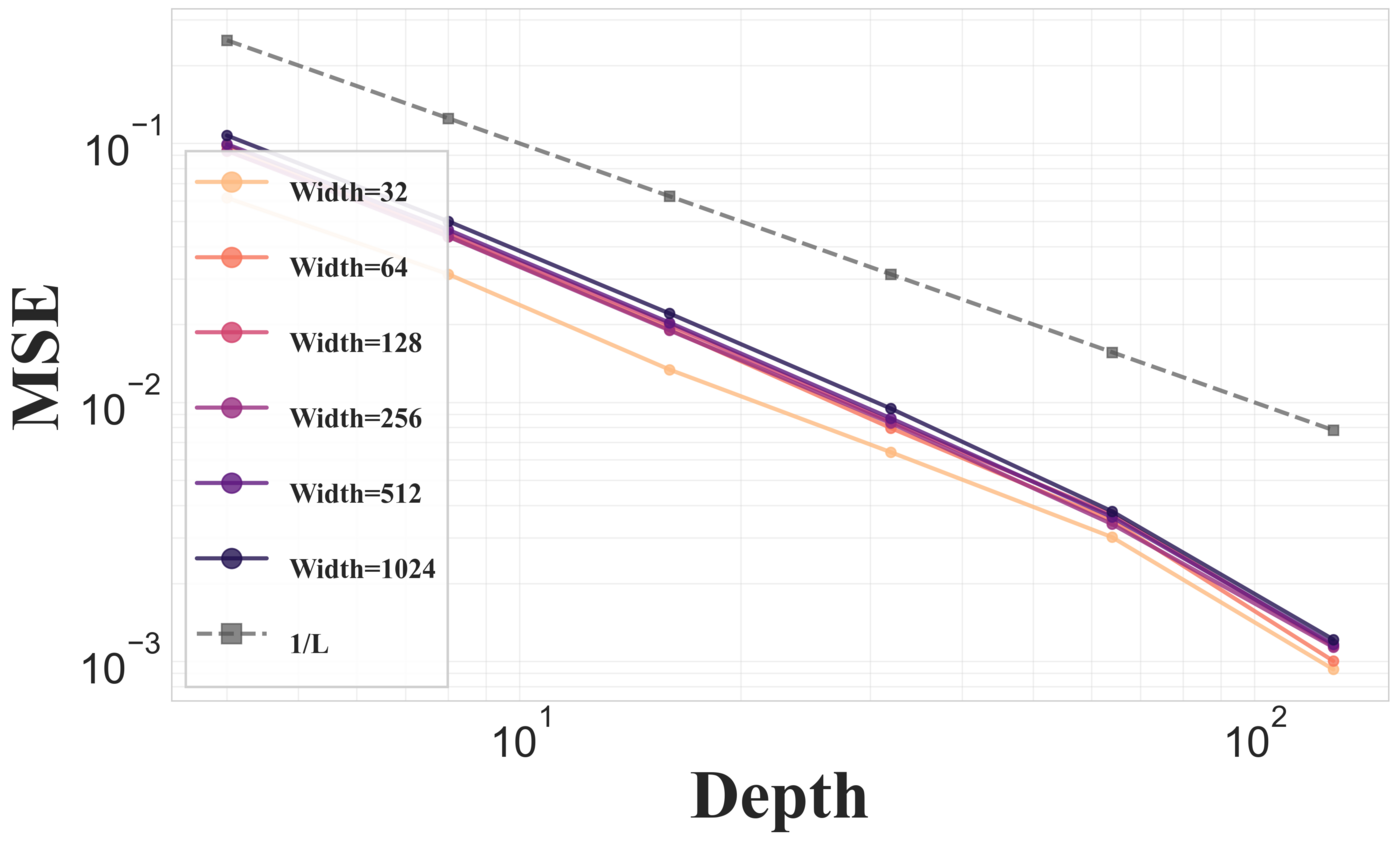}
    \end{minipage}
    \begin{minipage}{0.24\textwidth}
        \centering
        \includegraphics[width=\linewidth]{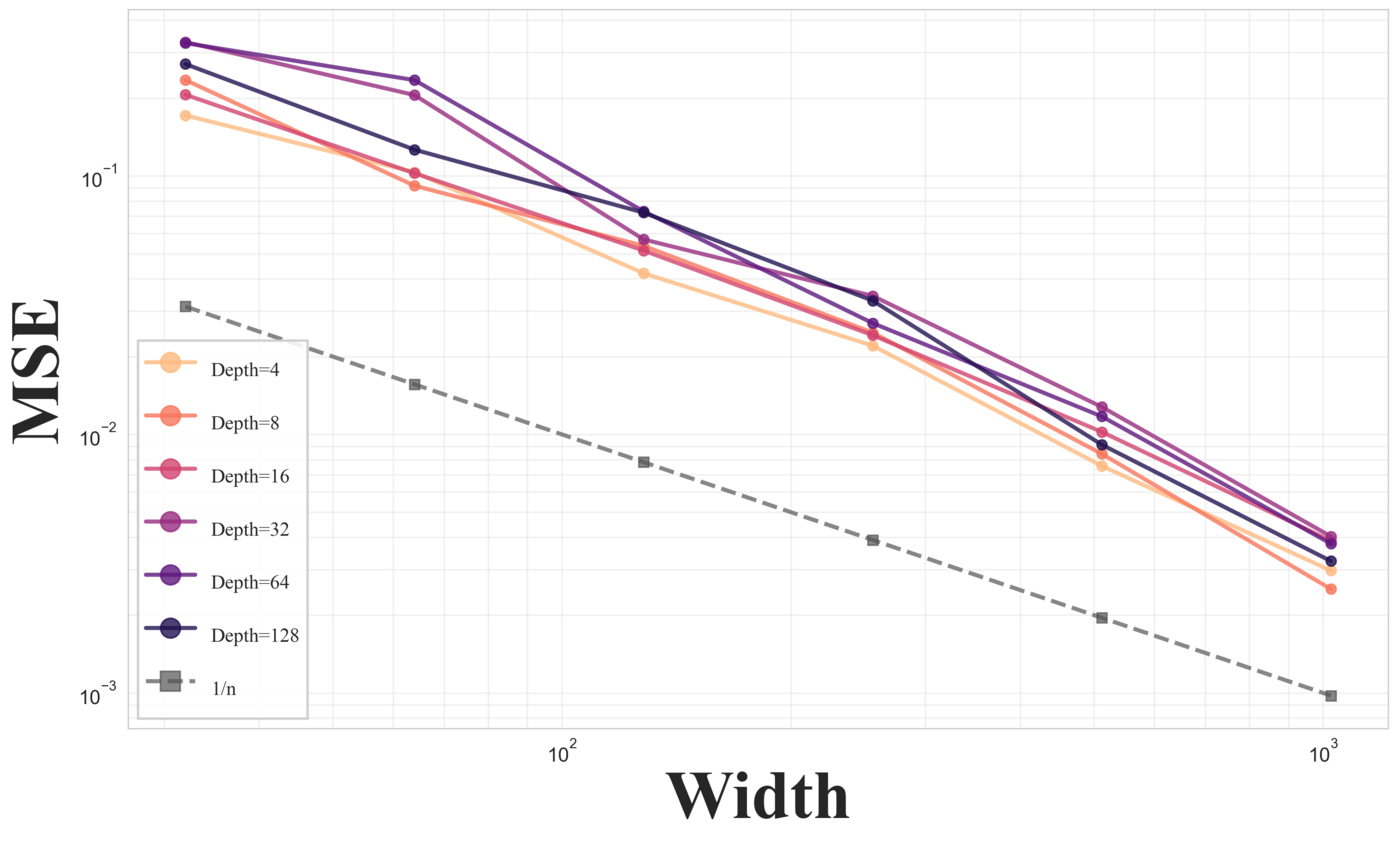}
    \end{minipage}
    \caption{\textbf{Convergence to NFD at initialization and after training.}
    Depth-$\mu$P ResNets are evaluated on CIFAR-10 with SGD. The first two panels show convergence with depth and width at initialization, and the last two after 30 training epochs. Across settings, the error decays consistently with $\mathcal{O}(1/L+1/n)$, supporting the commutativity of the width and depth limits.}
    \label{fig:depth-width-convergence}
\end{figure*}

We first analyze the forward--backward coupling at initialization. The goal is to show that the reused-weight correlation in backward propagation is a finite-width effect, so the GIA holds asymptotically as $n\to\infty$, with coordinate-wise mean-square error of order $1/n$, uniformly over depth.

\paragraph{Backward Gradient Propagation.}
Analogous to the forward feature propagation represented by the hidden states $\{\vh_\ell\}_{\ell=0}^{L}$, the backward gradient propagation can be represented by the recursion
\begin{equation}
    \vg_{L}^{(0)}= \vv^{(0)},
    \qquad
    \vg_{\ell-1}^{(0)}
    =
    \vg_{\ell}^{(0)}
    +
    \frac{1}{\sqrt{Ln}}\,
    \phi^{\prime}(\vh_{\ell-1}^{(0)}) \odot (\mW_{\ell}^{(0)})^{\top} \vg_{\ell}^{(0)},
    \quad \ell\in [L],
    \label{eq:grad-prop}
\end{equation}
where $\odot$ denotes element-wise multiplication. The key challenge is that backward recursion reuses the same matrix $\mW_\ell^{(0)}$ that appears in the forward pass, through its transpose $(\mW_\ell^{(0)})^\top$. This creates a $(\mW_\ell,\mW_\ell^\top)$ correlation between the forward features and backward gradients. To clearly isolate this effect, we compare the true backward process with an auxiliary backward process $\{\bar \vg_\ell^{(0)}\}_{\ell=0}^{L}$ obtained by replacing $(\mW_\ell^{(0)})^\top$ in \eqref{eq:grad-prop} with an independent copy $(\widetilde{\mW}_\ell^{(0)})^\top$. This auxiliary construction formalizes \textit{the gradient independence assumption (GIA)} \cite{yang2017mean,yang2020ntk}, a common heuristic for computing neural-network limits at initialization.

\paragraph{Gaussian Representation.}
The forward and backward recursions can be written in \textit{increment} form
\begin{align}
    \vh_{\ell}^{(0)}
    &=
    \vh_{\ell-1}^{(0)}
    +
    \frac{1}{\sqrt{L}}\,\va_{\ell}^{(0)},
    \qquad
    \vg_{\ell-1}^{(0)}
    =
    \vg_{\ell}^{(0)}
    +
    \frac{1}{\sqrt{L}}\,\vb_{\ell}^{(0)}\odot \dot\vx_{\ell-1}^{(0)},
    \label{eq:forward-backward-increment-form}
\end{align}
where $\vx_{\ell}^{(0)}:=\phi(\vh_{\ell}^{(0)})$, 
$\dot \vx_{\ell}^{(0)}:=\phi'(\vh_{\ell}^{(0)})$, and the innovative increments are given by
\begin{align}
    \va_{\ell}^{(0)}
    :=
    \frac{1}{\sqrt n}\mW_{\ell}^{(0)}\vx_{\ell-1}^{(0)},
    \qquad
    \vb_{\ell}^{(0)}
    :=
    \frac{1}{\sqrt n}(\mW_{\ell}^{(0)})^\top \vg_{\ell}^{(0)}.
    \label{eq:init-forward-backward-increments}
\end{align}
Since the initial weights $\mW_\ell^{(0)}$ have i.i.d. Gaussian entries, the increments $\va_\ell^{(0)}$ and $\vb_\ell^{(0)}$ admit a conditional Gaussian representation. This representation explicitly reveals the finite-width forward--backward correlation induced by reusing the same weights in the forward and backward passes. The following proposition states this representation; its proof is provided in Appendix~\ref{app:gaussian-representation}.

\begin{proposition}[Gaussian representation]
\label{prop:layerwise-gaussian-representation}
For each $\ell\in[L]$ with $\vx_{\ell-1}^{(0)}\neq 0$, define
\begin{align}
    \mP_{x,\ell-1}^{(0)}
    :=
    \frac{\vx_{\ell-1}^{(0)}(\vx_{\ell-1}^{(0)})^{\top}}
    {\|\vx_{\ell-1}^{(0)}\|^2}.
\end{align}
Then, on a possibly enlarged probability space, there exist standard Gaussian innovations
$\vz_\ell^{(0)},\widetilde{\vz}_\ell^{(0)}\sim\mathcal N(0,\mI_n)$, mutually independent across layers and between the two families, such that
\begin{align}
    \va_{\ell}^{(0)}
    \overset{d}{=}
    \frac{\|\vx_{\ell-1}^{(0)}\|}{\sqrt n}\,\vz_{\ell}^{(0)}, 
    \qquad
    \vb_{\ell}^{(0)}
    \overset{d}{=}
    \frac{\langle\vz_{\ell}^{(0)},\vg_{\ell}^{(0)}\rangle}{\sqrt n}\,
    \frac{\vx_{\ell-1}^{(0)}}{\|\vx_{\ell-1}^{(0)}\|}
    +
    \frac{\|\vg_{\ell}^{(0)}\|}{\sqrt n}\,
    (\mI-\mP_{x,\ell-1}^{(0)})\widetilde{\vz}^{(0)}_{\ell}.
    \label{eq:a-b-ell-gaussian-rep}
\end{align}
By contrast, the auxiliary backward process with decoupled backward weights has the increment
\begin{align}
    \bar \vb_{\ell}^{(0)}
    \overset{d}{=}
    \frac{\|\bar\vg_{\ell}^{(0)}\|}{\sqrt n}\,
    \widetilde{\vz}_{\ell}^{(0)}.
    \label{eq:gbar-rewrite}
\end{align}
If $\vx_{\ell-1}^{(0)}=0$, we set $\mP_{x,\ell-1}^{(0)}=0$, and the true backward Gaussian increment reduces in distribution to the decoupled Gaussian form without the corresponding correlation term.
\end{proposition}

The Gaussian representation in Proposition~\ref{prop:layerwise-gaussian-representation} explicitly exhibit the precise finite-width obstruction to GIA at initialization. The true backward increment $\vb_{\ell}^{(0)}$ in \eqref{eq:a-b-ell-gaussian-rep} contains an aligned component along the forward feature direction $\vx_{\ell-1}^{(0)}$ and a residual Gaussian component $(\mI-\mP_{x,\ell-1}^{(0)})\widetilde{\vz}_{\ell}^{(0)}$ in the orthogonal complement. In contrast, the auxiliary increment $\bar \vb_{\ell}^{(0)}$ in \eqref{eq:gbar-rewrite}, with decoupled backward weights, contains no aligned correlation term and is driven by a fresh isotropic Gaussian innovation.
\vspace{-0.3em}
\paragraph{Asymptotic GIA at Initialization.}
The next result shows that this coupling effect at initialization is only finite-width and vanishes as $n$ grows. Intuitively, the backward recursion starts from a fresh Gaussian vector $\vg_L^{(0)}=\vv^{(0)}$ that is independent of the forward features, so the dependence induced by reusing $(\mW_\ell^{(0)})^\top$ remains weak at large width. Although the projection $\mP_{x,\ell-1}^{(0)}$ is visible at finite width, its effect is confined to a one-dimensional direction and becomes negligible at the coordinate level as $n$ grows. The proposition below makes this precise by showing that, uniformly over depth, each coordinate of the true backward process converges in law to its auxiliary GIA counterpart. The proof is provided in Appendix~\ref{app:init-convergence-proofs}.

\begin{proposition}[GIA at initialization]
\label{prop:init-backward-to-auxiliary}
Suppose $\phi'$ is Lipschitz continuous. Then, for every $i\in[n]$,
\begin{align}
    \sup_{\ell\in[L]} 
    W_2\!\left(
        \Law(\vg_{\ell,i}^{(0)}),
        \Law(\bar \vg_{\ell,i}^{(0)})
    \right)^2
    \le C n^{-1},
\end{align}
where $W_2$ is the Wasserstein distance and the constant $C>0$ does not depend on $n$ or $L$.
\end{proposition}

Importantly, GIA removes only the correlation induced by reusing $(\mW_{\ell}^{(0)})^\top$ in the backward pass. It does not remove the structural dependence of the backward $\vg_{\ell}^{(0)}$ on the forward features through the Jacobian factor $\phi'(\vh_{\ell-1}^{(0)})$. Since the GIA error is uniform over depth, we can next take the large-depth limit and obtain an initialization-time forward-backward SDE system, where the forward and backward noises are independent. The proof is included in Appendix~\ref{app:init-convergence-proofs}.

\begin{theorem}[Forward--backward SDE limit at initialization]
\label{thm:init-forward-backward-sde}
Suppose $\phi$ and $\phi'$ are Lipschitz continuous. Let $t_\ell := \ell/L$. 
As $\min(n,L)\to\infty$, each coordinate of $(\vh_\ell^{(0)},\vg_\ell^{(0)})$ converges in law to the limiting forward--backward process $(H_t,G_t)$ satisfying
\begin{align}
    dH_t = \sigma_t \, dW_t,\quad H_0 \sim \mathcal{N}(0,\|\vx\|^2/d),
    \qquad
    dG_t = \widetilde{\sigma}_t(H_t)\, d\widetilde W_t,\quad G_1 \sim \mathcal{N}(0,1),
    \label{eq:forward-backward-sde-init}
\end{align}
where $G_t$ evolves backward in depth direction, $W_t$ and $\widetilde W_t$ are independent standard Brownian motions, $\sigma_t^2 := \E[\phi(H_t)^2]$, and $\widetilde{\sigma}_t(H_t)^2 := \phi'(H_t)^2 \E[G_t^2]$. Moreover, for every $\ell\in [L]$ and $i\in[n]$,
\begin{align}
    W_2\!\left(\Law(\vh_{\ell,i}^{(0)}),\Law(H_{t_{\ell}})\right)^2
    +
    W_2\!\left(\Law(\vg_{\ell,i}^{(0)}),\Law(G_{t_{\ell}})\right)^2
    \le C\left(n^{-1} + L^{-1}\right),
\end{align}
where $C>0$ does not depend on $n$ or $L$.
\end{theorem}

Together with Proposition~\ref{prop:init-backward-to-auxiliary}, Theorem~\ref{thm:init-forward-backward-sde} gives a quantitative initialization-time picture. The reused-weight correlation in backward propagation is a finite-width effect and vanishes at a rate $O(n^{-1})$, uniformly over depth; meanwhile, the finite-network propagation converges to its forward--backward SDE limit with error $O(n^{-1}+L^{-1})$. In this sense, the width and depth limits commute at initialization. Figure~\ref{fig:depth-width-convergence} empirically supports this theoretical prediction. The next section examines whether an analogous decoupling persists during training, where the weights are no longer fresh Gaussian matrices but have been updated by SGD.

\section{Training-Time Coupling and Neural Feature Dynamics}
\label{sec:training-time-coupling}

\begin{figure*}[t]
    \centering

    \begin{minipage}[t]{0.32\textwidth}
        \centering
        \includegraphics[width=0.9\linewidth]{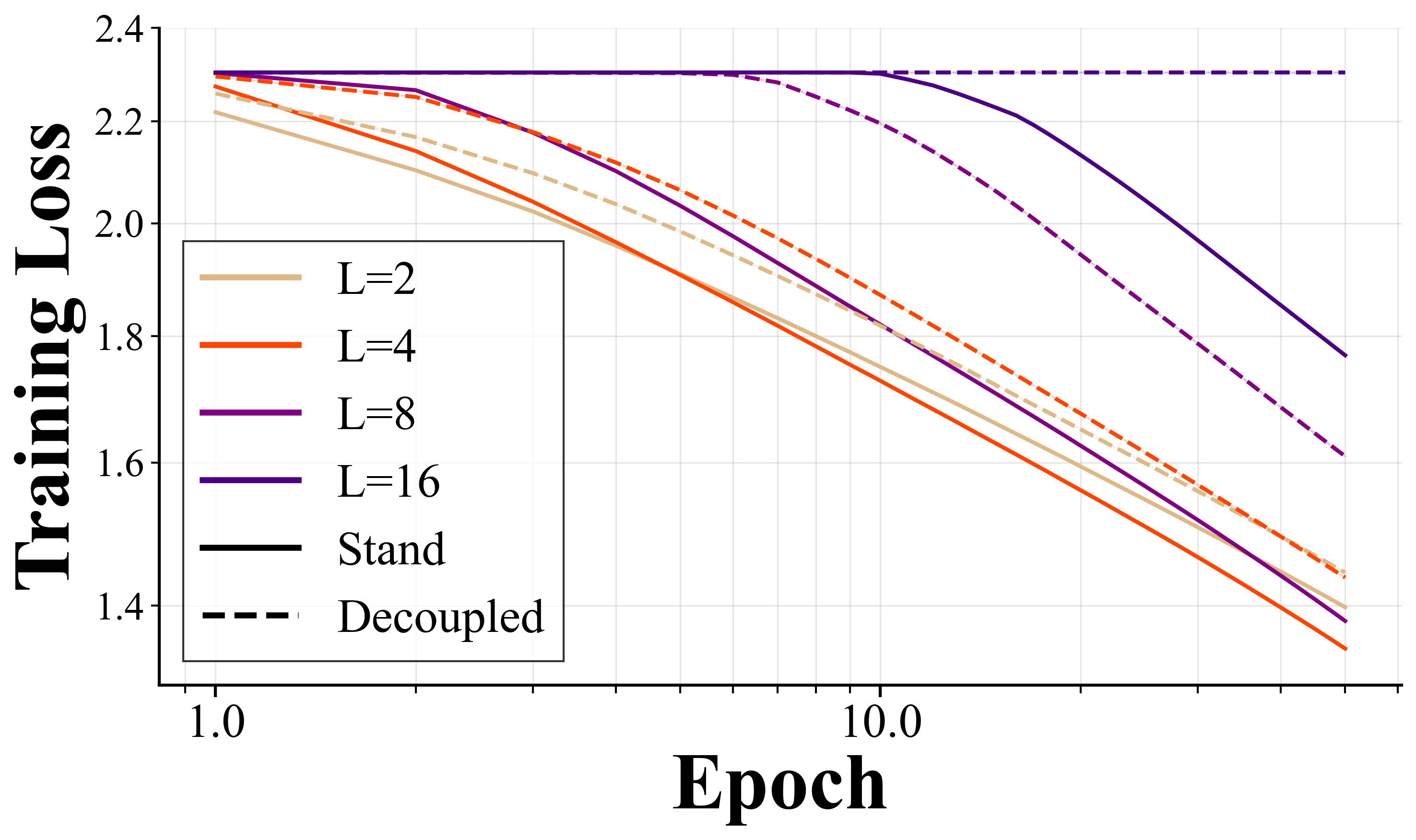}

        \vspace{0.3em}

        \includegraphics[width=0.9\linewidth]{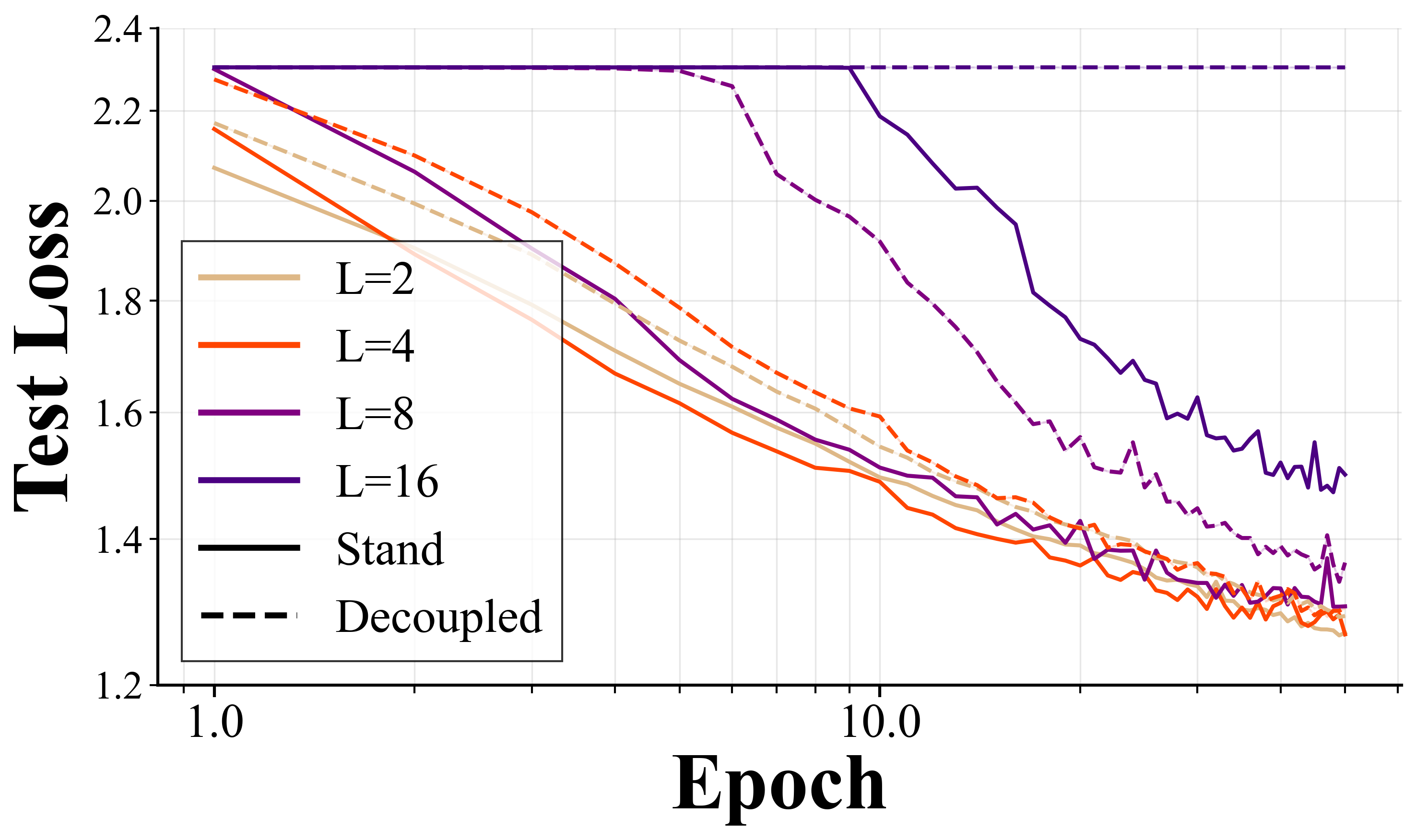}

        \vspace{0.3em}
        \textbf{(a) Vanilla DNN}
    \end{minipage}
    \hfill
    \begin{minipage}[t]{0.32\textwidth}
        \centering
        \includegraphics[width=0.9\linewidth]{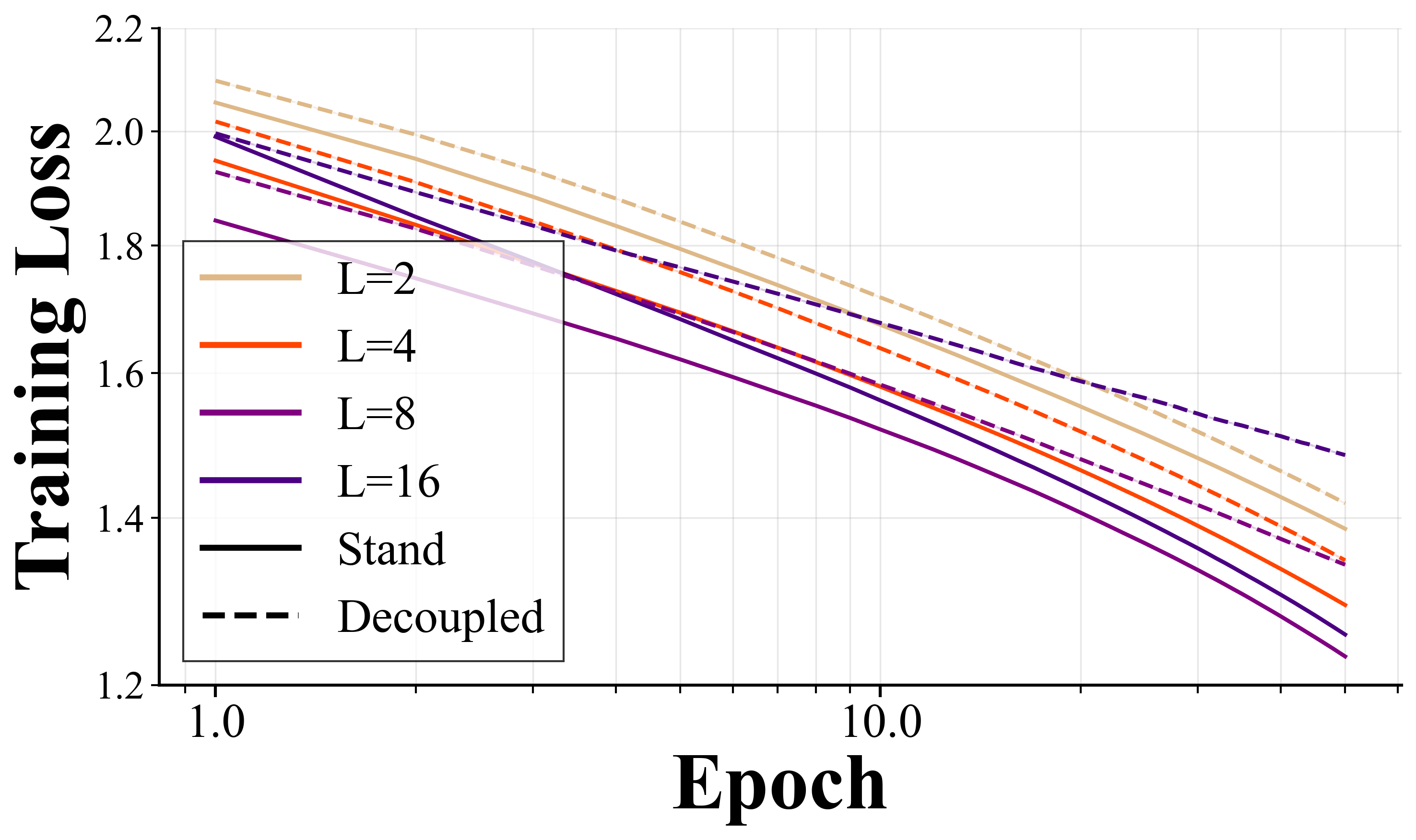}

        \vspace{0.3em}

        \includegraphics[width=0.9\linewidth]{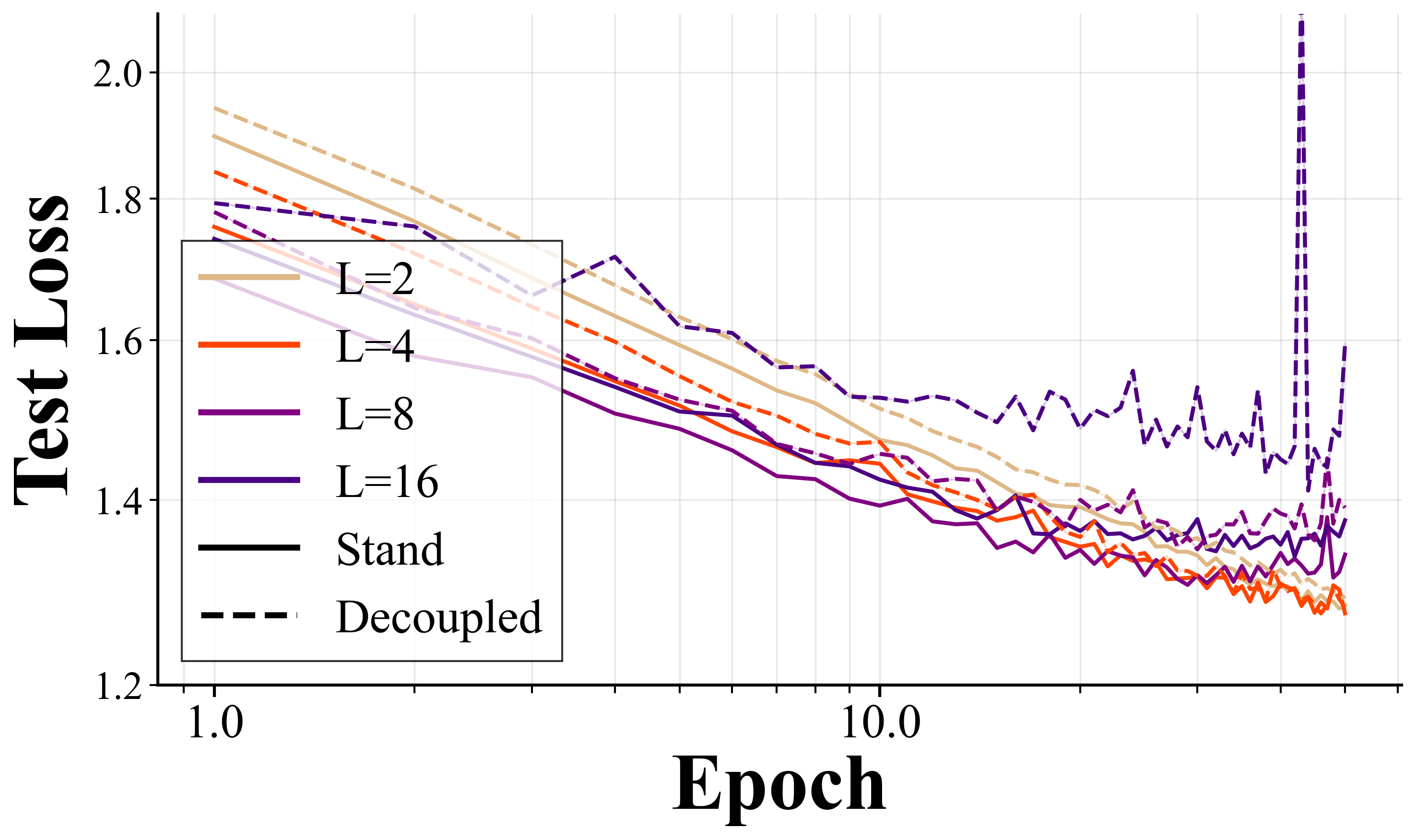}

        \vspace{0.3em}
        \textbf{(b) ResNet under $\mu$P}
    \end{minipage}
    \hfill
    \begin{minipage}[t]{0.32\textwidth}
        \centering
        \includegraphics[width=0.9\linewidth]{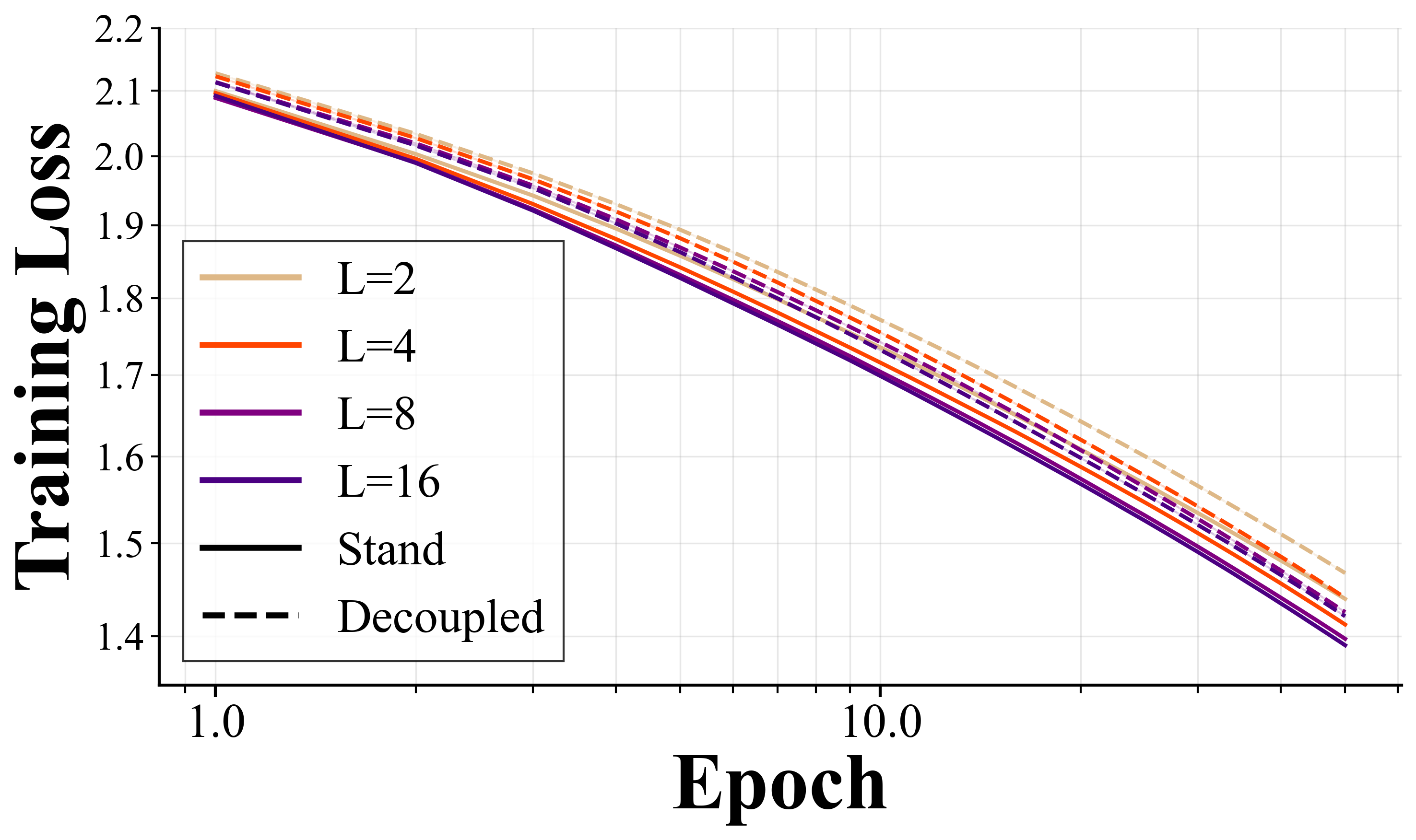}

        \vspace{0.3em}

        \includegraphics[width=0.9\linewidth]{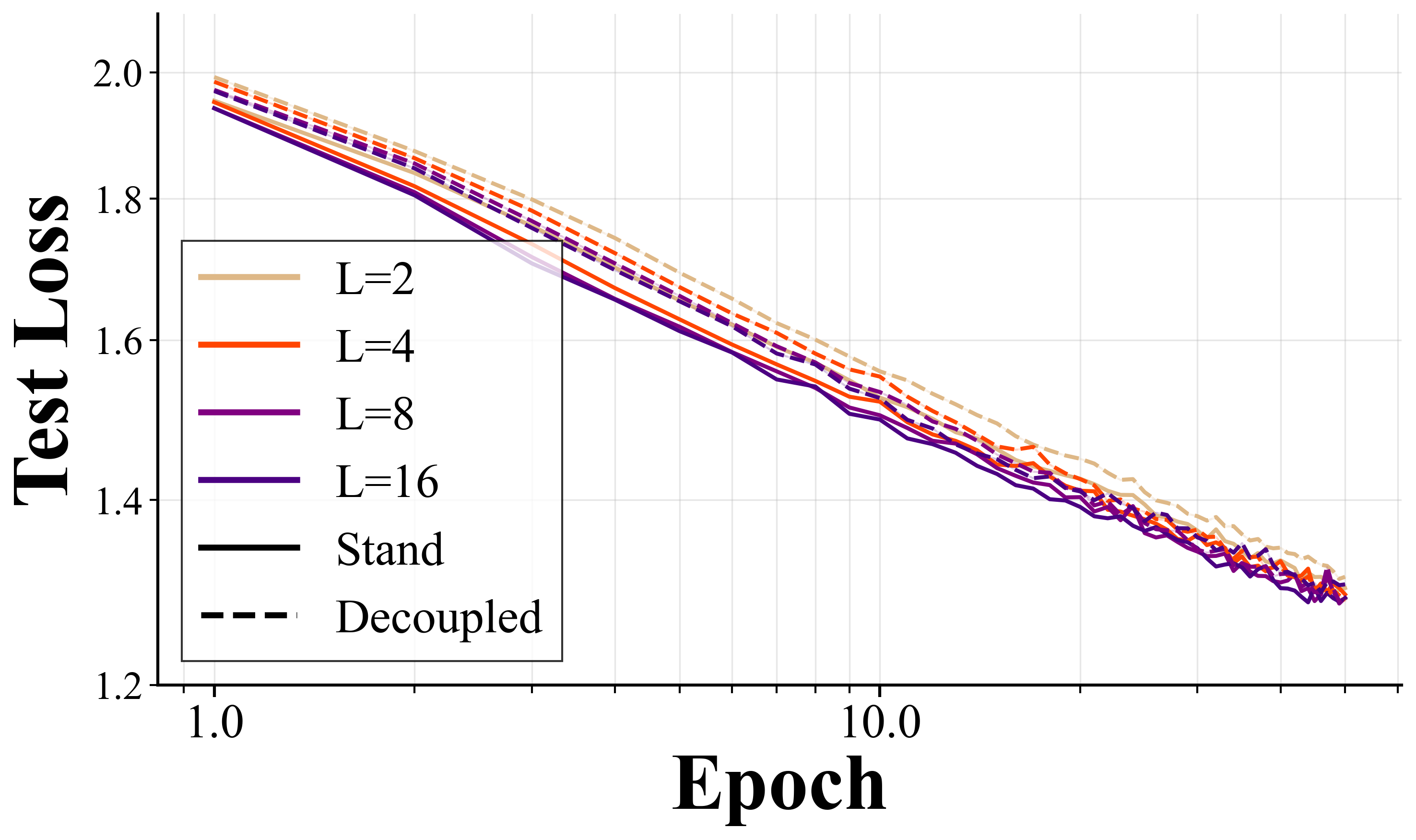}

        \vspace{0.3em}
        \textbf{(c) ResNet under depth-$\mu$P}
    \end{minipage}

    \caption{
    \textbf{Empirical evaluation of GIA restoration.}
    We train all models with width~128 on CIFAR-10 using SGD, comparing
    \textbf{\emph{standard}} training, where the same initialized weights are used
    in both the forward and backward passes, with a \textbf{\emph{decoupled}} setup,
    where the backward pass starts from an independent i.i.d.\ copy of the initial
    forward weights while using the same subsequent SGD update increments. Both
    training and test losses are shown on log--log scales. As depth increases:
    \textbf{(a)} Vanilla DNNs suffer from vanishing gradients, make little progress,
    and standard and decoupled trajectories remain misaligned.
    \textbf{(b)} ResNets under $\mu$P also show poor alignment and eventually
    overfit, e.g., at $L=16$.
    \textbf{(c)} ResNets under depth-$\mu$P improve both training and test
    performance with depth, and the standard and decoupled trajectories align
    closely even at moderate depths. This agrees with
    Corollary~\ref{cor:nfd-depth-rate}: the reused-weight coupling terms decay at
    the faster $O(L^{-2})$ second-moment rate, so the coupled and decoupled
    dynamics become nearly indistinguishable without requiring very large depth.
    Appendix~\ref{app:experiments} shows that larger widths reduce instability in
    (a,b) but do not eliminate their pathologies, while further reinforcing the
    alignment and performance gains of depth-$\mu$P.
    }

    \label{fig:gia_alignment}
\end{figure*}

Section~\ref{sec:at-initial} shows that, at initialization, reused-weight correlation vanishes as the width increases. The training-time setting is more subtle: after even one update, the next forward pass uses weights that have been updated
using the previous backward gradients, creating new dependencies between the weights, features, and gradients. We characterize this training-induced coupling and show that depth-$\mu$P scaling suppresses it as a higher-order effect in depth, motivating neural feature dynamics (NFD) as a decoupled forward--backward SDE description of infinite-depth feature learning.

\paragraph{Gaussian Representation After One Update.}
After one update, the forward feature $\vh_{\ell}^{(1)}$ satisfies
\begin{align}
    \vh_{\ell}^{(1)}
    =
    \vh_{\ell-1}^{(1)}
    +
    \frac{1}{\sqrt{Ln}}\mW_{\ell}^{(0)}\vx_{\ell-1}^{(1)}
    -
    \frac{\eta_c}{L}\chi^{(0)}
    \frac{\langle \vx_{\ell-1}^{(0)},\vx_{\ell-1}^{(1)}\rangle}{n}
    \vg_{\ell}^{(0)} ,
    \label{eq:one-step-forward-feature}
\end{align}
To expose the training-time coupling, we first represent the second forward innovative increment
\begin{align}
    \va_{\ell}^{(1)}
    :=
    \frac{1}{\sqrt n}\mW_{\ell}^{(0)}\vx_{\ell-1}^{(1)}
\end{align}  
by conditioning on the first forward and backward passes. The proof is included in Appendix~\ref{app:gaussian-representation}.
\begin{proposition}[Gaussian representation after one update]
\label{prop:second-forward-gaussian-representation}
Assume $\vx_{\ell-1}^{(0)}\neq 0$ and $\vg_{\ell}^{(0)}\neq 0$. Define
\begin{align}
    \mP_{g,\ell}^{(0)}
    :=
    \frac{\vg_\ell^{(0)}(\vg_\ell^{(0)})^\top}
    {\|\vg_\ell^{(0)}\|^2},
    \quad
    \mu_{x,\ell}^{(1)}
    &:=
    \va_{\ell}^{(0)}
    \frac{\langle \vx_{\ell-1}^{(0)},\vx_{\ell-1}^{(1)}\rangle}
    {\|\vx_{\ell-1}^{(0)}\|^2},
    \quad
    (\sigma_{x,\ell}^{(1)})^2
    :=
    \frac{\|(\mI-\mP_{x,\ell-1}^{(0)})\vx_{\ell-1}^{(1)}\|^2}{n}.
\end{align}
Then, on a possibly enlarged probability space, there exists a standard Gaussian innovation
$\vz_{\ell}^{(1)}\sim\mathcal N(0,\mI_n)$, independent of the first-pass innovations
$\vz_{\ell}^{(0)}$ and $\widetilde{\vz}_{\ell}^{(0)}$, such that
\begin{align}
    \va_{\ell}^{(1)}
    \overset{d}{=}
    (\mI-\mP_{g,\ell}^{(0)})\mu_{x,\ell}^{(1)}
    +
    \vg_{\ell}^{(0)}
    \frac{\langle \vb_{\ell}^{(0)},\vx_{\ell-1}^{(1)}\rangle}
    {\|\vg_{\ell}^{(0)}\|^2}
    +
    \sigma_{x,\ell}^{(1)}
    (\mI-\mP_{g,\ell}^{(0)})\vz_{\ell}^{(1)} .
    \label{eq:second-forward-gaussian-rep}
\end{align}
If $\vx_{\ell-1}^{(0)}=0$, we set $\mP_{x,\ell-1}^{(0)}=0$; if $\vg_\ell^{(0)}=0$, we set $\mP_{g,\ell}^{(0)}=0$. In these degenerate cases, the formula is interpreted by omitting the corresponding aligned and projected component.
\end{proposition}

\vspace{-0.5em}
\paragraph{The Training-Induced Coupling Term.}
The representation \eqref{eq:second-forward-gaussian-rep} first separates out the forward-forward effect caused by reusing $\mW_\ell^{(0)}$ in the first and second forward passes. This effect is captured by the conditional mean $\mu_{x,\ell}^{(1)}$ and variance $(\sigma_{x,\ell}^{(1)})^2$. In particular, the mean records the component of the second forward increment aligned with the first forward increment $\va_\ell^{(0)}$, while the variance records the residual Gaussian fluctuation after removing the first forward feature direction $\vx_{\ell-1}^{(0)}$.

The additional effect comes from conditioning on the first backward pass. This conditioning projects away the direction $\vg_\ell^{(0)}$ and introduces the term 
$\vg_{\ell}^{(0)}\langle \vb_{\ell}^{(0)},\vx_{\ell-1}^{(1)}\rangle/\|\vg_{\ell}^{(0)}\|^2$, which represents the central training-time coupling effect. Indeed, $\vx_{\ell-1}^{(1)}$ has already been modified by the first gradient update and hence carries gradient information from the first backward pass, while $\vb_\ell^{(0)}$ defined in \eqref{eq:init-forward-backward-increments} explicitly contains the reused layer weights $(\mW_{\ell}^{(0)})^{\top}$. Consequently, the inner product can produce a structured $\mW_\ell^{(0)}(\mW_\ell^{(0)})^\top\vg_\ell^{(0)}$-type contraction. Thus, unlike at initialization, this is a genuine training-induced forward-backward correlation and need not vanish by taking $n\to\infty$ alone.

The next result makes this intuition precise. It compares the true one-step dynamics with an auxiliary dynamics in which the first backward pass uses independent backward weights, and shows that the training-induced coupling persists in the infinite-width limit. The proof is provided in
Appendix~\ref{app:training-infinite-width-convergence}.

\begin{proposition}
\label{prop:width-convergence-one-step}
Suppose $\mathcal L'$, $\phi$, and $\phi'$ are Lipschitz continuous. As
$n\to\infty$, for fixed $L$, the coordinates of
$\{\vh_\ell^{(1)}\}_{\ell=0}^{L}$ become asymptotically i.i.d. copies of a
process $\{H_\ell^{(1)}\}_{\ell=0}^{L}$ satisfying
\begin{align}
    H_{\ell}^{(1)}
    &=
    H_{\ell-1}^{(1)}
    +
    \frac{1}{\sqrt L}\,A_\ell^{(1)}
    -
    \frac{\eta_c}{L}\chi^{(0)}
    \E\!\left[\phi(H_{\ell-1}^{(0)})\phi(H_{\ell-1}^{(1)})\right]
    G_{\ell}^{(0)}
    +
    \mathcal C_{\ell}^{(0\to 1)},
    \label{eq:one-step-width-limit}
\end{align}
where the per-layer training-induced coupling term is
\begin{align}
    \mathcal C_{\ell}^{(0\to 1)}
    &:=
    -\frac{\eta_c}{L^2}\chi^{(0)}
    \E\!\left[\phi(H_{\ell-2}^{(0)})\phi(H_{\ell-2}^{(1)})\right]
    \E\!\left[\phi'(H_{\ell-1}^{(0)})\phi'(H_{\ell-1}^{(1)})\right]
    G_{\ell}^{(0)},
    \label{eq:one-step-coupling-term}
\end{align}
with the convention that $\mathcal C_{1}^{(0\to 1)}=0$. Here
$\{A_\ell^{(1)}\}_{\ell=1}^{L}$ are centered Gaussian innovations whose joint
law with the initial innovations $\{A_\ell^{(0)}\}_{\ell=1}^{L}$ satisfies
\begin{align}
    \E\!\left[A_\ell^{(i)}A_{\ell'}^{(i')}\right]
    =
    \delta_{\ell,\ell'}
    \E\!\left[\phi(H_{\ell-1}^{(i)})\phi(H_{\ell-1}^{(i')})\right],
    \qquad i,i'\in\{0,1\}.
\end{align}
In particular, the innovation pairs
$\{(A_\ell^{(0)},A_\ell^{(1)})\}_{\ell=1}^{L}$ are independent across layers.

For the auxiliary dynamics with decoupled backward weights, the second forward
process converges to a limiting process
$\{\bar H_\ell^{(1)}\}_{\ell=0}^{L}$ satisfying
\begin{align}
    \bar H_{\ell}^{(1)}
    &=
    \bar H_{\ell-1}^{(1)}
    +
    \frac{1}{\sqrt L}\,\bar A_\ell^{(1)}
    -
    \frac{\eta_c}{L}\chi^{(0)}
    \E\!\left[\phi(H_{\ell-1}^{(0)})\phi(\bar H_{\ell-1}^{(1)})\right]
    \bar G_{\ell}^{(0)} ,
    \label{eq:one-step-width-limit-decoupled}
\end{align}
where $\bar A_\ell^{(1)}$ is defined analogously to $A_\ell^{(1)}$, with
$H^{(1)}$ replaced by $\bar H^{(1)}$.
\end{proposition}

Proposition~\ref{prop:width-convergence-one-step} shows that, after even a single SGD update, the large-width limit of the true dynamics no longer coincides with the auxiliary GIA dynamics. The difference is precisely the training-induced forward--backward correlation term $\mathcal C_{\ell}^{(0\to 1)}$ in \eqref{eq:one-step-coupling-term}, which originates from the middle component in the Gaussian representation \eqref{eq:second-forward-gaussian-rep} and is absent from the auxiliary limit. Thus, unlike at initialization, increasing width alone does not recover a GIA-decoupled description during training.

\vspace{-0.7em}
\paragraph{Depth Suppression and Neural Feature Dynamics.}
At the same time, Proposition~\ref{prop:width-convergence-one-step} shows that the surviving correlation term $\mathcal{C}_{\ell}^{(0\to 1)}$ is higher order in depth. Its per-layer contribution is $O(L^{-2})$, whereas the SGD drift is $O(L^{-1})$. After accumulation over $L$ layers, the total contribution is $O(L^{-1})$ and becomes negligible as $L\to\infty$. Hence, although width alone does not restore a GIA-style description during training, the depth scaling suppresses this forward--backward coupling.

This depth suppression motivates \textit{neural feature dynamics (NFD)}, a candidate limiting dynamics for infinite-depth feature learning. In NFD, we replace the backward weights by independent copies, thereby omitting the additional forward--backward correlation terms caused by reusing the same initial weights, while retaining the feature and gradient covariance structure generated by training.

\begin{definition}[Neural Feature Dynamics]
\label{def:nfd}
For the ResNet defined in \eqref{eq:resnet} with fixed $K$, the neural feature dynamics (NFD) is the following forward-backward system: for each $k=0,\ldots,K$,
\begin{align}
    dH_t^{(k)}
    &=
    -
    \eta_c
    \sum_{i=0}^{k-1}
    \mathring\chi^{(i)}
    \Sigma_t^{(i,k)}
    G_t^{(i)}\,dt
    +
    dW_t^{(k)},
    \label{eq:nfd-forward}
    \\
    dG_t^{(k)}
    &=
    -
    \eta_c
    \sum_{i=0}^{k-1}
    \mathring\chi^{(i)}
    \Theta_t^{(i,k)}
    \phi(H_t^{(i)})\phi'(H_t^{(k)})\,dt
    +
    \phi'(H_t^{(k)})\,d\widetilde W_t^{(k)},
    \label{eq:nfd-backward}
\end{align}
with boundary conditions
\begin{align}
    H_0^{(k)}
    &=
    \widehat H_0^{(k)}
    -
    \eta_c
    \sum_{i=0}^{k-1}
    \mathring\chi^{(i)}
    \Sigma_0^{(i,k)}
    G_0^{(i)},
    \qquad
    G_1^{(k)}
    =
    \widehat G_1^{(k)}
    -
    \eta_c
    \sum_{i=0}^{k-1}
    \mathring\chi^{(i)}
    H_1^{(i)},
    \label{eq:nfd-boundary}
\end{align}
where $\eta_c>0$ is the effective learning-rate and 
\begin{align}
    \Sigma_t^{(i,k)}
    &:=
    \E[\phi(H_t^{(i)})\phi(H_t^{(k)})],
    \qquad
    \Theta_t^{(i,k)}
    :=
    \E[G_t^{(i)}G_t^{(k)}].
\end{align}
The Gaussian initial and terminal data satisfy
\begin{align}
    \Cov(\widehat H_0^{(i)},\widehat H_0^{(j)})
    =
    \frac{\langle \vx^{(i)},\vx^{(j)}\rangle}{d},
    \qquad
    \Cov(\widehat G_1^{(i)},\widehat G_1^{(j)})
    =
    1.
\end{align}
The forward Brownian family $\{W_t^{(k)}\}_{k=0}^{K}$ and backward Brownian family
$\{\widetilde W_t^{(k)}\}_{k=0}^{K}$ are independent of each other, with within-family covariations
\begin{align}
    \label{eq:Brownian-family}
    d\langle W^{(i)},W^{(j)}\rangle_t
    &=
    \Sigma_t^{(i,j)}\,dt,
    \qquad
    d\langle \widetilde W^{(i)},\widetilde W^{(j)}\rangle_t
    =
    \Theta_t^{(i,j)}\,dt.
\end{align}
\end{definition}

More generally, an NFD candidate can be defined for a given architecture and optimizer by replacing reused backward weights with independent copies. However, convergence to this candidate is not automatic; it depends on the architecture, parameterization, optimizer, and width-depth scaling.

For the ResNet \eqref{eq:resnet}, Propositions~\ref{prop:second-forward-gaussian-representation}--\ref{prop:width-convergence-one-step} show that the training-induced correlation term remains nontrivial in the infinite-width limit but is higher order in depth. Under the regularity and nondegeneracy conditions below, this depth suppression yields convergence of the finite-network feature-learning dynamics to its NFD limit in the infinite-width-then-depth limit. The nondegeneracy condition is empirically supported in Figure~\ref{fig:eigenvalue}; the proof is provided in Appendix~\ref{app:training-infinite-width-convergence}-\ref{app:depth-convergence}.

\begin{assumption}
\label{assmp:eigenvalue}
We assume:
\begin{enumerate}[leftmargin=*]
    \item $\mathcal L'$, $\phi$, and $\phi'$ are Lipschitz continuous.
    
    \item For each finite training horizon $K$, the NFD in Definition~\ref{def:nfd} admits a solution on $[0,1]$, and its covariance matrices are uniformly nondegenerate, i.e., there exists $\varepsilon>0$ such that,
    \begin{align}
        \inf_{t\in[0,1]}
        \min\left\{
        \lambda_{\min}(\mK_t^{(k)}),
        \lambda_{\min}(\mG_t^{(k)})
        \right\}
        \geq \varepsilon,
        \quad \forall k\leq K,
    \end{align}
    where $(\mK_t^{(k)})_{ij}:=\Sigma_t^{(i,j)}$ and $(\mG_t^{(k)})_{ij}:=\Theta_t^{(i,j)}$, for all $0\leq i,j\leq k$.
\end{enumerate}
\end{assumption}

\begin{remark}[Existence of the NFD system]
The existence of a (weak) solution of the NFD system may be argued carefully via an Euler--Maruyama discretization with tightness and weak convergence of stochastic integrals (cf.\ \cite[Section 7]{kurtz2006weak}). Since the Brownian families in \eqref{eq:Brownian-family} have a nonstandard feature-gradient covariance structure, we leave a complete existence theory to future work.
\end{remark}

\begin{figure}
    \centering
    \begin{minipage}{0.24\linewidth}
        \centering
        \includegraphics[width=0.9\linewidth]{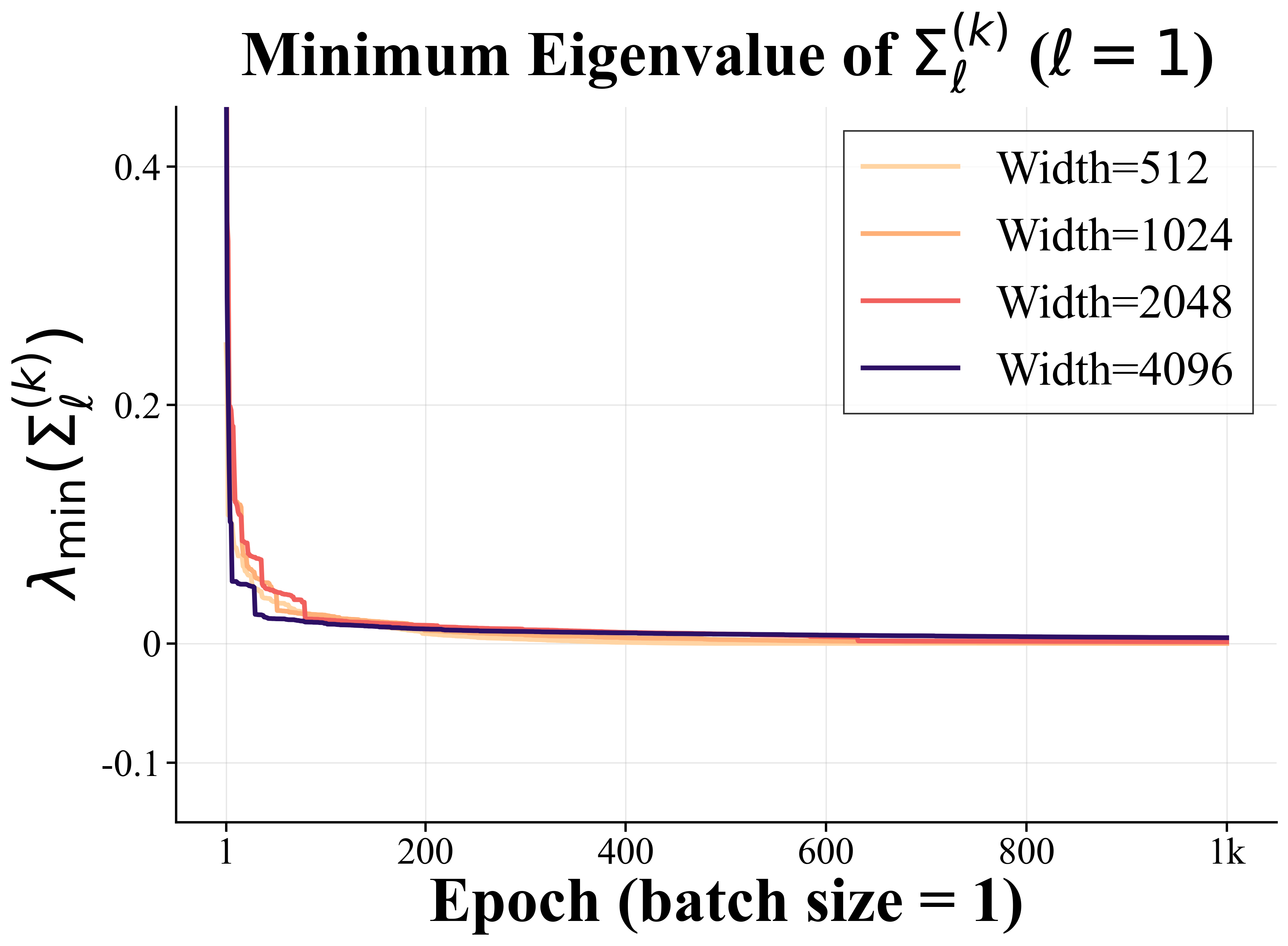}
    \end{minipage}%
    \hfill
    \begin{minipage}{0.24\linewidth}
        \centering
        \includegraphics[width=0.9\linewidth]{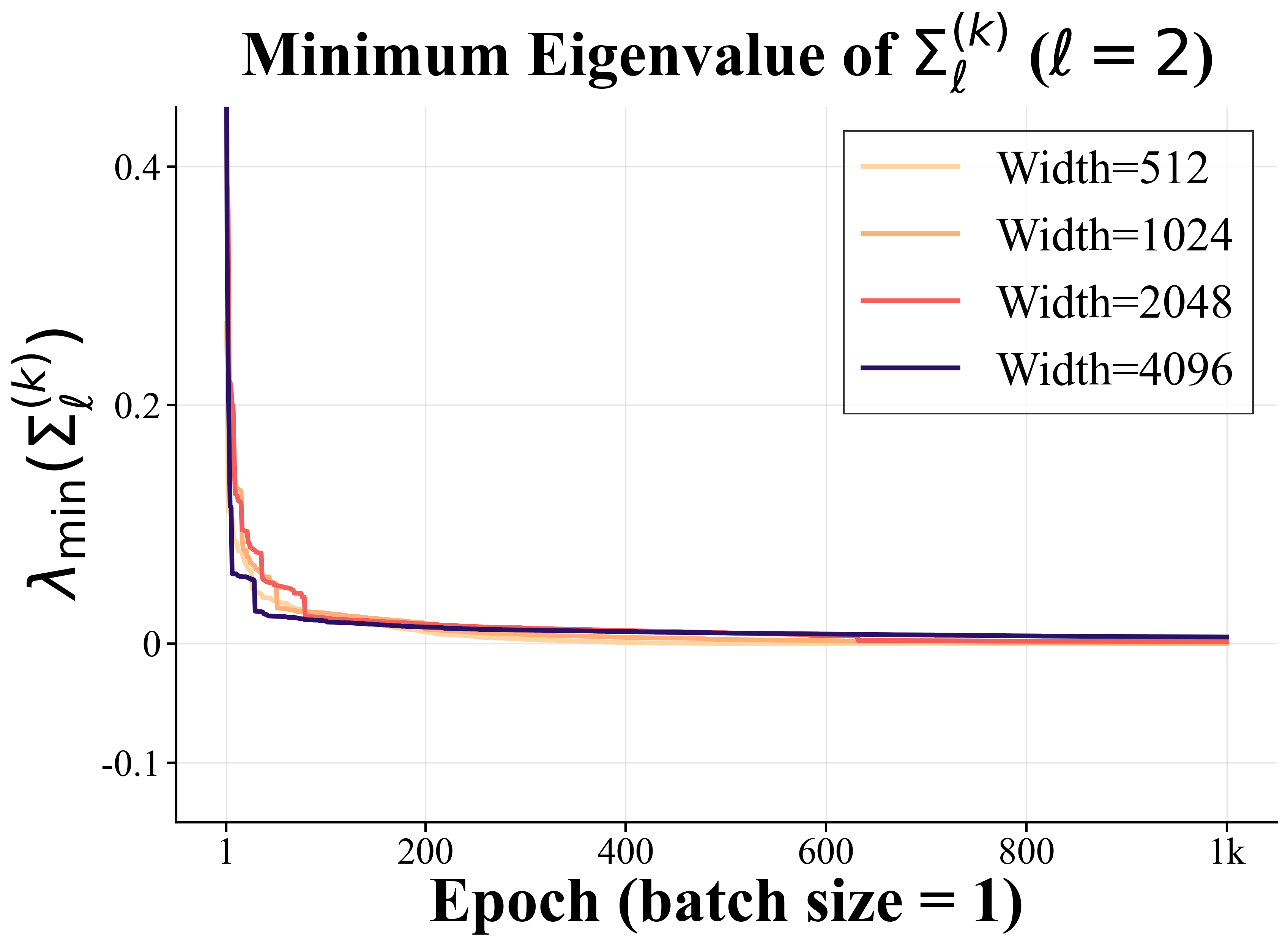}
    \end{minipage}
    \hfill
    \begin{minipage}{0.24\linewidth}
        \centering
        \includegraphics[width=0.9\linewidth]{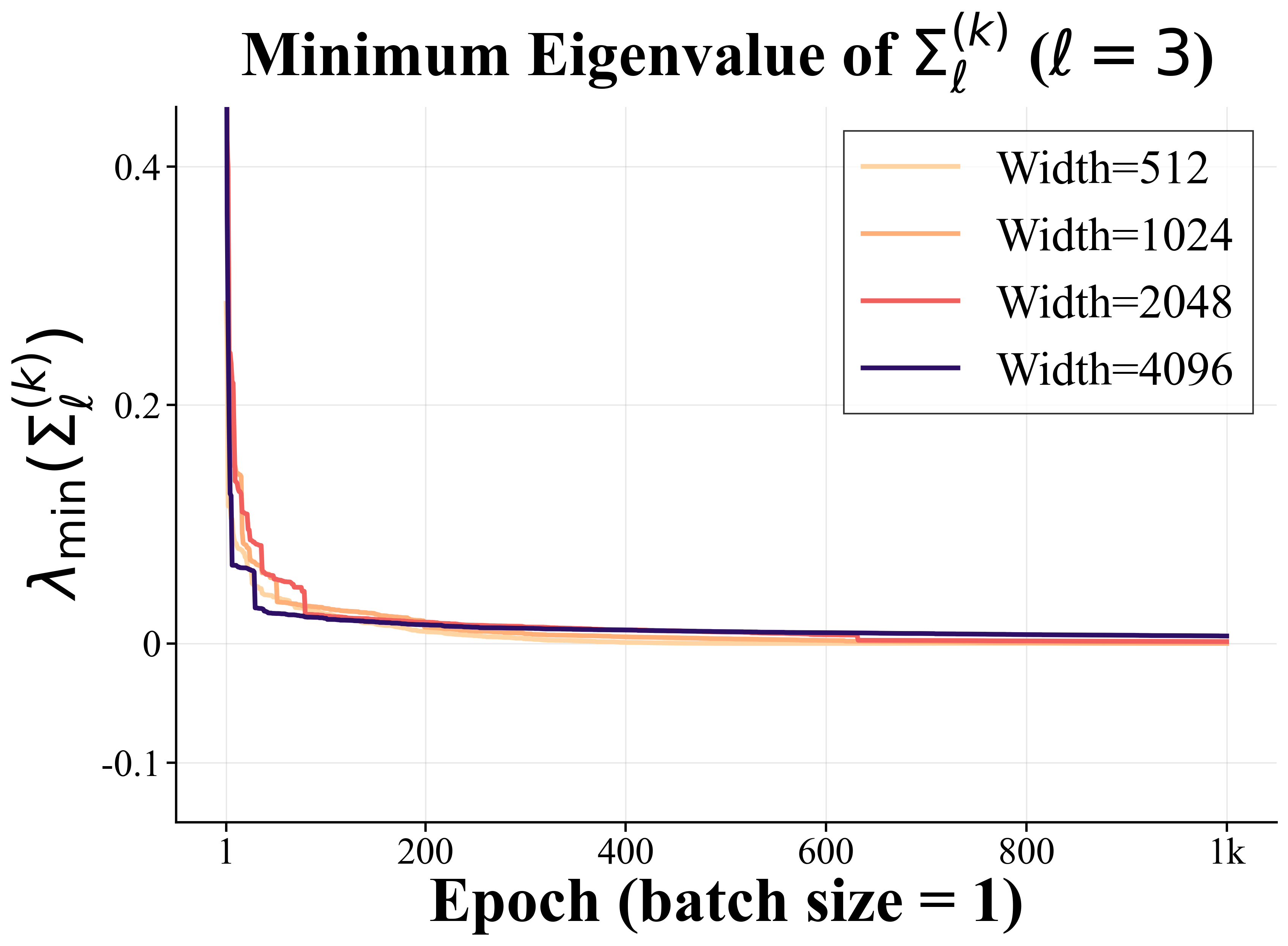}
    \end{minipage}%
    \begin{minipage}{0.24\linewidth}
        \centering
        \includegraphics[width=0.9\linewidth]{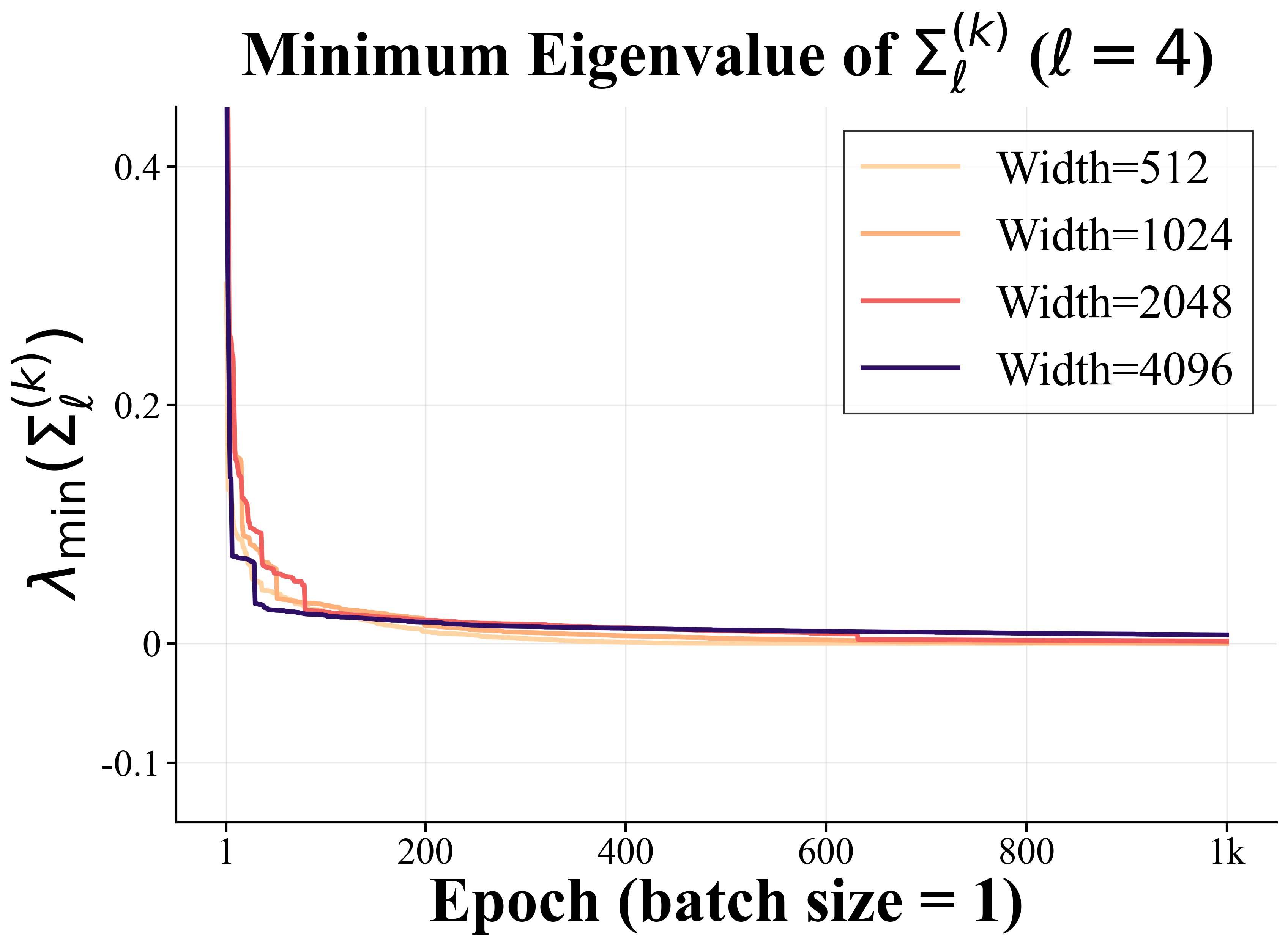}
    \end{minipage}%
    \caption{\textbf{Minimum covariance eigenvalues during training.}
    ResNets are trained on CIFAR-10 with online SGD (batch size 1) across 5 seeds, using 4 hidden layers and widths of 512--4096. The minimum eigenvalues of $\Sigma_t^{(k)}$ and $\Theta_t^{(k)}$ remain positive across layers, supporting Assumption~\ref{assmp:eigenvalue}; wider networks yield larger eigenvalues, while a width of 512 approaches near-degenerate regimes.}
    \label{fig:eigenvalue}
\end{figure}

\begin{theorem}[Convergence to neural feature dynamics]
\label{thm:feature-learning-dynamics}
Suppose Assumption~\ref{assmp:eigenvalue} holds and $\eta = \eta_c n$. As $n, L\to\infty$ sequentially, the network output $f^{(k)}$ converges in probability to $\mathring f^{(k)}=\E\![G_1^{(k)}H_1^{(k)}]$, where $(H_t^{(k)},G_t^{(k)})$ evolves according to the NFD system in Definition~\ref{def:nfd}. 
\end{theorem}

\begin{corollary}
\label{cor:nfd-depth-rate}
Under the assumptions of Theorem~\ref{thm:feature-learning-dynamics}, for every fixed training horizon $K<\infty$ and each $k\le K$, the depth-discretized dynamics converge to their NFD limit with
\begin{align}
    W_2\!\left(\Law(H_{\ell}^{(k)}),\Law(H_{t_\ell}^{(k)})\right)^2
    +
    W_2\!\left(\Law(G_{\ell}^{(k)}),\Law(G_{t_\ell}^{(k)})\right)^2
    \le C L^{-1},
    \label{eq:nfd-depth-rate}
\end{align}
uniformly over $\ell\in[L]$. Moreover, let
$\mathcal C_{\ell}^{(k\to K)}$ denote the per-layer forward--backward coupling
contribution induced by the $k$-th training step and omitted in the NFD
dynamics at horizon $K$. Then
\begin{align}
    \E\left|
        \sum_{k=0}^{K-1}\sum_{\ell=1}^{L}
        \mathcal C_{\ell}^{(k\to K)}
    \right|^2
    \le C_K L^{-2}.
    \label{eq:nfd-coupling-rate}
\end{align}
Here, $C_K>0$ may depend on the fixed training horizon $K$, but is independent of $L$ and $n$.
\end{corollary}

Theorem~\ref{thm:feature-learning-dynamics}, together with Corollary~\ref{cor:nfd-depth-rate}, shows that the finite-depth dynamics of the ResNet \eqref{eq:resnet} converge to its NFD limit with $O(L^{-1})$ depth discretization error, whereas the accumulated coupling term decays faster at rate $O(L^{-2})$. Thus, the leading depth error is the SDE discretization error rather than the omitted forward--backward correlation. Figure~\ref{fig:gia_alignment} supports this depth-restoration mechanism, and Figure~\ref{fig:depth-width-convergence} further supports convergence with both width and depth after training.
\vspace{-0.5em}
\section{Conclusions}
\vspace{-0.5em}

In this work, we studied infinite-depth feature-learning dynamics through the lens of reused-weight forward--backward coupling. Using conditional Gaussian representations, we isolated the correlation induced by reusing each forward weight matrix in backpropagation before taking any network limit. At initialization, this coupling is a finite-width effect and vanishes as the width increases. During training, however, SGD introduces a nontrivial forward--backward correlation term that survives the infinite-width limit. Under depth-$\mu$P scaling, this term is higher order in depth, and its accumulated effect becomes negligible as $L\to\infty$. This depth-induced suppression motivates neural feature dynamics (NFD), a candidate infinite-depth feature-learning dynamics with decoupled backward weights. Under suitable conditions and for a fixed training horizon $K$, we proved convergence of the finite-network training dynamics to NFD, with $O(L^{-1})$ depth-discretization error and a faster $O(L^{-2})$ second-moment decay of the omitted reused-weight coupling.

\vspace{-0.5em}
\paragraph{Limitations.}
Our analysis focuses on one-layer ResNets under depth-$\mu$P scaling trained by stream SGD, and therefore does not yet cover more general architectures or training algorithms, such as multi-layer residual blocks, attention-based models, Transformers, mini-batch training, or adaptive optimizers. The convergence result is also sequential, assumes nondegeneracy conditions, and holds for a fixed finite training horizon $K$. Extending the theory to joint width-depth limits, longer training horizons, weaker assumptions, and realistic data distributions remains open. Finally, our results identify when reused-weight coupling becomes negligible in the infinite-depth limit, but, by themselves, do not determine whether suppressing this coupling is beneficial or harmful for optimization, representation quality, or generalization.

\paragraph{Future Directions.} A natural next step is to extend this framework to more practical architectures and optimizers, especially Transformers trained with adaptive methods. On the theoretical side, an important goal is to use NFD as a foundation for optimization and generalization analyses in infinite-depth feature-learning regimes. More broadly, conditional Gaussian representations may provide a useful tool for studying how different parameterizations, scaling rules, architectures, and optimizers affect forward--backward coupling. In particular, while depth-$\mu$P suppresses the reused-weight correlation studied here, it remains open whether vanishing coupling is always desirable, or whether persistent forward--backward correlations may sometimes play a useful role in feature learning. Understanding when such coupling should vanish, persist, or be controlled may lead to sharper principles for designing scalable deep learning systems.

\bibliographystyle{plain}
\bibliography{refs}


\clearpage
\appendix

\section{Useful Mathematical Results}\label{app:Useful Mathematical Results}
\begin{lemma}[Gronwall’s inequality]
    \label{lem:Gronwall}
    Let $I=[a,b]$ for an interval such that $a < b < \infty$.
	Let $u$, $\alpha$, $\beta$ be real-valued continuous functions such that $\beta$ is non-negative and $u$ satisfies the integral inequality
	\begin{align*}
		u(t)\leq \alpha(t) + \int_0^t\beta(s) u(s) ds,\quad\forall t\in I.
	\end{align*}
	Then 
	\begin{align*}
		u(t)\leq \alpha(t) + \int_0^t\alpha(s)\beta(s)\exp\left(\int_s^t \beta(r) dr\right), \quad\forall t\in I.
	\end{align*}
	If, in addition, $\alpha(t)$ is non-decreasing, then
	\begin{align*}
		u(t)\leq \alpha(t)\exp\left(\int_0^t\beta(s) ds\right),\quad \forall t\in I.
	\end{align*}
\end{lemma}

\begin{lemma}[Gronwall’s inequality (discrete version)]
    \label{lem:Gronwall-discrete}
    Let $(u_n)$ and $(\beta_n)$ be non-negative sequences satisfying
    \begin{equation*}
        u_n \le \alpha + \sum_{k=0}^{n-1} \beta_k u_k, \quad \forall\,n,
    \end{equation*}
    where $\alpha \ge 0$.
    Then
    \begin{equation*}
        u_n \le \alpha \exp\left( \sum_{k=0}^{n-1} \beta_k\right), \quad \forall\,n.
    \end{equation*}    
\end{lemma}

\section{Post-Activation vs. Pre-Activation ResNet Design}
\label{app:post-act-style}

The ResNet architecture defined in \eqref{eq:resnet} follows the so-called \textit{pre-activation style} \citep{he2016identity}, where the residual connection is applied to the \emph{pre-activation} feature vectors. There is, however, another widely used variant: the \textit{post-activation style} \citep{he2016deep}, in which the residual connection is applied after the nonlinearity. The corresponding residual recursion is
\begin{equation}
    h_{0} = \phi\left(\frac{1}{\sqrt{d}}U x\right), \quad
    h_{\ell} = h_{\ell-1} + \frac{1}{\sqrt{L}} \, \phi(\frac{1}{\sqrt{n}}W_{\ell}h_{\ell-1}), \quad \forall \ell \in \{1,2,\dots,L\}.
    \label{eq:post-act-style}
\end{equation}

Historically, the post-activation formulation was first introduced in the original ResNet paper \citep{he2016deep}, but subsequent large-scale architectures, including modern Transformers \citep{vaswani2017attention}, have overwhelmingly favored the pre-activation design. This shift was motivated primarily by the empirically observed superior stability of pre-activation models. However, the theoretical reasons for this preference are less often articulated.  Proposition~\ref {prop:post-act} shows that pre-activation is essentially more stable as the depth $L \to \infty$ by proving that the feature magnitudes in the post-activation variant \eqref{eq:post-act-style} may diverge, even under the $\frac{1}{\sqrt{L}}$ depth-scaling.

\begin{proof}[Proof of Proposition \ref{prop:post-act}]
    Recall that $\phi$ is assumed to be positive dominate, that is, there exist nonnegative constants $c_1, c_2$, not both zero, such that 
    $\E[\phi(xZ)] \geq c_1 |x| + c_2$, $\forall x \in \mathbb{R}$,
    where $Z$ is the standard Gaussian random variable.
    
    First note that $(Ux)_i \sim \Gaus(0,\norm{x}^2)$ and hence $\E[h_{0,i}] \ge c_1\frac{\norm{x}}{\sqrt{d}} + c_2$.
    Let $\gB_{\ell}$ be the $\sigma$-algebra generated by $\{h_{0},\cdots, h_{\ell-1} \}$. Then, observe that
    \begin{align*}
        \E[h_{\ell,i}\mid \gB_{\ell}]
        &=h_{\ell-1,i} 
        + \frac{1}{\sqrt{L}}\E\left[\phi\left(\frac{ \norm{h_{\ell-1}} }{\sqrt{n}} Z\right)\mid \gB_{\ell}\right]
        \geq h_{\ell-1,i}  
        + \frac{1}{\sqrt{L}} \left(c_1 \frac{ \norm{h_{\ell-1}} }{\sqrt{n}} + c_2\right)\\
        &\geq h_{\ell-1,i} + \frac{1}{\sqrt{L}} \left(c_1 \frac{h_{\ell-1,i}}{\sqrt{n}} + c_2\right)
        = \left(1 + c_1\frac{1}{\sqrt{Ln}}\right)h_{\ell-1,i} + c_2 \frac{1}{\sqrt{L}}.
    \end{align*}
    Then taking expectation of $\gB_{\ell}$ yields
    \begin{align*}
        \E[h_{\ell,i}] \geq \left(1 + c_1\frac{1}{\sqrt{Ln}}\right)\E[h_{\ell-1,i}] + c_2 \frac{1}{\sqrt{L}}.
    \end{align*}
    Therefore, we obtain
    \begin{align*}
        \E[h_{\ell,i}]\geq& \left(1 + c_1\frac{1}{\sqrt{Ln}}\right)^{\ell} \E[h_{0,i}] + c_2\frac{1}{\sqrt{L}} \sum_{j=0}^{\ell-1} \left(1 + c_1\frac{1}{\sqrt{Ln}}\right)^j\\
        \geq &\left(1 + c_1\frac{1}{\sqrt{Ln}}\right)^{\ell} \left(c_1\frac{\norm{x}}{\sqrt{d}} + c_2\right) + c_2\frac{1}{\sqrt{L}} \ell.
    \end{align*}
    Therefore we obtain
    \begin{align*}
        \E[h_{L,i}]\geq c_1\frac{\norm{x}}{\sqrt{d}}\left(1 + c_1\frac{1}{\sqrt{Ln}}\right)^{L} + c_2 \sqrt{TL}
        \rightarrow\infty
    \end{align*}
    as $L\rightarrow\infty$,
    provided either $c_1>0$ or $c_2>0$. 
\end{proof}

\section{Gaussian Representation of Forward and Backward Innovations}
\label{app:gaussian-representation}
\subsection{At Initialization}
In this section, we restate Proposition \ref{prop:layerwise-gaussian-representation} as Proposition \ref{prop:layerwise-gaussian-representation-pf} below and provide its proof.
At initialization, the forward and backward recursions are
\begin{align}
    h_{\ell}
    &= h_{\ell-1} + \frac{1}{\sqrt{L}}\, a_{\ell},
    \label{eq:forward-aell}\\
    g_{\ell-1}
    &= g_{\ell} + \frac{1}{\sqrt{L}}\, \bigl(b_{\ell}\odot \phi'(h_{\ell-1})\bigr),
    \label{eq:backward-bell}
\end{align}
where
\begin{align}
    a_{\ell} := \frac{1}{\sqrt n}W_{\ell}\phi(h_{\ell-1}),
    \qquad
    b_{\ell} := \frac{1}{\sqrt n}W_{\ell}^{\top}g_{\ell}.
\end{align}
To analyze their limiting behavior, define the filtrations
\begin{align}
    \mathcal F_{\ell} := \sigma(h_0,\dots,h_{\ell}),
    \qquad
    \mathcal B_{\ell} := \sigma(g_L,g_{L-1},\dots,g_{\ell}).
\end{align}

\begin{proposition}
\label{prop:layerwise-gaussian-representation-pf}
For each \(\ell\in[L]\), the following hold.

\begin{enumerate}[leftmargin=*]
\item[(i)] \textbf{Forward Gaussian innovation.}
Conditionally on \(\mathcal F_{\ell-1}\),
\begin{align}
    a_{\ell}\mid \mathcal F_{\ell-1}
    \sim
    \mathcal N\!\left(
        0,\,
        \frac{\|\phi(h_{\ell-1})\|^2}{n} I
    \right).
\end{align}
Equivalently, there exists \(z_{\ell}\sim\mathcal N(0,I)\), independent of \(\mathcal F_{\ell-1}\), such that
\begin{align}
    a_{\ell}\mid \mathcal F_{\ell-1}
    \overset{d}=
    \sqrt{\frac{\|\phi(h_{\ell-1})\|^2}{n}}\, z_{\ell}.
\end{align}
Hence, on a possibly enlarged probability space, one may realize \(\{z_{\ell}\}_{\ell=1}^L\) as mutually independent standard Gaussian vectors and write
\begin{align}
    h_{\ell}
    =
    h_{\ell-1}
    + \frac{1}{\sqrt{L}}
      \sqrt{\frac{\|\phi(h_{\ell-1})\|^2}{n}}\, z_{\ell}.
    \label{eq:h-rewrite}
\end{align}

\item[(ii)] \textbf{Backward Gaussian innovation.}
Conditionally on \((\mathcal F_L,\mathcal B_{\ell})\), the following hold.

If \(\phi(h_{\ell-1})\neq 0\), let
\begin{align}
    P_{\ell-1}
    :=
    \frac{\phi(h_{\ell-1})\phi(h_{\ell-1})^{\top}}{\|\phi(h_{\ell-1})\|^2}.
\end{align}
Then
\begin{align}
    b_{\ell}\mid (\mathcal F_L,\mathcal B_{\ell})
    \sim
    \mathcal N\!\left(
        \frac{a_{\ell}^{\top}g_{\ell}}{\|\phi(h_{\ell-1})\|^2}\,\phi(h_{\ell-1}),
        \frac{\|g_{\ell}\|^2}{N}(I-P_{\ell-1})
    \right).
\end{align}
Equivalently, there exists \(\widetilde z_{\ell}\sim \mathcal N(0,I)\), independent of \((\mathcal F_L,\mathcal B_{\ell})\), such that
\begin{align}
    b_{\ell}\mid (\mathcal F_L,\mathcal B_{\ell})
    \overset{d}=
    \frac{a_{\ell}^{\top}g_{\ell}}{\|\phi(h_{\ell-1})\|^2}\,\phi(h_{\ell-1})
    + \sqrt{\frac{\|g_{\ell}\|^2}{N}}\,(I-P_{\ell-1})\widetilde z_{\ell}.
\end{align}
Hence, on a possibly enlarged probability space, one may realize
\(\{\widetilde z_{\ell}\}_{\ell=1}^L\) as mutually independent standard Gaussian vectors, independent of \(\{z_{\ell}\}_{\ell=1}^L\), and write
\begin{align}
    g_{\ell-1}
    =
    g_{\ell}
    + \frac{1}{\sqrt{L}}
    \left[
        \frac{a_{\ell}^{\top}g_{\ell}}{\|\phi(h_{\ell-1})\|^2}\,\phi(h_{\ell-1})
        + \sqrt{\frac{\|g_{\ell}\|^2}{N}}\,(I-P_{\ell-1})\widetilde z_{\ell}
    \right]
    \odot \phi'(h_{\ell-1}).
    \label{eq:g-rewrite}
\end{align}

If \(\phi(h_{\ell-1})=0\), then necessarily \(a_{\ell}=0\). In this case, we set \(P_{\ell-1}:=0\), and
\begin{align}
    b_{\ell}\mid(\mathcal F_L,\mathcal B_{\ell})
    \overset{d}{=}
    \sqrt{\frac{\|g_{\ell}\|^2}{N}}\,\widetilde z_{\ell},
\end{align}
so that
\begin{align}
    g_{\ell-1}
    =
    g_{\ell}
    + \frac{1}{\sqrt{L}}
    \left[
        \sqrt{\frac{\|g_{\ell}\|^2}{N}}\,\widetilde z_{\ell}
    \right]
    \odot \phi'(h_{\ell-1}).
\end{align}
\end{enumerate}
In particular, the forward innovations \(\{z_{\ell}\}\) and the residual backward innovations \(\{\widetilde z_{\ell}\}\) can be chosen mutually independent.
\end{proposition}

\begin{proof}
Fix \(\ell\in[L]\). Conditionally on \(\mathcal F_{\ell-1}\), the vector \(\phi(h_{\ell-1})\) is deterministic and
\(W_{\ell}\) is independent of \(\mathcal F_{\ell-1}\) with iid \(\mathcal N(0,1)\) entries. Hence, 
\[
a_{\ell}\mid\mathcal F_{\ell-1}
\sim
\mathcal N\!\left(
0,\,
\frac{\|\phi(h_{\ell-1})\|^2}{n}I
\right).
\]
Equivalently, there exists \(z_{\ell}\sim\mathcal N(0,I)\), independent of \(\mathcal F_{\ell-1}\), such that
\[
a_{\ell}\mid\mathcal F_{\ell-1}
\overset d=
\sqrt{\frac{\|\phi(h_{\ell-1})\|^2}{n}}\,z_{\ell}.
\]
Realizing these variables on an enlarged probability space gives \eqref{eq:h-rewrite}.

Next, conditionally on \((\mathcal F_L,\mathcal B_{\ell})\), the matrix \(W_{\ell}\) is constrained by
\[
\frac{1}{\sqrt n}W_{\ell}\phi(h_{\ell-1})=a_{\ell}.
\]

Assume first that \(\phi(h_{\ell-1})\neq 0\). Define
\[
P_{\ell-1}
:=
\frac{\phi(h_{\ell-1})\phi(h_{\ell-1})^{\top}}{\|\phi(h_{\ell-1})\|^2},
\qquad
W_{\ell}^{*}
=
\sqrt n\,a_{\ell}\frac{\phi(h_{\ell-1})^{\top}}{\|\phi(h_{\ell-1})\|^2}.
\]
Then \(W_{\ell}^{*}\) is the minimum-Frobenius-norm solution to the constraint. Since the remaining degrees of freedom lie on the orthogonal complement of \(\phi(h_{\ell-1})\), Gaussian linear regression gives
\[
W_{\ell}\mid(\mathcal F_L,\mathcal B_{\ell})
\overset d=
W_{\ell}^{*}+\widetilde W_{\ell}(I-P_{\ell-1}),
\]
where \(\widetilde W_{\ell}\) is an iid copy of \(W_{\ell}\), independent of \((\mathcal F_L,\mathcal B_{\ell})\). Therefore,
\[
b_{\ell}
=
\frac{1}{\sqrt n}W_{\ell}^{\top}g_{\ell}
\overset d=
\frac{1}{\sqrt n}(W_{\ell}^{*})^{\top}g_{\ell}
+
\frac{1}{\sqrt n}(I-P_{\ell-1})\widetilde W_{\ell}^{\top}g_{\ell}.
\]
Hence
\[
b_{\ell}\mid (\mathcal F_L,\mathcal B_{\ell})
\sim
\mathcal N\!\left(
\frac{a_{\ell}^{\top}g_{\ell}}{\|\phi(h_{\ell-1})\|^2}\,\phi(h_{\ell-1}),
\frac{\|g_{\ell}\|^2}{n}(I-P_{\ell-1})
\right).
\]
Equivalently, there exists \(\widetilde z_{\ell}\sim\mathcal N(0,I)\), independent of \((\mathcal F_L,\mathcal B_{\ell})\), such that
\[
b_{\ell}\mid (\mathcal F_L,\mathcal B_{\ell})
\overset d=
\frac{a_{\ell}^{\top}g_{\ell}}{\|\phi(h_{\ell-1})\|^2}\,\phi(h_{\ell-1})
+
\sqrt{\frac{\|g_{\ell}\|^2}{n}}(I-P_{\ell-1})\widetilde z_{\ell}.
\]
Substituting this into \eqref{eq:backward-bell} yields \eqref{eq:g-rewrite}.

If instead \(\phi(h_{\ell-1})=0\), then \(a_\ell=0\), so the conditioning imposes no constraint on \(W_\ell\). Setting \(P_{\ell-1}:=0\), we have
\[
b_\ell
=
\frac{1}{\sqrt n}W_\ell^\top g_\ell
\sim
\mathcal N\!\left(0,\frac{\|g_\ell\|^2}{n}I\right),
\]
and therefore
\[
b_{\ell}\mid(\mathcal F_L,\mathcal B_{\ell})
\overset{d}{=}
\sqrt{\frac{\|g_{\ell}\|^2}{n}}\,\widetilde z_{\ell}.
\]
This gives the stated expression for \(g_{\ell-1}\) in this case.

Finally, \(\{z_{\ell}\}\) comes from the forward conditional laws, whereas \(\{\widetilde z_{\ell}\}\) comes from the independent copies \(\{\widetilde W_{\ell}\}\). Thus, these two families can be realized as mutually independent on a common enlarged probability space.
\end{proof}

\subsection{After One SGD Update}
In this section, we prove Proposition \ref{prop:second-forward-gaussian-representation}.
After one SGD update, the forward feature $h_{\ell}^{(1)}$ becomes
\begin{align}
    h_{\ell}^{(1)}
    =
    h_{\ell-1}^{(1)}
    +
    \frac{1}{\sqrt{Ln}}W_{\ell}^{(0)}x^{(1)}_{\ell}
    -
    \frac{\eta_c}{L}\chi^{(0)}
    \frac{\langle x^{(0)},x^{(1)}\rangle}{n}
    g_{\ell}^{(0)},
\end{align}
As all variables $x$, $\bar x$, and $g$ are from the first forward and backward passes or previous features in the second forward pass, we consider the innovative increment
\begin{align}
    a_{\ell}^{(1)}:=\frac{1}{\sqrt{n}} W_{\ell}^{(0)} x^{(1)}.
\end{align}
Then, we can prove the Gaussian presentation stated in Proposition~\ref {prop:second-forward-gaussian-representation}.
\begin{proof}
Fix a layer $\ell\in[L]$ and write, for simplicity, we denote
\[
    x:=x_{\ell}^{(0)}, 
    \qquad
    \bar x:= x_{\ell}^{(1)}, 
    \qquad
    g:=g_{\ell}^{(0)},
    \qquad
    W:=W_{\ell}^{(0)}.
\]
and
\[
    a
    =
    \frac{1}{\sqrt n}W x,
    \qquad
    b
    =
    \frac{1}{\sqrt n}W^\top g,
    \qquad
    \bar a
    =
    \frac{1}{\sqrt n}W \bar x.
\]
Thus, conditioning on the first forward and backward increments
$(a_{\ell},b_{\ell})$ is equivalent to conditioning on
\[
    \frac{1}{\sqrt n}W x
    =
    a,
    \qquad
    \frac{1}{\sqrt n}W^\top g
    =
    b.
\]
This gives the compatibility condition 
\[
    \langle g,a\rangle
    =
    \langle x,b\rangle.
\]

By Gaussian conditioning, conditionally on these two linear constraints, we have
\begin{align}
    W
    \overset{d}{=}
    &
    \sqrt n\,a
    \frac{x^\top}{\|x\|^2}
    +
    g
    \frac{\sqrt n\,b^\top}{\|g\|^2}
    -
    \frac{\sqrt n\,\langle g,a\rangle}
    {\|g\|^2\|x\|^2}
    g(x)^\top
    +
    (I-P_{g})
    \widetilde{W}
    (I-P_{x}),
    \label{eq:W-two-sided-conditioning-one-update}
\end{align}
where $\widetilde{W}$ is an independent copy of $W$, the projection matrices are given by
\[
    P_{x}
    =
    \frac{x x^\top}
    {\|x\|^2},
    \qquad
    P_{g}
    =
    \frac{g g^\top}
    {\|g\|^2}.
\]

Multiplying \eqref{eq:W-two-sided-conditioning-one-update} by $\bar x/\sqrt n$ gives
\begin{align}
    \bar a
    =
    \frac{1}{\sqrt n}W\bar x
    \overset{d}{=}
    &
    a_{\ell}
    \frac{\langle x,\bar x\rangle}{\|x\|^2}
    +
    g
    \frac{\langle b,\bar x\rangle}{\|g\|^2}
    -
    g
    \frac{
        \langle g,a\rangle
        \langle x,\bar x\rangle
    }
    {\|g\|^2\|x\|^2}
    +
    \frac{1}{\sqrt n}
    (I-P_{g})
    \widetilde{W}
    (I-P_{x})\bar x.
    \label{eq:second-forward-full-rep-current}
\end{align}
Since $\widetilde{W}$ is independent of the conditioning $\sigma$-field, we can choose a $\bar z_{\ell}\sim\mathcal N(0,I)$, independent of the first-pass innovations in a probably enlarged probability space, such that
\begin{align}
    \frac{1}{\sqrt n}
    (I-P_{g})
    \widetilde{W}
    (I-P_{x})\bar x
    \overset{d}{=}
    \sigma_{x}
    (I-P_{g})\bar z,
\end{align}
where
\[
    \sigma_{x}^2
    =
    \frac{
    \|(I-P_{x})\bar x\|^2
    }{n}.
\]

It remains to simplify the deterministic terms. By definition,
\[
    \mu_x
    =
    a
    \frac{
    \langle x,\bar x\rangle
    }
    {
    \|x\|^2
    }.
\]
Using the definition of $P_g$, we obtain
\[
    a
    \frac{\langle x,\bar x\rangle}{\|x\|^2}
    -
    g
    \frac{
        \langle g,a\rangle
        \langle x,\bar x\rangle
    }
    {\|g\|^2\|x\|^2}
    =
    (I-P_{g})\mu_x.
\]
Substituting this identity into \eqref{eq:second-forward-full-rep-current} yields
\begin{align}
    \bar a
    \overset{d}{=}
    (I-P_{g})\mu_x
    +
    g
    \frac{
    \langle b,\bar x\rangle
    }
    {
    \|g\|^2
    }
    +
    \sigma_{x}
    (I-P_{g})\bar z.
\end{align}
This is exactly \eqref{eq:second-forward-gaussian-rep}.
\end{proof}
\section{Convergence Proofs at Initialization}
\label{app:init-convergence-proofs}
In this section, we focus our convergence analysis of feature and gradient propagation at initialization, considering both the first forward feature propagation and the first backward gradient propagation. 
The overall approach first takes the width $n \to \infty$, showing that the coordinates of the feature and gradient vectors become asymptotically independent and are governed by their respective mean-field recursions. 
This argument can be made precise using classical propagation of chaos techniques, in particular, the synchronous coupling together with moment bounds and discrete Gronwall’s inequality. 
This establishes simplified finite-depth recursions. 
We then let the depth $L \to \infty$, interpret the Gaussian increments as Brownian motion increments, and recognize the resulting dynamics as Euler--Maruyama discretizations of McKean--Vlasov SDEs. 
By combining moment bounds with the Lipschitz continuity of the variance functional, we obtain explicit convergence rates along with the proof of existence and uniqueness of the limiting forward and backward SDEs stated in Theorem~\ref{thm:init-forward-backward-sde}.

\subsection{First forward}

Suppose $\phi$ satisfies the assumption in Theorem~\ref{thm:init-forward-backward-sde} throughout.
We will use the Gaussian representation in Proposition \ref{prop:layerwise-gaussian-representation}, that is,
\begin{align*}
    h_{\ell} = h_{\ell-1} + \frac{1}{\sqrt{Ln}} \norm{\phi(h_{\ell})} z_{\ell},
    \quad z_{\ell}\sim \Gaus(0, I).
\end{align*}
To distinguish from quantities after taking limits of $n \to \infty$ and $L \to \infty$, we add superscripts and write each coordinate as
\begin{align*}
    h^{n,L}_{\ell,i} = h^{n,L}_{\ell-1,i} + \frac{1}{\sqrt{Ln}} \norm{\phi(h^{n,L}_{\ell-1})} z_{\ell,i}.
\end{align*}
We want to show that each coordinate converges to 
\begin{align*}
    h^L_{\ell,i} = h^L_{\ell-1,i} + \frac{1}{\sqrt{L}} \sqrt{\E \phi^2(h^L_{\ell-1,i})} z_{\ell,i}
\end{align*}
as $n \to \infty$.
We will use $C$ to denote positive constants that depend only on $K_1$ and $\phi(0)$.
Its value may change from line to line.


\begin{lemma}
    \label{lem:h-L-moment-bound-first}
    For each $i \in \mathbb{N}$, $\displaystyle \sup_{L \ge 1} \sup_{\ell=0,\dotsc,L} \E[h_{\ell,i}^L]^4 < \infty$ and $\displaystyle \inf_{L \ge 1} \inf_{\ell=0,\dotsc,L} \E \phi^2(h_{\ell,i}^L) > 0$.
\end{lemma}

\begin{proof}[Proof of Lemma \ref{lem:h-L-moment-bound-first}]
    By symmetry, we only have to consider a fixed $i \in \mathbb{N}$. 
    Using independence of $z_{\ell,i}$, we have
    \begin{align*}
        \E [h_{\ell,i}^L]^4 & \le C\E [h_{0,i}^L]^4 + C\E \left[\sum_{u=1}^\ell \frac{1}{\sqrt{L}} \sqrt{\E \phi^2(h^L_{u-1,i})} z_{u,i} \right]^4 \\
        & \le C + \frac{C}{L}\sum_{u=1}^\ell \E \phi^4(h_{u-1,i}^L) \le C + \frac{C}{L} \sum_{u=0}^{\ell-1} \E [h_{u,i}^L]^4.
    \end{align*}
    It then follows from discrete Gronwall's inequality (Lemma \ref{lem:Gronwall-discrete}) that
    \begin{equation}
        \E [h_{\ell,i}^L]^4 \le C e^{C\ell/L}.
    \end{equation}
    This gives the first assertion. 

    For the second assertion, note that $\phi$ is a continuous function and not identically zero.
    So there exists some interval $(a,b) \subset \mathbb{R}$ such that $\inf_{a<x<b} \phi^2(x) > 0$. 
    Since $h_{\ell-1,i}^L$ and $z_{\ell,i}$ are independent, we have
    \begin{align*}
        \mbox{Var}(h_{\ell,i}^L) \ge \mbox{Var}(h_{\ell-1,i}^L) \ge \dotsb \ge \mbox{Var}(h_{0,i}^L) = C_1 > 0. 
    \end{align*}
    Also 
    \begin{equation*}
        \mbox{Var}(h_{\ell,i}^L) \le \E [h_{\ell,i}^L]^2 \le C_2.
    \end{equation*}    
    So $\{h_{\ell,i}^L\}$ are Gaussian random variables with mean zero and variance in $[C_1,C_2]$.
    Therefore, 
    \begin{equation*}
        \inf_{L \ge 1} \inf_{\ell=0,\dotsc,L} \mathbb{P}(h_{\ell,i}^L \in (a,b)) > 0.
    \end{equation*}
    This gives the second assertion.
\end{proof}

\begin{proposition}
    \label{prop:convergence-n-h0}
    For each $i \in \mathbb{N}$,
    \begin{align*}
        \sup_{L \ge 1} \sup_{\ell=0,\dotsc,L} \E (h^{n,L}_{\ell,i}-h^{L}_{\ell,i})^2
        \leq C/n.
    \end{align*}
\end{proposition}
\begin{proof}[Proof of Proposition \ref{prop:convergence-n-h0}]
Since $h^{n,L}_{0,i}=h^{L}_{0,i}$, we have
\begin{align*}
    \E (h^{n,L}_{\ell,i}-h^{L}_{\ell,i})^2
    & = \E \left[\sum_{u=1}^\ell \left( \frac{\norm{\phi(h^{n,L}_{u-1})}}{\sqrt{n}}  - \sqrt{\E \phi^2(h^L_{u-1,i})} \right) \frac{1}{\sqrt{L}}z_{u,i}\right]^2 \\
    & = \frac{C}{L} \E \sum_{u=1}^\ell \left( \frac{\norm{\phi(h^{n,L}_{u-1})}}{\sqrt{n}}  - \sqrt{\E \phi^2(h^L_{u-1,i})} \right)^2,
\end{align*} 
where the second line uses the fact that $\{z_{\ell,i}\}_\ell$ are independent standard normal random variables.
By adding and subtracting terms, we have
\begin{align}
    \E \left( \frac{\norm{\phi(h^{n,L}_{u-1})}}{\sqrt{n}}  - \sqrt{\E \phi^2(h^L_{u-1,i})} \right)^2
    & \le 2 \E \left( \sqrt{\frac{1}{n} \sum_{j=1}^n \phi^2(h^{n,L}_{u-1,j})} - \sqrt{\frac{1}{n} \sum_{j=1}^n \phi^2(h^{L}_{u-1,j})} \right)^2 \notag \\
    & \quad + 2 \E \left( \sqrt{\frac{1}{n} \sum_{j=1}^n \phi^2(h^{L}_{u-1,j})} - \sqrt{\E \phi^2(h^L_{u-1,i})} \right)^2. \label{eq:first-forward-pf-1}
\end{align}
For the first term on the right hand side, using Minkowski's inequality, we have
\begin{align*}
    &\E \left( \sqrt{\frac{1}{n} \sum_{j=1}^n \phi^2(h^{n,L}_{u-1,j})} - \sqrt{\frac{1}{n} \sum_{j=1}^n \phi^2(h^{L}_{u-1,j})} \right)^2 \\
    & \leq \E \left( \sqrt{\frac{1}{n} \sum_{j=1}^n \left( \phi(h^{n,L}_{u-1,j}) - \phi(h^{L}_{u-1,j}) \right)^2} \right)^2\\
    & \leq \frac{K_1^2}{n}\sum_{j=1}^{n} \E(h^{n,L}_{u-1,j}-h^L_{u-1,j})^2 = K_1^2 \E(h^{n,L}_{u-1,i}-h^L_{u-1,i})^2.
\end{align*}
For the second term on the right hand side of \eqref{eq:first-forward-pf-1}, we have
\begin{align*}
    &\E \left( \sqrt{\frac{1}{n} \sum_{j=1}^n \phi^2(h^{L}_{u-1,j})} - \sqrt{\E \phi^2(h^L_{u-1,i})} \right)^2\\
    &=\E \left(\frac{\frac{1}{n} \sum_{j=1}^n \phi^2(h^{L}_{u-1,j}) - \frac{1}{n} \sum_{j=1}^n \E\phi^2(h^{L}_{u-1,j})}{\sqrt{\frac{1}{n} \sum_{j=1}^n \phi^2(h^{L}_{u-1,j})} + \sqrt{\E \phi^2(h^L_{u-1,i})}}\right)^2\\
    &\leq C \frac{1}{n^2}\sum_{j=1}^{n} \E \left[\phi^2(h^{L}_{u-1,j}) - \E \phi^2(h^{L}_{u-1,j})\right]^2
    \leq C/n,
\end{align*}
where the last line uses the independence of $\{h^{L}_{u-1,j}\}_j$ and Lemma \ref{lem:h-L-moment-bound-first}. 
Therefore, we obtain
\begin{align*}
    \E (h^{n,L}_{\ell,i}-h^{L}_{\ell,i})^2
    \leq \frac{C}{L} \sum_{u=0}^{\ell-1} \E (h^{n,L}_{u,i}-h^{L}_{u,i})^2 
    +\frac{C}{n}.
\end{align*}
By discrete Gronwall's inequality (Lemma \ref{lem:Gronwall-discrete}), we have the desired result.
\end{proof}

Next, to analyze the limit of $h^L_{\ell,i}$ as $L \to \infty$, we omit the subscript $i$ and view 
\begin{equation*}
    \frac{1}{\sqrt{L}} z_{\ell} = W\left(\frac{\ell}{L}\right) - W\left(\frac{\ell-1}{L}\right)
\end{equation*}
for a standard Brownian motion $w$. 
Then we can write $h^L_{\ell} = h^{(L)}_{\ell/L}$, where
\begin{equation*}
    dh^{(L)}_t = \sqrt{\E \phi^2(h^{(L)}_{t_L})} \,dW_t
\end{equation*}
and $t_L := \frac{\lfloor tL \rfloor}{L}$ for $t \in [0,1]$.
Consider the McKean--Vlasov process
\begin{equation*}
    d H_t = \sqrt{\E \phi^2(H_t)} \,dW_t.
\end{equation*}
Then $\{h^{(L)}_t\}$ is just the Euler--Maruyama discretization for $\{h_t\}$ with step size $\Delta t = 1/L$.

The following is a standard result (see e.g.\ \citep[Section I.1]{Sznitman1991}) and we only provide a sketch of the proof.

\begin{proposition}
    \label{prop:well-posed-h0}
    There exists a unique $\{H_t\}$ and 
    \begin{equation}
        \label{eq:h-moment-bound-0}
        \sup_{0 \le t \le 1} \E [H_t^2] < \infty, \quad \E [H_t-H_{t_L}]^2 \le C(t-t_L) \le C/L.
    \end{equation}
\end{proposition}

\begin{proof}[Proof of Proposition \ref{prop:well-posed-h0}]
    The evolution of $H_t$ can be written as
    \begin{align*}
        dH_t=\sigma(U_t)\,d W_t, \quad H_0 \sim \Gaus(0,\norm{x}^2/d),
    \end{align*}
    where $U_t=\text{Law}(H_t)$ and
    \begin{equation*}
        \sigma(\nu) := \sqrt{\int \phi^2(x)\,\nu(dx)}
    \end{equation*}    
    for $\nu \in \mathcal{P}(\mathbb{R})$.
    Note that for any $X \sim U \in \mathcal{P}(\mathbb{R})$ and $Y \sim \nu \in \mathcal{P}(\mathbb{R})$, using Minkowski's inequality we have
    \begin{align*}
        |\sigma(U)-\sigma(\nu)| = |\sqrt{\E \phi^2(X)} - \sqrt{\E \phi^2(Y)}| \le \sqrt{\E[\phi(X)-\phi(Y)]^2}.
    \end{align*}
    By Lipschitz property of $\phi$, we have
    \begin{equation}
        \label{eq:uniqueness-sigma-Lipschitz}
        |\sigma(U)-\sigma(\nu)| \le C W_2(U,\nu),
    \end{equation}   
    where $W_2(\cdot,\cdot)$ is the Wasserstein metric on $\mathcal{P}(\mathbb{R})$.
    Therefore, $\sigma$ is a Lipschitz function.
    Now let 
    \begin{equation*}
        \mathcal{M}:=\{U \in \mathcal{P}(\mathbb{C}([0,1]:\mathbb{R})) : \sup_{0 \le t \le 1} \int x^2\,U_t(dx) < \infty \}.
    \end{equation*}
    For $U \in \mathcal{M}$, consider the process
    \begin{equation*}
        dX_t = \sigma(U_t)\,dW_t, \quad X_0=H_0.
    \end{equation*}
    It is well-defined and Law$(X) \in \mathcal{M}$, by Lipschitz property of $\phi$.
    Denote the map from $U \in \mathcal{M}$ to Law$(X) \in \mathcal{M}$ by $\Gamma$.
    For $U,\nu \in \mathcal{M}$, denote the Wasserstein metric by 
    \begin{equation*}
        W_{2,t}(U,\nu) := \inf \{ \left(\E[\sup_{u \le t} |X_u-Y_u|^2]\right)^{1/2} : \mbox{Law}(X)=U, \mbox{Law}(Y)=\nu\}.
    \end{equation*}
    Now given $U,\nu\in\mathcal{M}$, let
    \begin{equation*}
        dX_t = \sigma(U_t)\,d W_t, \quad dY_t = \sigma(\nu_t)\,d W_t, \quad X_0=Y_0=H_0.
    \end{equation*}
    Then, using Doob's maximal inequality, we have
    \begin{align*}
        W_{2,t}^2(\Gamma(U),\Gamma(\nu)) & \le \E[\sup_{u \le t} |X_u-Y_u|^2] 
        = \E[\sup_{u \le t} |\int_0^u [\sigma(U_s) - \sigma(\nu_s)] \,d W_s|^2] 
        \\
        & \le 4 \E |\int_0^t [\sigma(U_s) - \sigma(\nu_s)] \,d W_s|^2 = 4 \int_0^t [\sigma(U_s) - \sigma(\nu_s)]^2 \,ds 
        \\
        & \le C \int_0^t W_2^2(U_s,\nu_s) \,ds 
        \le C \int_0^t W_{2,s}^2(U,\nu) \,ds.
    \end{align*}    
    Existence and uniqueness of $\{H_t\}$ then follows from standard arguments (cf.\ \citep[Section I.1]{Sznitman1991}).  
    The first estimate in \eqref{eq:h-moment-bound-0} follows from standard arguments on observing that $\phi$ is Lipscthiz and hence has linear growth.
    From this we immediately get the second estimate in \eqref{eq:h-moment-bound-0}.
\end{proof}

The following result quantifies the error as $L \to \infty$.
This is not the stronger result one would usually get for Euler--Maruyama approximations.
But it is sufficient for our use and also will be used in later inductive arguments for traning steps.

\begin{proposition}
    \label{prop:convergence-L-h0}
    For all $L \ge 1$,
    \begin{equation*}
        \sup_{\ell=0,1,\dotsc,L} \E [h_\ell^L - H_{\ell/L}]^2 = \sup_{\ell=0,1,\dotsc,L} \E [h_{\ell/L}^{(L)} - H_{\ell/L}]^2 \le C/L.
    \end{equation*}
\end{proposition}

\begin{proof}[Proof of Proposition \ref{prop:convergence-L-h0}]
    Let $s_L := \frac{\lfloor sL \rfloor}{L}$.
    Since $h_0^{(L)} = H_0$, we have
    \begin{align*}
        & \E [h_\ell^L - H_{\ell/L}]^2 = \E [h_{\ell/L}^{(L)} - H_{\ell/L}]^2 \\
        & = \int_0^{\ell/L} |\sqrt{\E \phi^2(h_{s_L}^{(L)})} - \sqrt{\E \phi^2(H_s)}|^2\,ds \\
        & \le 2\int_0^{\ell/L} |\sqrt{\E \phi^2(h_{s_L}^{(L)})} - \sqrt{\E \phi^2(H_{s_L})}|^2\,ds + 2\int_0^{\ell/L} |\sqrt{\E \phi^2(H_{s_L})} - \sqrt{\E \phi^2(H_s)}|^2\,ds \\
        & \le \frac{C}{L} \sum_{u=0}^{\ell-1} \E [h_u^L - H_{u/L}]^2 + \frac{C}{L},
    \end{align*}
    where the last line uses the Lipschitz property in \eqref{eq:uniqueness-sigma-Lipschitz} and Lemma \ref{prop:well-posed-h0}.
    It then follows from discrete Gronwall's inequality (Lemma \ref{lem:Gronwall-discrete}) that
    \begin{equation*}
        \E [h_\ell^L - H_{\ell/L}]^2 \le \frac{C}{L} e^{C\ell/L}.
    \end{equation*}
    This completes the proof.
\end{proof}

Combining Propositions \ref{prop:convergence-n-h0} and \ref{prop:convergence-L-h0}, we get the forward SDE characterization in Theorem~\ref{thm:init-forward-backward-sde}.

\subsection{First Backward}

For each layer $\ell$, define
\begin{align}
    u_{\ell-1}
    :=
    \frac{\phi(h_{\ell-1})}{\|\phi(h_{\ell-1})\|},
    \qquad
    P_{\ell-1}:=u_{\ell-1}u_{\ell-1}^\top,
    \qquad
    z_\ell:=W_\ell u_{\ell-1}.
\end{align}
On the event $\phi(h_{\ell-1})=0$, we use the convention $u_{\ell-1}=0$ and $P_{\ell-1}=0$. Conditionally on the forward process, we have
\[
    W_\ell\mid\mathcal F_\ell
    \overset d=
    z_\ell u_{\ell-1}^\top
    +
    \widetilde W_\ell(I-P_{\ell-1}).
\]
Therefore, the backward process can be written as
\[
    g_{\ell-1}
    =
    g_\ell
    +
    \frac1{\sqrt L}D_{\ell-1}b_\ell,
    \label{eq:g-n-L}
\]
where
\[
    D_{\ell-1}:=\diag(\phi'(h_{\ell-1})),
    \qquad
    b_\ell
    :=
    \frac{z_\ell^\top g_\ell}{\sqrt n}u_{\ell-1}
    +
    \frac1{\sqrt n}
    (I-P_{\ell-1})\widetilde W_\ell^\top g_\ell.
\]

We compare $g_\ell$ with the auxiliary backward process with decoupled backward weights $\widetilde W_{\ell}$
\[
    \bar g_{\ell-1}
    =
    \bar g_\ell
    +
    \frac1{\sqrt L}D_{\ell-1}\bar b_\ell,
    \qquad
    \bar b_\ell
    :=
    \frac1{\sqrt n}\widetilde W_\ell^\top \bar g_\ell,
\]
with terminal condition $g_L=\bar g_L=v$, where $v$ is independent of the forward process.

\begin{proposition}
    \label{prop:convergence-n-g0}
    Assume $\phi'$ and $\phi''$ are bounded. Then
    \[
        \sup_{L\ge 1}
        \sup_{\ell=0,\dots,L}
        \frac{1}{n}
        \E\|g_\ell-\bar g_\ell\|^2
        \le
        \frac Cn.
    \]
\end{proposition}

\begin{proof}
Define
\[
    \delta_\ell
    :=
    \frac{1}{n}
    \E\|g_\ell-\bar g_\ell\|^2.
\]
Then
\[
    g_{\ell-1}-\bar g_{\ell-1}
    =
    g_\ell-\bar g_\ell
    +
    \frac1{\sqrt L}
    D_{\ell-1}(b_\ell-\bar b_\ell).
\]
Therefore
\begin{align}
    \delta_{\ell-1}
    =
    \delta_\ell
    +
    \frac{2}{n\sqrt L}
    \E\left[
        (g_\ell-\bar g_\ell)^\top
        D_{\ell-1}(b_\ell-\bar b_\ell)
    \right]
    +
    \frac1{Ln}
    \E\left[
        \|D_{\ell-1}(b_\ell-\bar b_\ell)\|^2
    \right].
\end{align}

We first bound the quadratic term. Since $\|D_{\ell-1}\|\le C$,
\[
    \frac{1}{n}
    \E\left[
        \|D_{\ell-1}(b_\ell-\bar b_\ell)\|^2
    \right]
    \le
    \frac Cn
    \E\|b_\ell-\bar b_\ell\|^2
    \le
    3I_1+3I_2+3I_3,
\]
where
\begin{align*}
    I_1
    &:=
    \frac{1}{n}
    \E\left[
        \left(
            \frac{z_\ell^\top g_\ell}{\sqrt n}
        \right)^2
        \|u_{\ell-1}\|^2
    \right],\\
    I_2
    &:=
    \frac{1}{n}
    \E\left\|
        \frac1{\sqrt n}
        \widetilde W_\ell^\top(g_\ell-\bar g_\ell)
    \right\|^2,\\
    I_3
    &:=
    \frac{1}{n}
    \E\left\|
        \frac1{\sqrt n}
        P_{\ell-1}\widetilde W_\ell^\top g_\ell
    \right\|^2.
\end{align*}
When $\phi(h_{\ell-1})=0$, both $I_1$ and $I_3$ vanish by convention. Otherwise,
$\|u_{\ell-1}\|=1$.

For $I_1$, Lemma~\ref{lem:backward-z-g-bound} gives
\[
    I_1
    =
    \frac1{n^2}
    \E[(z_\ell^\top g_\ell)^2]
    \le
    \frac Cn.
\]

For $I_2$, conditioning on all variables except the current innovation
$\widetilde W_\ell$ gives
\[
    \E_{\widetilde W_\ell}
    \left[
        \left\|
            \frac1{\sqrt n}
            \widetilde W_\ell^\top(g_\ell-\bar g_\ell)
        \right\|^2
    \right]
    =
    \|g_\ell-\bar g_\ell\|^2.
\]
Therefore
\[
    I_2
    =
    \frac{1}{n}
    \E\|g_\ell-\bar g_\ell\|^2
    =
    \delta_\ell.
\]

For $I_3$, conditionally on $g_\ell$ and $P_{\ell-1}$,
\[
    \frac1{\sqrt n}\widetilde W_\ell^\top g_\ell
    \overset d=
    \frac{\|g_\ell\|}{\sqrt n}\widetilde z_\ell,
    \qquad
    \widetilde z_\ell\sim N(0,I).
\]
Since $P_{\ell-1}$ has rank one,
\[
    \E\|P_{\ell-1}\widetilde z_\ell\|^2
    =
    \Tr(P_{\ell-1})
    =
    1.
\]
Thus
\[
    I_3
    =
    \frac1{n^2}
    \E\|g_\ell\|^2
    \le
    \frac Cn,
\]
where we used the bound $\E\|g_\ell\|^2\le Cn$ from Lemma~\ref{lem:backward-z-g-bound}.

Combining the estimates for $I_1,I_2,I_3$, we obtain
\[
    \frac1{Ln}
    \E\left[
        \|D_{\ell-1}(b_\ell-\bar b_\ell)\|^2
    \right]
    \le
    \frac CL\delta_\ell
    +
    \frac{C}{Ln}.
\]

For the cross term, Lemma~\ref{lem:backward-cross-term} gives
\[
    \frac{2}{n\sqrt L}
    \E\left[
        (g_\ell-\bar g_\ell)^\top
        D_{\ell-1}(b_\ell-\bar b_\ell)
    \right]
    \le
    \frac CL\delta_\ell
    +
    \frac{C}{Ln}.
\]
Hence
\[
    \delta_{\ell-1}
    \le
    \left(1+\frac CL\right)\delta_\ell
    +
    \frac{C}{Ln}.
\]
Since $\delta_L=0$, the backward discrete Gronwall inequality gives
\[
    \sup_{\ell=0,\dots,L}\delta_\ell
    \le
    \frac Cn.
\]
This proves the proposition.
\end{proof}
By exchangeability, Proposition~\ref{prop:convergence-n-g0} implies that, for a
typical coordinate $i$,
\[
    \sup_{L\ge 1}
    \sup_{\ell=0,\dots,L}
    \E\left[
        (g_{\ell,i}-\bar g_{\ell,i})^2
    \right]
    \le
    \frac Cn.
\]
Therefore, at the coordinate level, the true backward process and the auxiliary decoupled backward process are asymptotically equivalent as $n\to\infty$. Thus, to identify the limiting backward dynamics, it suffices to analyze the decoupled process $\bar g_\ell$.

Similar to the analysis of the first forward pass, we add superscripts and write
each coordinate as $\bar g_{\ell,i}^{n,L}$. For fixed $L$, define the
infinite-width discrete backward process by
\begin{equation}
    \label{eq:g-L-backward}
    \bar g_{\ell-1,i}^{L}
    =
    \bar g_{\ell,i}^{L}
    +
    \frac1{\sqrt L}
    \phi'(h_{\ell-1,i}^{L})
    \sqrt{\E[(\bar g_{\ell,i}^{L})^2]}
    \,\widetilde z_{\ell,i},
    \qquad
    \bar g_{L,i}^{L}=v_i,
\end{equation}
where $\{\widetilde z_{\ell,i}\}_{\ell,i}$ are i.i.d. standard Gaussian, independent of the limiting forward process. The following result shows $g_{\ell,i}^{n,L}$ converges to $\bar g_{\ell,i}^{L}$ as $n\to\infty$, and the proof is omitted here since it is the backward analog of Proposition~\ref{prop:convergence-n-h0}.

\begin{proposition}
\label{prop:convergence-n-bar-g0}
For each $i\in\mathbb N$,
\[
    \sup_{L\ge 1}
    \sup_{\ell=0,\dots,L}
    \E(
            \bar g_{\ell,i}^{n,L}
            -
            \bar g_{\ell,i}^{L}
        )^2
    \le C/n.
\]
\end{proposition}

Next, we analyze the limit of $\bar g_{\ell,i}^{L}$ as $L\to\infty$. We omit the
coordinate subscript $i$ and write
\[
    \frac1{\sqrt L}\widetilde z_\ell
    =
    \widetilde w\left(\frac{\ell-1}{L}\right)
    -
    \widetilde w\left(\frac{\ell}{L}\right),
\]
where $\widetilde w$ is a Brownian motion run backward in time. Thus
$\bar g_\ell^L$ is naturally viewed as the Euler--Maruyama approximation of the
backward-time McKean--Vlasov equation
\[
    dG_t
    =
    \phi'(H_t)\sqrt{\E[G_t^2]}\,d\widetilde w_t,
    \qquad
    G_1\sim \Gaus(0,1).
\]

\begin{lemma}
    \label{lem:g-L-moment-bound-first}
    For each $i\in\mathbb N$,
    \[
        \sup_{L\ge 1}
        \sup_{\ell=0,\dots,L}
        \E[(\bar g_{\ell,i}^{L})^4]
        <
        \infty,
    \]
    and
    \[
        \inf_{L\ge 1}
        \inf_{\ell=0,\dots,L}
        \E[(\bar g_{\ell,i}^{L})^2]
        >
        0.
    \]
\end{lemma}

\begin{proof}
The fourth-moment bound follows by the same induction as in
Lemma~\ref{lem:h-L-moment-bound-first}. Indeed, from
\[
    \bar g_{\ell-1,i}^{L}
    =
    \bar g_{\ell,i}^{L}
    +
    \frac1{\sqrt L}
    \phi'(h_{\ell-1,i}^{L})
    \sqrt{\E[(\bar g_{\ell,i}^{L})^2]}
    \widetilde z_{\ell,i},
\]
the boundedness of $\phi'$ and Gaussian moment estimates give
\[
    \E[(\bar g_{\ell-1,i}^{L})^4]
    \le
    \left(1+\frac CL\right)
    \E[(\bar g_{\ell,i}^{L})^4]
    +
    \frac CL.
\]
Since $\bar g_{L,i}^{L}=v_i$ and $\E[v_i^4]<\infty$, discrete Gronwall gives
\[
    \sup_{L\ge1}
    \sup_{\ell=0,\dots,L}
    \E[(\bar g_{\ell,i}^{L})^4]
    \le C.
\]

For the lower second-moment bound, using the same recursion and the centering of
$\widetilde z_{\ell,i}$,
\[
    \E[(\bar g_{\ell-1,i}^{L})^2]
    =
    \E[(\bar g_{\ell,i}^{L})^2]
    +
    \frac1L
    \E[(\phi'(h_{\ell-1,i}^{L}))^2]
    \E[(\bar g_{\ell,i}^{L})^2].
\]
Therefore
\[
    \E[(\bar g_{\ell-1,i}^{L})^2]
    \ge
    \E[(\bar g_{\ell,i}^{L})^2].
\]
Iterating backward from $\bar g_{L,i}^{L}=v_i$ gives
\[
    \E[(\bar g_{\ell,i}^{L})^2]
    \ge
    \E[v_i^2]
    >
    0.
\]
This proves the lemma.
\end{proof}

\begin{proposition}
    \label{prop:well-posed-g0}
    There exists a unique backward-time process $\{G_t:0\le t\le 1\}$ satisfying
    \[
        dG_t
        =
        \phi'(H_t)\sqrt{\E[G_t^2]}\,d\widetilde w_t,
        \qquad
        G_1\sim \Gaus(0,1).
    \]
    Moreover,
    \begin{equation}
        \label{eq:g-moment-bound-0}
        \sup_{0\le t\le 1}\E[G_t^2]<\infty,
        \qquad
        \E[(G_t-G_s)^2]\le C|t-s|,
        \qquad
        s,t\in[0,1].
    \end{equation}
    In particular,
    \[
        \E[(G_t-G_{\tilde t_L})^2]
        \le
        C|\tilde t_L-t|
        \le
        \frac CL,
    \]
    where $\tilde t_L:=\lceil tL\rceil/L$.
\end{proposition}

\begin{proposition}
    \label{prop:convergence-L-g0}
    For all $L\ge 1$,
    \[
        \sup_{\ell=0,\dots,L}
        \E\left[
            \left(
                \bar g_\ell^{L}
                -
                G_{\ell/L}
            \right)^2
        \right]
        \le
        \frac CL.
    \]
\end{proposition}

\begin{proposition}
\label{prop:convergence-nL-g0}
For each $i\in\mathbb N$,
\[
    \sup_{\ell=0,\dots,L}
    \E\left[
        \left(
            g_{\ell,i}^{n,L}
            -
            G_{\ell/L}
        \right)^2
    \right]
    \le
    C\left(
        \frac{1}{n}+\frac1L
    \right).
\]
\end{proposition}

\begin{proof}
By the triangle inequality,
\begin{align*}
    \E\left[
        \left(
            g_{\ell,i}^{n,L}
            -
            G_{\ell/L}
        \right)^2
    \right]
    \le
    3\E\left[
        \left(
            g_{\ell,i}^{n,L}
            -
            \bar g_{\ell,i}^{n,L}
        \right)^2
    \right]
    +
    3\E\left[
        \left(
            \bar g_{\ell,i}^{n,L}
            -
            \bar g_{\ell,i}^{L}
        \right)^2
    \right]  
    +
    3\E\left[
        \left(
            \bar g_{\ell,i}^{L}
            -
            G_{\ell/L}
        \right)^2
    \right].
\end{align*}
The first term is bounded by $C/n$ by Proposition~\ref{prop:convergence-n-g0}
and exchangeability. The second term is bounded by $C/n$ by
Proposition~\ref{prop:convergence-n-bar-g0}. The third term is bounded by
$C/L$ by Proposition~\ref{prop:convergence-L-g0}. 
\end{proof}

Proposition~\ref {prop:convergence-nL-g0} provides the first backward SDE characterization stated in Theorem~\ref{thm:init-forward-backward-sde}.

\subsection{Auxiliary Estimates for the Backward Pass}

We first have the following preparatory moment bound.

\begin{lemma}
    \label{lem:backward-z-g-bound}
Assume that $\phi'$ is bounded. Then there exists a constant
$C>0$ such that for every $\ell$,
\begin{align}
    \mathbb E [(z_{\ell}^{\top}g_{\ell})^2 \mid \mathcal{F}_{\ell-1}] \le Cn,
    \qquad
    \mathbb E\|g_{\ell}\|^2 \le Cn.
\end{align}
\end{lemma}
\begin{proof}
By the tower property,
\begin{align}
    \label{eq:backward-z-g-bound-pf}
    \mathbb E [(z_{\ell}^{\top}g_{\ell})^2 \mid \mathcal{F}_{\ell-1}]
    =
    \mathbb E\left[ u_{\ell-1}^{\top} W_\ell^\top \mathbb E \left(g_\ell g_\ell^\top \mid \mathcal{F}_\ell \right) W_\ell u_{\ell-1} \mid \mathcal{F}_{\ell-1} \right].
\end{align}

By the definition of $g_{\ell}$, we have $g_{\ell} = H_{\ell+1}^{\top} g_{\ell+1}$ with
\begin{align}
    H_k
    :=
    \frac{\partial h_k}{\partial h_{k-1}}
    =
    I+\frac{1}{\sqrt{Ln}}W_kD_{k-1},
    \qquad
    D_{k-1}:=\operatorname{diag}(\phi'(h_{k-1})).
\end{align}
Then we can write
\begin{align}
    g_\ell
    =
    H_{\ell+1}^{\top}\cdots H_L^{\top}g_L,
\end{align}
Since $\|D_{k-1}\|\le K$, by conditioning successively on
$W_L,W_{L-1},\ldots,W_{\ell+1}$, we obtain
\begin{align}
    \mathbb E[g_\ell g_\ell^\top\mid \mathcal F_\ell]
    \preceq
    \left(1+\frac{1}{L}K^2\right)^{L-\ell} I.
\end{align}
Applying this to \eqref{eq:backward-z-g-bound-pf} gives
\begin{align}
    \mathbb E [(z_{\ell}^{\top}g_{\ell})^2 \mid \mathcal{F}_{\ell-1}] \le
    \left(1+\frac{1}{L}K^2\right)^{L-\ell} u_{\ell-1}^{\top}
    \mathbb E\left[
    W_\ell^\top W_\ell \mid \mathcal{F}_{\ell-1}
    \right] u_{\ell-1}.
\end{align}
Since $W_\ell$ is independent of $\mathcal F_{\ell-1}$ and $\|u_{\ell-1}\|=1$,
\begin{align}
    u_{\ell-1}^{\top}
    \mathbb E\left[
    W_\ell^\top W_\ell \mid \mathcal{F}_{\ell-1}
    \right] u_{\ell-1}
    =
    n.
\end{align}
Therefore,
\begin{align}
    \mathbb E [(z_{\ell}^{\top}g_{\ell})^2 \mid \mathcal{F}_{\ell-1}]
    \le
    n\left(1+\frac{1}{L}K^2\right)^{L-\ell}
    \le
    ne^{K^2}.
\end{align}

Finally, the normalized second moment of $g_\ell$ follows similarly:
\begin{align}
    \mathbb E\|g_\ell\|^2
    =
    \mathbb E\operatorname{Tr}(g_\ell g_\ell^{\top})
    \le
    n\left(1+\frac{1}{L}K^2\right)^{L-\ell}
    \le
    n e^{K^2}.
\end{align}
The proof is complete.
\end{proof}

\begin{lemma}[Cross-term estimate]\label{lem:backward-cross-term}
For every $\ell\in[L]$,
\[
    \frac{1}{\sqrt{L}}
    \E[
        (g_\ell-\bar g_\ell)^\top
        D_{\ell-1}(b_\ell-\bar b_\ell)
    ]
    \le
    \frac CL
    \E\|g_\ell-\bar g_\ell\|^2
    +
    \frac CL.
\]
Equivalently,
\[
    \frac{1}{N\sqrt{L}}
    \E[
        (g_\ell-\bar g_\ell)^\top
        D_{\ell-1}(b_\ell-\bar b_\ell)
    ]
    \le
    \frac CL\delta_\ell
    +
    \frac{C}{Ln}.
\]
\end{lemma}

\begin{proof}
Let
\[
    v_{\ell-1}:=D_{\ell-1}u_{\ell-1}.
\]
Then $\|v_{\ell-1}\|\le C$. Conditioning on $\mathcal F_{\ell-1}$, the terms
in $b_\ell-\bar b_\ell$ that are linear in $\widetilde W_\ell$ are centered.
Thus, only the rank-one term contributes:
\[
    \frac{1}{\sqrt{L}}
    \E[
        (g_\ell-\bar g_\ell)^\top D_{\ell-1}(b_\ell-\bar b_\ell)
        \mid
        \mathcal F_{\ell-1}
    ]
    =
    \frac{1}{\sqrt{Ln}}
    v_{\ell-1}^\top
    \E[
        (g_\ell-\bar g_\ell) g_\ell^\top z_\ell
        \mid
        \mathcal F_{\ell-1}
    ].
\]
Conditionally on $\mathcal F_{\ell-1}$, $z_\ell$ is standard Gaussian and
\[
    h_\ell
    =
    h_{\ell-1}
    +
    \frac{1}{\sqrt{L}}
    \frac{\|\phi(h_{\ell-1})\|}{\sqrt n}z_\ell,
\]
with
\[
    \frac{\partial h_\ell}{\partial z_\ell}
    =
    \frac{1}{\sqrt{L}}
    \frac{\|\phi(h_{\ell-1})\|}{\sqrt n}I.
\]
Gaussian integration by parts gives
\begin{align*}
    &\frac{1}{\sqrt{Ln}}
    v_{\ell-1}^\top
    \E[
        (g_{\ell}-\bar g_{\ell}) g_\ell^\top z_\ell
        \mid
        \mathcal F_{\ell-1}
    ]\\
    =&
    \frac{1}{Ln}
    \|\phi(h_{\ell-1})\|
    v_{\ell-1}^\top
    \E[
        (J_\ell-\bar J_\ell)g_\ell
        +
        \Tr(J_\ell)(g_{\ell}-\bar g_{\ell})
        \mid
        \mathcal F_{\ell-1}
    ].
\end{align*}
where
\begin{align}
    J_{\ell}:=\frac{\partial g_{\ell}}{\partial h_{\ell}},
    \qquad
    \bar J_{\ell}:=\frac{\partial \bar g_{\ell}}{\partial h_{\ell}}.
\end{align}
Taking absolute values and then expectation, the cross term is bounded by
\[
    \frac{C}{Ln}
    \E[
        \|\phi(h_{\ell-1})\|
        |v_{\ell-1}^{\top} (J_\ell-\bar J_\ell) g_\ell|
    ]
    +
    \frac{C}{Ln}
    \E[
        \|\phi(h_{\ell-1})\|
        |\Tr(J_\ell)|\,|v_{\ell-1}^\top (g_\ell-\bar g_{\ell})|
    ],
\]

For the first Jacobian term, condition on $\mathcal F_{\ell-1}$, Cauchy-Schwarz gives
\[
    \E[
        |v_{\ell-1}^{\top} (J_\ell-\bar J_\ell) g_\ell|
        \mid
        \mathcal F_{\ell-1}
    ]
    \le
    \left(
        \E[
            \|(J_\ell-\bar J_\ell)^\top v_{\ell-1}\|^2
            \mid
            \mathcal F_{\ell-1}
        ]
    \right)^{1/2}
    \left(
        \E[
            \|g_\ell\|^2
            \mid
            \mathcal F_{\ell-1}
        ]
    \right)^{1/2}.
\]
By Lemma~\ref{lem:J-diff-direction} and Lemma~\ref{lem:backward-z-g-bound},
\[
    \E[
        |v_{\ell-1}^{\top} (J_\ell-\bar J_\ell) g_\ell|
        \mid
        \mathcal F_{\ell-1}
    ]
    \le C\sqrt{n}.
\]
Therefore,
\[
    \frac{C}{Ln}
    \E[
        \|\phi(h_{\ell-1})\|
        |v_{\ell-1}^{\top} (J_\ell-\bar J_\ell) g_\ell|
    ]
    \le
    \frac{C}{L\sqrt n}
    \E\|\phi(h_{\ell-1})\|
    \le
    \frac CL,
\]
where the last step uses Lemma~\ref{lem:h-L-moment-bound-first}.

For the trace term, condition again on $\mathcal F_{\ell-1}$, using Cauchy-Schwarz,
\begin{align*}
    &\E[
        |\Tr(J_\ell)|\,|v_{\ell-1}^\top (g_{\ell}-\bar g_{\ell})|
        \mid
        \mathcal F_{\ell-1}
    ]\\
    &\le
    \left(
        \E[
            (\Tr J_\ell)^2
            \mid
            \mathcal F_{\ell-1}
        ]
    \right)^{1/2}
    \left(
        \E[
            |v_{\ell-1}^\top (g_{\ell}-\bar g_{\ell})|^2
            \mid
            \mathcal F_{\ell-1}
        ]
    \right)^{1/2}.
\end{align*}
Using Lemma~\ref{lem:trace-J} and $\|v_{\ell-1}\|\le C$,
\[
    \E[
        |\Tr(J_\ell)|\,|v_{\ell-1}^\top (g_{\ell}-\bar g_{\ell})|
        \mid
        \mathcal F_{\ell-1}
    ]
    \le
    C\sqrt{n}
    \left(
        \E[
            \|g_{\ell}-\bar g_{\ell}\|^2
            \mid
            \mathcal F_{\ell-1}
        ]
    \right)^{1/2}.
\]
Hence, by Young's inequality, the trace contribution is 
\begin{align}
    &\frac{C}{L\sqrt n}
    \E\left[
        \|\phi(h_{\ell-1})\|
        \left(
            \E[
                \|g_{\ell}-\bar g_{\ell}\|^2
                \mid
                \mathcal F_{\ell-1}
            ]
        \right)^{1/2}
    \right]                                      \notag\\
    &\le
    \frac CL
    \E\left[
        \frac{\|\phi(h_{\ell-1})\|^2}{n}
    \right]
    +
    \frac CL
    \E[
        \|g_{\ell}-\bar g_{\ell}\|^2
    ]                                            \notag\\
    &\le
    \frac CL
    +
    \frac CL\E\|g_{\ell}-\bar g_{\ell}\|^2.
\end{align}
Combining the two estimates gives
\[
    \frac{1}{\sqrt{L}}
    \E[
        (g_{\ell}-\bar g_{\ell})^\top D_{\ell-1}(b_\ell-\bar b_\ell)
    ]
    \le
    \frac CL\E\|g_{\ell}-\bar g_{\ell}\|^2+\frac CL.
\]
Dividing by $n$ gives the normalized estimate.
\end{proof}

\begin{lemma}\label{lem:J-diff-direction}
Fix $\ell\in[L]$. Then
\[
    \E\left[
        \|(J_\ell-\bar J_\ell)^\top v_{\ell-1}\|^2
        \,\middle|\,
        \mathcal F_{\ell-1}
    \right]
    \le C.
\]
\end{lemma}

\begin{proof}
For $k\ge \ell$, define
\begin{align}
    H_k:=I+\frac{1}{\sqrt{Ln}}W_{k+1}D_k,
    \qquad
    \widetilde H_k:=I+\frac{1}{\sqrt{Ln}}\widetilde W_{k+1}D_k.  
\end{align}
Then
\begin{align}
    g_k=H_k^\top g_{k+1},
    \qquad
    \bar g_k=\widetilde H_k^\top \bar g_{k+1}.
\end{align}
Define
\begin{align}
    T_\ell:=I,
    \qquad
    T_{k+1}:=H_kT_k,
\end{align}
and
\begin{align}
    \widetilde T_\ell:=I,
    \qquad
    \widetilde T_{k+1}:=\widetilde H_k\widetilde T_k.
\end{align}
Thus
\begin{align}
    g_\ell=T_L^\top v,
    \qquad
    \bar g_\ell=\widetilde T_L^\top v.
\end{align}
Here $\widetilde T_k$ is an auxiliary product driven by
$\widetilde W_{k+1}$.

Furthermore, define
\begin{align}
    p_k:=T_k v_{\ell-1},
    \qquad
    \widetilde p_k:=\widetilde T_k v_{\ell-1},
\end{align}
and
\begin{align}
    M_k:=\frac{\partial p_k}{\partial h_\ell},
    \qquad
    \widetilde M_k:=\frac{\partial \widetilde p_k}{\partial h_\ell}.
\end{align}
Since $v_{\ell-1}$ is $\mathcal F_{\ell-1}$-measurable, it does not depend on $h_\ell$.
Therefore
\[
    M_\ell=\widetilde M_\ell=0.
\]

Differentiating
\[
    g_\ell-\bar g_\ell
    =
    (T_L^\top-\widetilde T_L^\top) v
\]
with respect to $h_\ell$ in the direction $v_{\ell-1}$ gives
\[
    (J_\ell-\bar J_\ell)^\top v_{\ell-1}
    =
    v^\top(M_L-\widetilde M_L).
\]
Since $v$ is an independent standard Gaussian,
\[
    \E_v
    \|v^\top(M_L-\widetilde M_L)\|^2
    =
    \|M_L-\widetilde M_L\|_F^2.
\]
Therefore,
\[
    \E\left[
        \|(J_\ell-\bar J_\ell)^\top v_{\ell-1}\|^2
        \,\middle|\,
        \mathcal F_{\ell-1}
    \right]
    =
    \E\left[
        \|M_L-\widetilde M_L\|_F^2
        \,\middle|\,
        \mathcal F_{\ell-1}
    \right].
\]
and so it is enough to prove
\[
    \sup_{k\ge \ell}
    \E[
        \|M_k\|_F^2
        \mid
        \mathcal F_{\ell-1}
    ]
    \le C,
    \qquad
    \sup_{k\ge \ell}
    \E[
        \|\widetilde M_k\|_F^2
        \mid
        \mathcal F_{\ell-1}
    ]
    \le C.
\]

We provide the proof of the estimate for $M_k$ here, as the estimate for $\widetilde M_k$ is identical. Since
\[
    p_{k+1}
    =
    p_k+\frac{1}{\sqrt{Ln}}W_{k+1}D_kp_k,
\]
differentiating with respect to $h_\ell$ gives
\[
    M_{k+1}
    =
    M_k
    +
    \frac{1}{\sqrt{Ln}}W_{k+1}
    \left[
        D_kM_k
        +
        \diag(p_k\odot\phi''(h_k))T_k
    \right].
\]
Conditioning on $\mathcal F_k$, $W_{k+1}$ is innovative and 
centered Gaussian, while the rest is fixed. Hence, the cross term vanishes and
\begin{align}
    \E[
        \|M_{k+1}\|_F^2
        \mid
        \mathcal F_k
    ]
    &=
    \|M_k\|_F^2
    +
    \frac{1}{L}
    \left\|
        D_kM_k
        +
        \diag(p_k\odot\phi''(h_k))T_k
    \right\|_F^2                                      \notag\\
    &\le
    \left(1+\frac CL\right)\|M_k\|_F^2
    +
    \frac CL\|\diag(p_k)T_k\|_F^2.
\end{align}
Taking conditional expectation with respect to $\mathcal F_{\ell-1}$ and using Lemma~\ref{lem:weighted-tangent},
\[
    \E[
        \|M_{k+1}\|_F^2
        \mid
        \mathcal F_{\ell-1}
    ]
    \le
    \left(1+\frac CL\right)
    \E[
        \|M_k\|_F^2
        \mid
        \mathcal F_{\ell-1}
    ]
    +
    \frac CL.
\]
Since $M_\ell=0$, discrete Gronwall gives
\[
    \sup_{k\ge \ell}
    \E[
        \|M_k\|_F^2
        \mid
        \mathcal F_{\ell-1}
    ]
    \le C.
\]

For the auxiliary estimate of $\widetilde M_{k}$,
\[
    \widetilde M_{k+1}
    =
    \widetilde M_k
    +
    \frac{1}{\sqrt{Ln}}\widetilde W_{k+1}
    \left[
        D_k\widetilde M_k
        +
        \diag(\widetilde p_k\odot\phi''(h_k))T_k
    \right].
\]
As $\widetilde W_{k+1}$ is an innovative and centered Gaussian, and the source term contains the same $T_k$ from the true forward path $h_k$, the same argument and Lemma~\ref{lem:weighted-tangent} give
\[
    \sup_{k\ge \ell}
    \E[
        \|\widetilde M_k\|_F^2
        \mid
        \mathcal F_{\ell-1}
    ]
    \le C.
\]
Consequently,
\[
    \E[
        \|M_L-\widetilde M_L\|_F^2
        \mid
        \mathcal F_{\ell-1}
    ]
    \le
    2\E[
        \|M_L\|_F^2
        \mid
        \mathcal F_{\ell-1}
    ]
    +
    2\E[
        \|\widetilde M_L\|_F^2
        \mid
        \mathcal F_{\ell-1}
    ]
    \le C.
\]
This proves the lemma.
\end{proof}

\begin{lemma}\label{lem:trace-J}
For every $\ell\in[L]$,
\[
    \E[
        (\Tr J_\ell)^2
        \mid
        \mathcal F_{\ell-1}
    ]
    \le Cn.
\]
\end{lemma}

\begin{proof}
Recall that
\[
    g_{\ell,a}
    =
    \frac{\partial}{\partial h_{\ell,a}}( v^\top h_L)
    =
    \sum_b  v_b
    \frac{\partial h_{L,b}}{\partial h_{\ell,a}}.
\]
Therefore
\[
    J_{\ell,ab}
    =
    \frac{\partial g_{\ell,a}}{\partial h_{\ell,b}}
    =
    \sum_{b'} v_{b'}
    \frac{\partial^2 h_{L,b'}}
    {\partial h_{\ell,a}\partial h_{\ell,b}}.
\]
Hence
\[
    \Tr(J_\ell)
    =
     v^\top
    \sum_a
    \frac{\partial^2 h_L}{\partial h_{\ell,a}^2}.
\]
Define
\[
    \tau_k
    :=
    \sum_a
    \frac{\partial^2 h_k}{\partial h_{\ell,a}^2},
    \qquad
    k\ge \ell.
\]
Then
\[
    \Tr(J_\ell)= v^\top\tau_L,
    \qquad
    \tau_\ell=0.
\]
Since $ v$ is an independent standard Gaussian from the forward pass,
\[
    \E[
        (\Tr J_\ell)^2
        \mid
        \mathcal F_{\ell-1}
    ]
    =
    \E[
        \|\tau_L\|^2
        \mid
        \mathcal F_{\ell-1}
    ].
\]

For each coordinate direction $a$, we have
\[
    \frac{\partial h_{k+1}}{\partial h_{\ell,a}}
    =
    \left(
        I+\frac{1}{\sqrt{Ln}}W_{k+1}D_k
    \right)
    \frac{\partial h_k}{\partial h_{\ell,a}}.
\]
Differentiating once more,
\begin{align}
    \frac{\partial^2 h_{k+1}}{\partial h_{\ell,a}^2}
    &=
    \frac{\partial^2 h_k}{\partial h_{\ell,a}^2}
    +
    \frac{1}{\sqrt{Ln}}W_{k+1}
    \left[
        D_k
        \frac{\partial^2 h_k}{\partial h_{\ell,a}^2}
        +
        \phi''(h_k)
        \odot
        \frac{\partial h_k}{\partial h_{\ell,a}}
        \odot
        \frac{\partial h_k}{\partial h_{\ell,a}}
    \right].
\end{align}
Summing over $a$ gives
\[
    \tau_{k+1}
    =
    \tau_k
    +
    \frac{1}{\sqrt{Ln}}W_{k+1}
    \left[
        D_k\tau_k
        +
        \phi''(h_k)\odot q_k
    \right],
\]
where
\[
    q_k
    :=
    \sum_a
    \frac{\partial h_k}{\partial h_{\ell,a}}
    \odot
    \frac{\partial h_k}{\partial h_{\ell,a}}.
\]
Conditioning on $\mathcal F_k$, $W_{k+1}$ is an innovative and 
centered Gaussian, while the rest is fixed. Therefore
\begin{align}
    \E[
        \|\tau_{k+1}\|^2
        \mid
        \mathcal F_k
    ]
    =
    \|\tau_k\|^2
    +
    \frac{1}{L}
    \left\|
        D_k\tau_k+\phi''(h_k)\odot q_k
    \right\|^2
    \le
    \left(1+\frac CL\right)\|\tau_k\|^2
    +
    \frac CL\|q_k\|^2.
\end{align}
Taking conditional expectation with respect to $\mathcal F_{\ell-1}$ and using
Lemma~\ref{lem:q_k},
\[
    \E[
        \|\tau_{k+1}\|^2
        \mid
        \mathcal F_{\ell-1}
    ]
    \le
    \left(1+\frac CL\right)
    \E[
        \|\tau_k\|^2
        \mid
        \mathcal F_{\ell-1}
    ]
    +
    \frac{Cn}{L}.
\]
Since $\tau_\ell=0$, discrete Gronwall gives
\[
    \E[
        (\Tr J_\ell)^2
        \mid
        \mathcal F_{\ell-1}
    ]
    =\E[
        \|\tau_L\|^2
        \mid
        \mathcal F_{\ell-1}
    ]
    \le Cn.
\]
\end{proof}

\begin{lemma}\label{lem:weighted-tangent}
Fix $\ell\in[L]$. 
Then, uniformly for $\ell\le k\le L$,
\[
    \E[
        \|p_k\|^4+\|\widetilde p_k\|^4
        \mid
        \mathcal F_{\ell-1}
    ]
    \le C,
\]
and
\[
    \E[
        \|\diag(p_k)T_k\|_F^2
        +
        \|\diag(\widetilde p_k)T_k\|_F^2
        \mid
        \mathcal F_{\ell-1}
    ]
    \le C.
\]
\end{lemma}

\begin{proof}
First, recall that, by definition, 
\[
    p_{k+1}
    =
    p_k+\frac{1}{\sqrt{Ln}}W_{k+1}D_kp_k.
\]
Conditioning on $\mathcal F_k$,
\[
    W_{k+1}D_kp_k
    \overset d=
    \|D_kp_k\|z_{k},
    \qquad
    z_{k}\sim N(0,I).
\]
Since $\|D_kp_k\|\le C\|p_k\|$, Gaussian fourth-moment estimates imply
\[
    \E[
        \|p_{k+1}\|^4
        \mid
        \mathcal F_k
    ]
    \le
    \left(1+\frac CL\right)\|p_k\|^4.
\]
Since $p_\ell=v_{\ell-1}$ and $\|v_{\ell-1}\|\le C$, discrete Gronwall gives
\[
    \sup_{k\ge \ell}
    \E[
        \|p_k\|^4
        \mid
        \mathcal F_{\ell-1}
    ]
    \le C.
\]
The proof for $\widetilde p_k$ is identical, using the centered Gaussian
innovation $\widetilde W_{k+1}$.

Let $\theta_i^\top$ be the $i$-th
row of $T_k$, and let $w_i^\top$ be the $i$-th row of $W_{k+1}$. Then
\[
    p_{k+1,i}
    =
    p_{k,i}
    +
    \frac{1}{\sqrt{Ln}}w_i^\top D_kp_k,
\]
and
\[
    (T_{k+1})_{i,\cdot}
    =
    \theta_i^\top
    +
    \frac{1}{\sqrt{Ln}}w_i^\top D_kT_k.
\]
Therefore,
\[
    \|\diag(p_{k+1})T_{k+1}\|_F^2
    =
    \sum_i
    p_{k+1,i}^2
    \|(T_{k+1})_{i,\cdot}\|^2.
\]
Expanding and conditioning on $\mathcal F_k$, all odd Gaussian terms vanish. The remaining are bounded by
\[
\begin{aligned}
    \E[
        \|\diag(p_{k+1})T_{k+1}\|_F^2
        \mid
        \mathcal F_k
    ]
    \le
    \left(1+\frac CL\right)
    \|\diag(p_k)T_k\|_F^2
    +
    \frac{C}{Ln}\|p_k\|^2\|T_k\|_F^2.
\end{aligned}
\]
Taking conditional expectation with respect to $\mathcal F_{\ell-1}$ and using
Cauchy--Schwarz,
\[
\begin{aligned}
    \E\left[
        \frac{1}{n}\|p_k\|^2\|T_k\|_F^2
        \,\middle|\,
        \mathcal F_{\ell-1}
    \right]
    \le
    \frac{1}{n}
    \left(
        \E[
            \|p_k\|^4
            \mid
            \mathcal F_{\ell-1}
        ]
    \right)^{1/2}
    \left(
        \E[
            \|T_k\|_F^4
            \mid
            \mathcal F_{\ell-1}
        ]
    \right)^{1/2}.
\end{aligned}
\]
By Lemma~\ref{lem:q_k},
\[
    \E[
        \Tr((T_kT_k^\top)^2)
        \mid
        \mathcal F_{\ell-1}
    ]
    \le Cn.
\]
Moreover,
\[
    \|T_k\|_F^4
    =
    [\Tr(T_kT_k^\top)]^2
    \le
    n\Tr((T_kT_k^\top)^2).
\]
Hence
\[
    \E[
        \|T_k\|_F^4
        \mid
        \mathcal F_{\ell-1}
    ]
    \le Cn^2.
\]
Therefore,
\[
    \E\left[
        \frac{1}{n}\|p_k\|^2\|T_k\|_F^2
        \,\middle|\,
        \mathcal F_{\ell-1}
    \right]
    \le C.
\]
Thus
\[
    \E[
        \|\diag(p_{k+1})T_{k+1}\|_F^2
        \mid
        \mathcal F_{\ell-1}
    ]
    \le
    \left(1+\frac CL\right)
    \E[
        \|\diag(p_k)T_k\|_F^2
        \mid
        \mathcal F_{\ell-1}
    ]
    +
    \frac CL.
\]
Since
\[
    \|\diag(p_\ell)T_\ell\|_F^2
    =
    \|\diag(v_{\ell-1})\|_F^2
    =
    \|v_{\ell-1}\|^2
    \le C,
\]
discrete Gronwall gives
\[
    \sup_{k\ge \ell}
    \E[
        \|\diag(p_k)T_k\|_F^2
        \mid
        \mathcal F_{\ell-1}
    ]
    \le C.
\]

The proof for $\widetilde p_k$ is the same. The only difference is that
$\widetilde p_{k+1}$ is updated with $\widetilde W_{k+1}$, while $T_{k+1}$ is updated with $W_{k+1}$. Conditioning on the joint past, $W_{k+1}$ and
$\widetilde W_{k+1}$ are centered Gaussian innovations. The same expansion gives
\[
    \sup_{k\ge \ell}
    \E[
        \|\diag(\widetilde p_k)T_k\|_F^2
        \mid
        \mathcal F_{\ell-1}
    ]
    \le C.
\]
\end{proof}

\begin{lemma}\label{lem:q_k}
Fix $\ell\in[L]$. For $k\ge \ell$, 
\[
    \E\left[
        \Tr\left((T_kT_k^\top)^2\right)
        \,\middle|\,
        \mathcal F_{\ell-1}
    \right]
    \le Cn,
    \qquad
    \E[
        \|q_k\|^2
        \mid
        \mathcal F_{\ell-1}
    ]
    \le Cn.
\]
\end{lemma}

\begin{proof}
Let
\[
    R_k:=T_kT_k^\top.
\]
Since
\[
    T_{k+1}
    =
    T_k+\frac{1}{\sqrt{Ln}}W_{k+1}D_kT_k,
\]
we have
\begin{align}
    R_{k+1}
    =
    R_k
    +
    \frac{1}{\sqrt{Ln}}
    \left(
        R_kD_kW_{k+1}^\top
        +
        W_{k+1}D_kR_k
    \right)
    +
    \frac{1}{Ln}
    W_{k+1}D_kR_kD_kW_{k+1}^\top .
\end{align}
Conditioning on $\mathcal F_k$, the matrices $R_k$ and $D_k$ are fixed, while $W_{k+1}$ is centered Gaussian and independent of $\mathcal F_k$. Therefore all
terms containing an odd number of copies of $W_{k+1}$ vanish in conditional
expectation.

Expanding $\Tr(R_{k+1}^2)$, under conditional expectation, gives
\begin{align*}
    &\E[
        \Tr(R_{k+1}^2)
        \mid
        \mathcal F_k
    ]                                                        \notag\\
    &=
    \Tr(R_k^2)                                               \notag\\
    &\quad
    +
    \frac{2}{Ln}
    \E\left[
        \Tr\left(
            R_kW_{k+1}D_kR_kD_kW_{k+1}^\top
        \right)
        \,\middle|\,
        \mathcal F_k
    \right]                                                  \notag\\
    &\quad
    +
    \frac{1}{Ln}
    \E\left[
        \left\|
            R_kD_kW_{k+1}^\top
            +
            W_{k+1}D_kR_k
        \right\|_F^2
        \,\middle|\,
        \mathcal F_k
    \right]                                                  \notag\\
    &\quad
    +
    \frac{1}{L^2n^2}
    \E\left[
        \Tr\left(
            \left(
                W_{k+1}D_kR_kD_kW_{k+1}^\top
            \right)^2
        \right)
        \,\middle|\,
        \mathcal F_k
    \right].
\end{align*}

For the first extra term, we have the identity
\[
    \E[
        W_{k+1}D_kR_kD_kW_{k+1}^\top
        \mid
        \mathcal F_k
    ]
    =
    \Tr(D_kR_kD_k)I.
\]
Therefore, the first extra term equals
\[
    \frac{2}{Ln}
    \Tr(R_k)\Tr(D_kR_kD_k).
\]
Since $\|D_k\|\le C$,
\[
    \Tr(D_kR_kD_k)\le C\Tr(R_k).
\]
Also, because $R_k\succeq0$,
\[
    [\Tr(R_k)]^2
    \le
    n\Tr(R_k^2).
\]
Thus
\[
    \frac{2}{Ln}
    \Tr(R_k)\Tr(D_kR_kD_k)
    \le
    \frac CL\Tr(R_k^2).
\]

For the second extra term, using Young's inequality and $\E[W_{k+1}^\top W_{k+1}]=\E[W_{k+1}W_{k+1}^\top]=n I$,
we obtain
\[
    \frac{1}{Ln}
    \E\left[
        \left\|
            R_kD_kW_{k+1}^\top
            +
            W_{k+1}D_kR_k
        \right\|_F^2
        \,\middle|\,
        \mathcal F_k
    \right]
    \le
    \frac CL\Tr(R_k^2).
\]

For the last term, we use the Gaussian fourth-moment identity: for deterministic
symmetric $B$,
\[
    \E\Tr((W_{k+1}BW_{k+1}^\top)^2)
    =
    n(n+1)\Tr(B^2)+n[\Tr(B)]^2.
\]
Applying this with
\[
    B=D_kR_kD_k,
\]
we get
\begin{align}
    &\E\left[
        \Tr\left(
            \left(
                W_{k+1}D_kR_kD_kW_{k+1}^\top
            \right)^2
        \right)
        \,\middle|\,
        \mathcal F_k
    \right] \notag\\
    &=
    n(n+1)\Tr((D_kR_kD_k)^2)
    +
    n[\Tr(D_kR_kD_k)]^2 \notag\\
    &\leq Cn^2 \Tr(R_k^2) \notag,
\end{align}
where we also use $\norm{D_k}\leq C$ and $[\Tr(R_k)]^2\leq n \Tr(R_k^2)$ through Cauchy-Schwarz inequality. 

Therefore,
\[
    \frac{T^2}{L^2n^2}
    \E\left[
        \Tr\left(
            \left(
                W_{k+1}D_kR_kD_kW_{k+1}^\top
            \right)^2
        \right)
        \,\middle|\,
        \mathcal F_k
    \right]
    \le
    \frac{C}{L^2}\Tr(R_k^2)
    \le
    \frac CL\Tr(R_k^2).
\]

Combining the three estimates gives
\[
    \E[
        \Tr(R_{k+1}^2)
        \mid
        \mathcal F_k
    ]
    \le
    \left(1+\frac CL\right)\Tr(R_k^2).
\]
Since $R_\ell=I$, iteration yields
\[
    \E[
        \Tr(R_k^2)
        \mid
        \mathcal F_{\ell-1}
    ]
    \le
    \left(1+\frac CL\right)^{k-\ell}\Tr(I)
    \le
    Cn.
\]
Finally,
\[
    q_k=\diag(T_k T_k^{\top}) = \diag(R_k),
\]
so
\[
    \|q_k\|^2
    =
    \|\diag(R_k)\|^2
    \le
    \|R_k\|_F^2
    =
    \Tr(R_k^2).
\]
Therefore,
\[
    \E[
        \|q_k\|^2
        \mid
        \mathcal F_{\ell-1}
    ]
    \le Cn.
\]
\end{proof}

\section{Width Convergence of Feature-Learning Dynamics via Tensor Programs}
\label{app:training-infinite-width-convergence}

In this section, we justify Proposition~\ref{prop:width-convergence-one-step}
and the width-convergence step in Theorem~\ref{thm:feature-learning-dynamics}.
The argument follows the Tensor Program analysis of feature-learning dynamics
\citep{yang2020feature}. We fix the depth $L$ and a finite training horizon
$K$, and take the width $n\to\infty$. For both the true SGD dynamics and
the auxiliary dynamics with decoupled backward weights, the finite-depth,
finite-horizon computation is a valid Tensor Program: features, gradients,
empirical inner products, and SGD updates are generated through matrix
multiplications, coordinate-wise nonlinearities, and moment operations. Hence
the infinite-width limit follows by applying the Tensor Program rules to the
SGD computation graph, with convergence guaranteed by the Master Theorem.
Instead of restating the general finite-$K$ construction, we analyze the
one-update case, which already captures the essential distinction: decoupled
backward weights yield a clean Gaussian mean-field dynamics, whereas true
backpropagation introduces an additional $(W_\ell,W_\ell^\top)$ correlation
term that is higher order in depth.

\paragraph{Decoupled-backward dynamics after one update.}
We first consider an auxiliary dynamics in which each backward matrix
$(W_\ell)^\top$ is replaced by an independent copy
$(\widetilde{W}_\ell)^\top$, while the forward pass and SGD updates keep the
same form. This is analogous to the partially decoupled backpropagation analysis
of \cite{yang2020feature}: the decoupling isolates the Gaussian CLT effect of
the $n\times n$ random matrices from the correlation effect between $W_\ell$ and $W_\ell^\top$.

Adding superscripts for the forward and backward passes, they become
\begin{align}
    h_0^{(0)}
    &=
    \frac{1}{\sqrt d}U^{(0)}x^{(0)}, \\
    h_\ell^{(0)}
    &=
    h_{\ell-1}^{(0)}
    +
    \frac{1}{\sqrt{Ln}}W_\ell^{(0)}\phi(h_{\ell-1}^{(0)}),
    \label{eq:tp-first-forward}
\end{align}
and, in the decoupled backward dynamics,
\begin{align}
    \bar{g}_L^{(0)}
    &=
    v^{(0)},\\
    \bar{g}_{\ell-1}^{(0)}
    &=
    \bar{g}_{\ell}^{(0)}
    +
    \frac{1}{\sqrt{Ln}}\,
    \phi'(h_{\ell-1}^{(0)})\odot
    (\widetilde{W}_\ell^{(0)})^\top \bar{g}_{\ell}^{(0)} .
    \label{eq:tp-first-backward-dec}
\end{align}
After one SGD update, the second forward pass with input $x^{(1)}$ satisfies
\begin{align}
    \bar{h}_0^{(1)}
    &=
    \frac{1}{\sqrt d}U^{(0)}x^{(1)}
    -
    \eta_c\chi^{(0)}
    \frac{\langle x^{(0)},x^{(1)}\rangle}{d}
    \bar{g}_0^{(0)},                                      \label{eq:tp-second-forward-0-dec}\\
    \bar{h}_\ell^{(1)}
    &=
    \bar{h}_{\ell-1}^{(1)}
    +
    \frac{1}{\sqrt{Ln}}W_\ell^{(0)}\phi(\bar{h}_{\ell-1}^{(1)})
    -
    \frac{\eta_c}{L}\chi^{(0)}
    \frac{
    \langle
    \phi(h_{\ell-1}^{(0)}),
    \phi(\bar{h}_{\ell-1}^{(1)})
    \rangle}{n}
    \bar{g}_\ell^{(0)} .
    \label{eq:tp-second-forward-dec}
\end{align}

For fixed $L$, the vectors in
\eqref{eq:tp-first-forward}--\eqref{eq:tp-second-forward-dec} are generated by
Tensor Program operations. Hence the Master Theorem gives joint coordinate
convergence and law-of-large-numbers limits for all empirical moments. In
particular,
\[
    \frac{1}{n}
    \left\langle
    \phi(h_{\ell-1}^{(0)}),
    \phi(\bar{h}_{\ell-1}^{(1)})
    \right\rangle
    \xrightarrow[n\to\infty]{a.s.}
    \E[
    \phi(H_{\ell-1}^{(0)})
    \phi(\bar H_{\ell-1}^{(1)})
    ].
\]
Moreover,
\[
    \frac{1}{\sqrt n}
    W_\ell^{(0)}\phi(\bar{h}_{\ell-1}^{(1)})
\]
converges coordinate-wise to a centered Gaussian innovation
$\bar A_\ell^{(1)}$, jointly Gaussian with the first-pass innovation
$A_\ell^{(0)}$, with covariance
\[
    \E[A_\ell^{(0)}\bar A_{\ell'}^{(1)}]
    =
    \delta_{\ell,\ell'}
    \E[
    \phi(H_{\ell-1}^{(0)})
    \phi(\bar H_{\ell-1}^{(1)})
    ] .
\]
Thus the decoupled infinite-width second forward dynamics are
\begin{align}
    \bar H_0^{(1)}
    &=
    \widehat H_0^{(1)}
    -
    \eta_c\mathring\chi^{(0)}
    \frac{\langle x^{(0)},x^{(1)}\rangle}{d}
    \bar G_0^{(0)},                                      \label{eq:tp-dec-limit-second-forward-0}\\
    \bar H_\ell^{(1)}
    &=
    \bar H_{\ell-1}^{(1)}
    +
    \frac{1}{\sqrt{L}}\,\bar A_\ell^{(1)}
    -
    \frac{\eta_c}{L}\mathring\chi^{(0)}
    \E[
    \phi(H_{\ell-1}^{(0)})
    \phi(\bar H_{\ell-1}^{(1)})
    ]
    \bar G_\ell^{(0)} ,
    \label{eq:tp-dec-limit-second-forward}
\end{align}
where $\mathring\chi^{(0)}=\mathcal L'(\mathring f^{(0)},y^{(0)})$.
Similarly, the decoupled second backward pass has the mean-field limit
\begin{align}
    \bar G_L^{(1)}
    &=
    \widehat G_L
    -
    \eta_c\mathring\chi^{(0)}H_L^{(0)},\\
    \bar G_{\ell-1}^{(1)}
    &=
    \bar G_\ell^{(1)}
    +
    \frac{1}{\sqrt{L}}\,
    \phi'(\bar H_{\ell-1}^{(1)})\bar B_\ell^{(1)}
    -
    \frac{\eta_c}{L}\mathring\chi^{(0)}
    \E[
    \bar G_\ell^{(0)}\bar G_\ell^{(1)}
    ]
    \phi(H_{\ell-1}^{(0)})
    \phi'(\bar H_{\ell-1}^{(1)}),
    \label{eq:tp-dec-limit-second-backward}
\end{align}
where $\bar B_\ell^{(1)}$ is centered Gaussian with covariance
\[
    \E[\bar B_\ell^{(0)}\bar B_{\ell'}^{(1)}]
    =
    \delta_{\ell,\ell'}
    \E[\bar G_\ell^{(0)}\bar G_\ell^{(1)}].
\]
Because the backward matrices $\widetilde{W}_\ell^{(0)}$ are independent of the forward matrices $W_\ell^{(0)}$, the forward innovation family and backward
innovation family are independent. This is the finite-depth Tensor Program
width limit underlying the NFD dynamics.

\paragraph{True shared-weight dynamics and the reused-weight correction.}
We now compare the decoupled dynamics with the true dynamics, where the first backward pass uses $(W_\ell^{(0)})^\top$ instead of
$(\widetilde{W}_\ell^{(0)})^\top$. The first forward pass is unchanged. The
difference appears in the second forward pass because the new features contain
an SGD drift proportional to the previous backward gradients, and these
gradients were computed using $(W_\ell^{(0)})^\top$.

To expose the mechanism, we first consider the linear-activation case
$\phi=\mathrm{id}$. The second forward pass contains
\[
    \frac{1}{\sqrt{Ln}}W_\ell^{(0)}h_{\ell-1}^{(1)} .
\]
The previous layer contains the SGD drift
\begin{align}
    -\frac{\eta_c}{L}\chi^{(0)}
    \frac{
    \langle
    h_{\ell-2}^{(0)},
    h_{\ell-2}^{(1)}
    \rangle}{n}
    g_{\ell-1}^{(0)} .
    \label{eq:true-second-forward-sgd-drift-linear}
\end{align}
Using the true first backward recursion,
\begin{align}
    g_{\ell-1}^{(0)}
    =
    g_\ell^{(0)}
    +
    \frac{1}{\sqrt{Ln}}
    (W_\ell^{(0)})^\top g_\ell^{(0)},
    \label{eq:true-first-backward-local-linear}
\end{align}
the part of \eqref{eq:true-second-forward-sgd-drift-linear} that depends on
$(W_\ell^{(0)})^\top$ is
\begin{align}
    -\frac{\eta_c}{L}\chi^{(0)}
    \frac{
    \langle
    h_{\ell-2}^{(0)},
    h_{\ell-2}^{(1)}
    \rangle}{n}
    \frac{1}{\sqrt{Ln}}
    (W_\ell^{(0)})^\top g_\ell^{(0)} .
    \label{eq:true-WT-dependent-part-linear}
\end{align}
Substituting this component into the forward multiplication by
$W_\ell^{(0)}$ gives 
\begin{align}
    -\frac{\eta_c}{L^2}\chi^{(0)}
    \frac{
    \langle
    h_{\ell-2}^{(0)},
    h_{\ell-2}^{(1)}
    \rangle}{n}
    \frac{1}{n}
    W_\ell^{(0)}
    (W_\ell^{(0)})^\top
    g_\ell^{(0)} .
    \label{eq:true-reuse-contraction-linear}
\end{align}
This is the additional term absent from the mean-field limit of the decoupled auxiliary process. In particular, if
$(W_\ell^{(0)})^\top$ is replaced by an independent copy
$(\widetilde{W}_\ell^{(0)})^\top$, then the corresponding contraction
$\frac{1}{n}W_\ell^{(0)}(\widetilde{W}_\ell^{(0)})^\top g_\ell^{(0)}$ has zero
diagonal mean and vanishes in the infinite-width limit.

We now evaluate \eqref{eq:true-reuse-contraction-linear}. By the Tensor Program
Master Theorem,
\begin{align}
    \frac{1}{n}
    \left\langle
    h_{\ell-2}^{(0)},
    h_{\ell-2}^{(1)}
    \right\rangle
    \xrightarrow[n\to\infty]{a.s.}
    \E[
    H_{\ell-2}^{(0)}
    H_{\ell-2}^{(1)}
    ].
    \label{eq:tp-sigma-limit-linear}
\end{align}
Moreover, for each coordinate $r$,
\begin{align}
    \left[
    \frac{1}{n}
    W_\ell^{(0)}
    (W_\ell^{(0)})^\top
    g_\ell^{(0)}
    \right]_r
    =
    g_{\ell,r}^{(0)}
    \frac{1}{n}\sum_{j=1}^{n}
    (W_{\ell,rj}^{(0)})^2
    +
    \mathrm{fluct.},
    \label{eq:tp-diagonal-contraction-linear}
\end{align}
where the off-diagonal fluctuation is centered and is absorbed into the Gaussian innovation in the mean-field limit. The diagonal term converges to
$G_\ell^{(0)}$. Therefore, in the linear case, the true one-step width limit
contains the deterministic reused-weight correction
\begin{align}
    \mathcal C_\ell^{(0\to 1)}
    =
    -\frac{\eta_c}{L^2}
    \mathring\chi^{(0)}
    \E[
    H_{\ell-2}^{(0)}
    H_{\ell-2}^{(1)}
    ]
    G_\ell^{(0)},
    \label{eq:tp-shared-correction-linear}
\end{align}
with the convention $\mathcal C_1^{(0\to 1)}=0$.

\paragraph{General activation.}
For a nonlinear activation, the same mechanism applies after expanding
$\phi(h_{\ell-1}^{(1)})$ in the $(W_\ell^{(0)})^\top$-dependent direction. The first-order term introduces the factor
\[
    \phi'(h_{\ell-1}^{(1)})
    \phi'(h_{\ell-1}^{(0)}),
\]
while the higher-order Taylor terms vanish in the $n\to\infty$ limit at fixed depth and fixed training horizon \cite{yang2020feature}. Thus, the nonlinear activation case gives
\begin{align}
    \mathcal C_\ell^{(0\to 1)}
    =
    -\frac{\eta_c}{L^2}
    \mathring\chi^{(0)}
    \E[
    \phi(H_{\ell-2}^{(0)})
    \phi(H_{\ell-2}^{(1)})
    ]
    \E[
    \phi'(H_{\ell-1}^{(0)})
    \phi'(H_{\ell-1}^{(1)})
    ]
    G_\ell^{(0)}.
    \label{eq:tp-shared-correction}
\end{align}
Hence, at fixed depth $L$, the true one-step width limit differs from the
decoupled width limit by the deterministic reused-weight correction
\eqref{eq:tp-shared-correction}. Its prefactor contains two depth factors, so
it is $O(L^{-2})$ per layer, whereas the leading SGD drift is $O(L^{-1})$ per
layer.

\paragraph{General finite training horizon.}
The same mechanism extends to any fixed training horizon $K<\infty$. After
$k$ SGD steps, the $k$-th forward and backward passes depend on weights updated
by the previous $k$ forward--backward computations. Hence, at each layer, the
reused-weight correction decomposes into pairwise contributions
\[
    \mathcal C_\ell^{(k)}
    =
    \sum_{i=0}^{k-1}\mathcal C_\ell^{(i\to k)},
\]
where $\mathcal C_\ell^{(i\to k)}$ denotes the correction induced by reusing
$W_\ell$ in the $k$-th forward pass and $W_\ell^\top$ in the $i$-th
backward pass. Each pairwise term is generated by the same local
$W_\ell(W_\ell)^\top$ contraction analyzed above and has per-layer depth
order $O(L^{-2})$. Thus, for fixed $k$, the omitted coupling accumulates
linearly over training steps, like the leading SGD drift. Since $K$ is fixed, this accumulation does not change the
depth order and remains one order
smaller in depth than the leading SGD drift. The following proposition summarizes the resulting finite-horizon
width limit.

\begin{proposition}[Finite-horizon width limit with reused-weight corrections]
\label{prop:finite-horizon-width-correction}
Suppose $\mathcal L'$, $\phi$, and $\phi'$ are Lipschitz continuous. Fix
$L<\infty$ and $K<\infty$. As $n\to\infty$, for each $k\le K$, the coordinates
of the finite-network training dynamics
$\{(h_\ell^{(k)},g_\ell^{(k)})\}_{\ell=0}^{L}$ become asymptotically
i.i.d. copies of a process
$\{(H_\ell^{(k)},G_\ell^{(k)})\}_{\ell=0}^{L}$ satisfying
\begin{align}
    H_{0}^{(k)}
    &=
    \widehat H_0^{(k)}
    -
    \eta_c
    \sum_{i=0}^{k-1}
    \chi^{(i)}
    \frac{\langle x^{(i)},x^{(k)}\rangle}{d}
    G_0^{(i)},
    \label{eq:finite-horizon-H0}
    \\
    H_{\ell}^{(k)}
    &=
    H_{\ell-1}^{(k)}
    +
    \frac{1}{\sqrt L}A_\ell^{(k)}
    -
    \frac{\eta_c}{L}
    \sum_{i=0}^{k-1}
    \chi^{(i)}
    \Sigma_{\ell-1}^{(i,k)}
    G_{\ell}^{(i)}
    +
    \mathcal C_{\ell,H}^{(k)},
    \label{eq:finite-horizon-H}
\end{align}
and
\begin{align}
    G_L^{(k)}
    &=
    \widehat G_L
    -
    \eta_c
    \sum_{i=0}^{k-1}
    \chi^{(i)}
    H_L^{(i)},
    \label{eq:finite-horizon-GL}
    \\
    G_{\ell-1}^{(k)}
    &=
    G_{\ell}^{(k)}
    +
    \frac{1}{\sqrt L}\phi'(H_{\ell-1}^{(k)})B_\ell^{(k)}
    -
    \frac{\eta_c}{L}
    \phi'(H_{\ell-1}^{(k)})
    \sum_{i=0}^{k-1}
    \chi^{(i)}
    \Theta_{\ell}^{(i,k)}
    \phi(H_{\ell-1}^{(i)})
    +
    \mathcal C_{\ell,G}^{(k)}.
    \label{eq:finite-horizon-G}
\end{align}
Here
\[
    \Sigma_{\ell}^{(i,k)}
    :=
    \E\!\left[
        \phi(H_{\ell}^{(i)})
        \phi(H_{\ell}^{(k)})
    \right],
    \qquad
    \Theta_{\ell}^{(i,k)}
    :=
    \E\!\left[
        G_{\ell}^{(i)}
        G_{\ell}^{(k)}
    \right],
\]
and
\[
    \chi^{(i)}
    :=
    \mathcal L'\!\left(\mathring f^{(i)},y^{(i)}\right),
    \qquad
    \mathring f^{(i)}
    :=
    \E[G_L^{(i)}H_L^{(i)}].
\]
The Gaussian input and terminal variables satisfy
\[
    \Cov(\widehat H_0^{(i)},\widehat H_0^{(k)})
    =
    \frac{\langle x^{(i)},x^{(k)}\rangle}{d},
    \qquad
    \widehat G_L\sim\mathcal N(0,1).
\]
The innovation families $\{A_\ell^{(k)}\}$ and $\{B_\ell^{(k)}\}$ are centered
Gaussian, independent across layers, with covariance
\begin{align}
    \E[A_\ell^{(i)}A_{\ell'}^{(k)}]
    &=
    \delta_{\ell,\ell'}\Sigma_{\ell-1}^{(i,k)},
    \label{eq:finite-horizon-A-cov}
    \\
    \E[B_\ell^{(i)}B_{\ell'}^{(k)}]
    &=
    \delta_{\ell,\ell'}\Theta_{\ell}^{(i,k)}.
    \label{eq:finite-horizon-B-cov}
\end{align}

The reused-weight correction in the forward recursion decomposes as
\[
    \mathcal C_{\ell,H}^{(k)}
    =
    \sum_{i=0}^{k-1}\mathcal C_{\ell,H}^{(i\to k)},
\]
where
\begin{align}
    \mathcal C_{\ell,H}^{(i\to k)}
    &:=
    -\frac{\eta_c}{L^2}\chi^{(i)}
    \E\!\left[
        \phi(H_{\ell-2}^{(i)})
        \phi(H_{\ell-2}^{(k)})
    \right]
    \E\!\left[
        \phi'(H_{\ell-1}^{(i)})
        \phi'(H_{\ell-1}^{(k)})
    \right]
    G_{\ell}^{(i)},
    \label{eq:finite-horizon-forward-correction}
\end{align}
with the convention $\mathcal C_{1,H}^{(i\to k)}=0$. Similarly, the
reused-weight correction in the backward recursion decomposes as
\[
    \mathcal C_{\ell,G}^{(k)}
    =
    \sum_{i=0}^{k-1}\mathcal C_{\ell,G}^{(i\to k)},
\]
where
\begin{align}
    \mathcal C_{\ell,G}^{(i\to k)}
    &:=
    -\frac{\eta_c}{L^2}\chi^{(i)}
    \phi'(H_{\ell-1}^{(k)})
    \phi(H_{\ell-1}^{(i)})
    \E\!\left[
        G_{\ell+1}^{(i)}
        G_{\ell+1}^{(k)}
    \right]
    \E\!\left[
        \phi'(H_{\ell}^{(i)})
        \phi'(H_{\ell}^{(k)})
    \right],
    \label{eq:finite-horizon-backward-correction}
\end{align}
with the convention that the boundary terms are omitted when the indices fall outside $\{0,\ldots,L\}$. 

Moreover, for the auxiliary dynamics with decoupled backward weights, the corresponding
limit
$\{(\bar H_\ell^{(k)},\bar G_\ell^{(k)})\}_{\ell=0}^{L}$
satisfies the same recursions but without the correction terms
$\mathcal C_{\ell,H}^{(k)}$ and $\mathcal C_{\ell,G}^{(k)}$.
\end{proposition}

This is the width-limit comparison used in
Theorem~\ref{thm:feature-learning-dynamics}; after summing over layers and using the moment bounds from Appendix~\ref{app:depth-convergence}, it yields
the coupling estimate in Corollary~\ref{cor:nfd-depth-rate}.

\section{Depth Convergence of Feature-Learning Dynamics via Stochastic Calculus}
\label{app:depth-convergence}

In this section, we use Proposition \ref{prop:finite-horizon-width-correction} to prove the convergence in Theorem \ref{thm:feature-learning-dynamics} and show that the rate of convergence as $L \to \infty$ is $1/L$.
Recall Assumption \ref{assmp:eigenvalue}.
For ease of writing, we consider the one sample case, and use lower case letters such as $h$ and $g$.

The limit as $n \to \infty$ for the $K$-th iteration can be written as
\begin{align}
    h_{\ell}^{(K),L} &= h_{\ell-1}^{(K),L} - \eta_0 \frac{1}{L} \sum_{k=0}^{K-1} \mathcal{L}^{\prime}(k,L) g_{\ell}^{(k),L} \E(\phi(h_{\ell-1}^{(k),L}) \phi(h_{\ell-1}^{(K),L})) + \frac{1}{\sqrt{L}} z_{\ell}^{(K),L} \notag \\
    &  \quad - \eta_0 \frac{1}{L^2} \sum_{k=0}^{K-1} \mathcal{L}^{\prime}(k,L) g_{\ell}^{(k),L} 
    \E(\phi(h_{\ell-2}^{(k),L}) \phi(h_{\ell-2}^{(K),L})) 
    \E(\phi'(h_{\ell-1}^{(k),L}) \phi'(h_{\ell-1}^{(K),L})), \notag \\
    g_{\ell-1}^{(K),L} &= g_{\ell}^{(K),L} - \eta_0 \frac{1}{L} \phi^{\prime}(h_{\ell-1}^{(K),L}) \sum_{k=0}^{K-1} \mathcal{L}^{\prime}(k,L) \phi(h_{\ell-1}^{(k),L}) \E(g_{\ell}^{(k),L}, g_{\ell}^{(K),L}) + \frac{1}{\sqrt{L}} \phi^{\prime}(h_{\ell-1}^{(K),L}) \tilde{z}_{\ell}^{(K),L} \notag \\
    & \quad -\eta_0 \frac{1}{L^2} \phi^{\prime}(h_{\ell-1}^{(K),L}) \sum_{k=0}^{K-1} \mathcal{L}'(k,L)\phi(h_{\ell-1}^{(k),L})  \E(g_{\ell+1}^{(k),L} g_{\ell+1}^{(K),L}) \E[\phi'(h_{\ell}^{(k),L}) \phi'(h_{\ell}^{(K),L})], \label{eq:g-tau-squared}
\end{align}
where $\{(z_{\ell}^{(k),L})_k, (\tilde{z}_{\ell}^{(k),L})_k : \ell=1,\dotsc,L\}$ are independent Gaussian random vectors with mean $0$ and variance-covariance matrix
\begin{align*}
    \mbox{Cov}(z_{\ell}^{(k),L},z_{\ell}^{(k'),L}) & = \E[\phi(h_{\ell-1}^{(k),L})\phi(h_{\ell-1}^{(k'),L})], \\
    \mbox{Cov}(\tilde{z}_{\ell}^{(k),L},\tilde{z}_{\ell}^{(k'),L}) & = \E[g_{\ell}^{(k),L} g_{\ell}^{(k'),L}],
\end{align*}
and
\begin{align*}
    h_{0}^{(K),L} & = \frac{\norm{x}}{\sqrt{d}} u(0) -\eta_0 \sum_{k=0}^{K-1} \mathcal{L}^{\prime}(k,L)  g_{0}^{(k),L}\frac{\inn{x}{x}}{d}, \\
    g_{L}^{(K),L} & = v(0) - \eta_0 \sum_{k=0}^{K-1} \mathcal{L}^{\prime}(k,L) h_{L}^{(k),L}, \\
    \mathcal{L}^{\prime}(k,L) & = \mathcal{L}^{\prime}(\E [g_{L}^{(k),L}h_L^{(k),L}], y),
\end{align*}
and $u(0),v(0)$ are standard Gaussian. 




Letting $L\rightarrow\infty$, we expect to have
\begin{align*}
    d h_t^{(K)} &= -\eta_0\sum_{k=0}^{K-1} \mathcal{L}^{\prime}(k) g_t^{(k)} \E [\phi(h_t^{(k)})\phi(h_t^{(K)})]dt + dw_t^{(K)}, \quad\forall t\in [0,1], \\
    d g_t^{(K)} &= -\eta_0\phi^{\prime}(h_t^{(K)}) \sum_{k=0}^{K-1} \mathcal{L}^{\prime}(k) \phi(h_t^{(k)}) \E [g_t^{(k)} g_t^{(K)}]dt + \phi^{\prime}(h_t^{(K)}) d\tilde{w}_t^{(K)},
    \quad\forall t\in [0,1].
\end{align*}
where $\{(w_t^{(k)})_k, (\tilde{w}_t^{(k)})_k : \ell=1,\dotsc,L\}$ are independent Brownian motions with mean $0$ and cross-variations
\begin{align*}
    d\langle w^{(k)},w^{(k')} \rangle_t & = \E[\phi(h_t^{(k)})\phi(h_t^{(k')})]\,dt, \\
    d\langle \tilde{w}^{(k)},\tilde{w}^{(k')} \rangle_t & = \E[g_t^{(k)}g_t^{(k')}]\,dt,
\end{align*}
and
\begin{align*}
    h_0^{(K)} & = \frac{\norm{x}}{\sqrt{d}} u(0) -\eta_0 \sum_{k=0}^{K-1} \mathcal{L}^{\prime}(k)  g_0^{(k)} \frac{\norm{x}^2}{d}, \\
    g_1^{(K)} & = v(0) - \eta_0 \sum_{k=0}^{K-1} \mathcal{L}^{\prime}(k) h_1^{(k)}, \\
    \mathcal{L}^{\prime}(k) & = \mathcal{L}^{\prime}(\E [g_1^{(k)}h_1^{(k)}], y).
\end{align*}

\begin{remark}
\label{rmk:well-FBSDE}
    (a) The evolution of $g_t^{(K)}$ is written for ease of notation and is interpreted backward from $t=1$ to $t=0$.
    This should not to be confused with the classic notion of backward stochastic differential equations.
    The precise meaning, instead, is that $(g_t,\tilde{w}_t) = (\widehat{g}_{1-t},\widehat{w}_{1-t})$ and
    \begin{equation*}
        d \widehat{g}_t^{(K)} = -\eta_0\phi^{\prime}(h_{1-t}^{(K)}) \sum_{k=0}^{K-1} \mathcal{L}^{\prime}(k) \phi(h_{1-t}^{(k)}) \E [\widehat{g}_t^{(k)} \widehat{g}_t^{(K)}]dt + \phi^{\prime}(h_{1-t}^{(K)}) d\widehat{w}_t^{(K)},
        \quad\forall t\in [0,1].
    \end{equation*}
    (b) The first forward $\{h_t^{(0)}\}$ is adapted to the driven Brownian motion. Due to the backpropagation, $g_t^{(0)}$ and $\{h_t^{(k)},g_t^{(k)}, k=1,2,\dotsc\}$ are not adapted any more.
    However, thanks to the deterministic diffusion coefficients in front of $dw_t$ and $d\tilde{w}_t$, which are automatically adapted, the SDEs are well-posed, as is justified in Propositions \ref{prop:well-posed-h0} and \ref{prop:well-posed-g0} above and Proposition \ref{prop:uniqueness} below.
\end{remark}

For ease of analysis, we write the above dynamics of $h_t := (h_t^{(0)},\dotsc,h_t^{(K)})$ and $g_t:= (g_t^{(0)},\dotsc,g_t^{(K)})$ in the following more standard manner of McKean--Vlasov equations: 
\begin{align*}
    d h_t &= b_t\,dt + \sigma_t\,dW_t, \\
    d g_t &= c_t\,dt + D_t\theta_t\,dB_t,
\end{align*}
where $b_t=(b_{t,k})_{k=0}^K$ and $c_t=(c_{t,k})_{k=0}^K$ are vectors given by
\begin{align*}
    b_{t,k} & = -\eta_0\sum_{i=0}^{k-1} \mathcal{L}^{\prime}(i) g_t^{(i)} \E [\phi(h_t^{(i)})\phi(h_t^{(k)})], \\
    c_{t,k} &:= -\eta_0\phi^{\prime}(h_t^{(k)}) \sum_{i=0}^{k-1} \mathcal{L}^{\prime}(i) \phi(h_t^{(i)}) \E [g_t^{(i)} g_t^{(k)}],
\end{align*}
$D_t$ is a diagonal matrix given by
\begin{align*}
    D_t &= \text{diag} \{\phi^{\prime}(h_t^{(0)}),\dotsc,\phi^{\prime}(h_t^{(K)})\},
\end{align*}
$\sigma_t$ and $\theta_t$ are (the Cholesky decomposition) such that
\begin{align*}
    \sigma_t \sigma_t^\top = \Sigma_t & := \E[(\phi(h_t^{(0)}),\dotsc,\phi(h_t^{(K)}))^\top (\phi(h_t^{(0)}),\dotsc,\phi(h_t^{(K)}))], \\
    \theta_t \theta_t^\top = \Theta_t & := \E[(g_t^{(0)},\dotsc,g_t^{(K)})^\top (g_t^{(0)},\dotsc,g_t^{(K)})],
\end{align*}
and $W_t$ and $B_t$ are independent $(K+1)$-dimensional standard Brownian motions.

We note that the above system is nested: when $K$ increases by $1$, one simply adds one additional dimension to the evolution of $h_t$ and $g_t$.
This allows us to apply induction arguments in the proofs of later results.
We also note that the existence of solutions to the above system is already guaranteed via the convergence of $n \to \infty$.

Denote by $\|\cdot\|$ the Frobenius norm of matrices (and vectors).
Denote by $\lambda_k(A)$ the eigenvalues of a symmetric matrix $A$.

\begin{lemma}
    \label{lem:h-moment-bound}
    If $\{h_t^{(k)},g_t^{(k)}, k=0,1,\dotsc,K\}$ is a solution, then
    \begin{equation*}
        \sup_{0 \le t \le 1} \E \|h_t\|^2 < \infty, \quad \sup_{0 \le t \le 1} \E \|g_t\|^2 < \infty.
    \end{equation*} 
\end{lemma}

\begin{proof}[Proof of Lemma \ref{lem:h-moment-bound}]
    We will prove by induction.
    The statement holds for $K=0$ by Propositions \ref{prop:well-posed-h0} and \ref{prop:well-posed-g0}.

    Now suppose the statement holds for $K$, namely
    \begin{equation*}
        \sup_{0 \le t \le 1} \sum_{k=0}^K \E [h_t^{(k)}]^2 < \infty, \quad \sup_{0 \le t \le 1} \sum_{k=0}^K \E [g_t^{(k)}]^2 < \infty.
    \end{equation*}
    We will show that
    \begin{equation*}
        \sup_{0 \le t \le 1} \E [h_t^{(K+1)}]^2 < \infty, \quad \sup_{0 \le t \le 1} \E [g_t^{(K+1)}]^2 < \infty.
    \end{equation*}
    For $h_t^{(K+1)}$, using Cauchy-Schwarz inequality, Lipschitz property of $\phi$, and induction assumption, we have
    \begin{align*}
        \E [h_t^{(K+1)}]^2 & \le C \E [h_0^{(K+1)}]^2 + C \sum_{k=0}^K \E \left( \int_0^t g_s^{(k)} \E [\phi(h_s^{(k)})\phi(h_s^{(K+1)})]\,ds \right)^2 + C \E [w_t^{(K+1)}]^2 \\
        & \le C + C \sum_{k=0}^K \int_0^t \E [g_s^{(k)}]^2 \E \phi^2(h_s^{(k)}) \E \phi^2(h_s^{(K+1)}) \,ds + C \int_0^t \E \phi^2(h_s^{(K+1)})\,ds \\
        & \le C + C \int_0^t \E [h_s^{(K+1)}]^2\,ds.
    \end{align*}
    It then follows from Gronwall's lemma that $\sup_{0 \le t \le 1} \E [h_t^{(K+1)}]^2 < \infty$.
    Since $\phi'$ is bounded, using similar arguments as above we can get $\sup_{0 \le t \le 1} \E [g_t^{(K+1)}]^2 < \infty$.
    Therefore the statement holds for $K+1$ and this completes the proof by induction.
\end{proof}

\begin{proposition}
    \label{prop:uniqueness}
    Suppose Assumption \ref{assmp:eigenvalue} holds. 
    Then pathwise uniqueness holds for $\{h_t^{(k)},g_t^{(k)}, k=0,1,\dotsc,K\}$.
\end{proposition}

\begin{proof}[Proof of Proposition \ref{prop:uniqueness}]
    %
    %
    %
    We will prove by induction. 
    By Propositions \ref{prop:well-posed-h0} and \ref{prop:well-posed-g0}, $h_t^{(0)}$ and $g_t^{(0)}$ are unique.
    So the statement holds for $K=0$.
    
    Now suppose the statement holds for $K$, namely $h_t^{(k)}$ and $g_t^{(k)}$, $k=0,1,\dotsc,K$, are unique.
    We will show that $h_t^{(k)}$ and $g_t^{(k)}$, $k=0,1,\dotsc,K+1$ are unique.
    Consider the solution $(h_t^{(k)}, g_t^{(k)})_{k=0}^{K+1}$ and any other solution $(\tilde{h}_t^{(k)}, \tilde{g}_t^{(k)})_{k=0}^{K+1}$.
    By the induction assumption on uniqueness, we must have $(h_t^{(k)}, g_t^{(k)})_{k=0}^{K} = (\tilde{h}_t^{(k)}, \tilde{g}_t^{(k)})_{k=0}^{K}$.
    Recall
    \begin{equation*}
        \sigma_t\sigma_t^\top = \Sigma_t = \E[(\phi(h_t^{(0)}),\dotsc,\phi(h_t^{(K+1)}))^\top (\phi(h_t^{(0)}),\dotsc,\phi(h_t^{(K+1)}))]
    \end{equation*}
    and let 
    \begin{align*}
        \tilde\sigma_t\tilde\sigma_t^\top = \tilde\Sigma_t & = \E[(\phi(\tilde h_t^{(0)}),\dotsc,\phi(\tilde h_t^{(K+1)}))^\top (\phi(\tilde h_t^{(0)}),\dotsc,\phi(\tilde h_t^{(K+1)}))] \\
        & = \E[(\phi(h_t^{(0)}),\dotsc,\phi(h_t^{(K)}),\phi(\tilde h_t^{(K+1)}))^\top (\phi(h_t^{(0)}),\dotsc,\phi(h_t^{(K)}),\phi(\tilde h_t^{(K+1)}))].
    \end{align*}
    Write $\Sigma_t$ in block matrix form
    \begin{align*}
        \Sigma_{11,t} & = \E[(\phi(h_t^{(0)}),\dotsc,\phi(h_t^{(K)}))^\top (\phi(h_t^{(0)}),\dotsc,\phi(h_t^{(K)}))], \\
        \Sigma_{21,t} & =\Sigma_{12,t}^\top = \E[\phi(h_t^{(K+1)}) (\phi(h_t^{(0)}),\dotsc,\phi(h_t^{(K)}))], \\
        \Sigma_{22,t} & = \E[\phi^2(h_t^{(K+1)})],
    \end{align*}
    corresponding to coordinates $0,1,\dotsc,K$ and $K+1$. 
    Also write $\sigma_t$, $\tilde\Sigma_t$ and $\tilde\sigma_t$ is the similar way.
    Then by Cholesky decomposition, we have
    \begin{align*}
        \sigma_{11,t} \sigma_{11,t}^\top & = \Sigma_{11,t}, & \tilde\sigma_{11,t} & = \sigma_{11,t}, \\
        \sigma_{12,t} &= 0, & \tilde\sigma_{12,t} &= 0, \\
        \sigma_{21,t}^\top &= \sigma_{11,t}^{-1} \Sigma_{12,t}, &  \tilde\sigma_{21,t}^\top &= \tilde\sigma_{11,t}^{-1} \tilde\Sigma_{12,t} = \sigma_{11,t}^{-1} \tilde\Sigma_{12,t}, \\
        \sigma_{22,t} &= \sqrt{\Sigma_{22,t} - \sigma_{21,t}\sigma_{21,t}^\top}, & \tilde\sigma_{22,t} &= \sqrt{\tilde\Sigma_{22,t} - \tilde\sigma_{21,t}\tilde\sigma_{21,t}^\top}.
    \end{align*}
    By Lipschitz property of $\phi$, we have
    \begin{equation}
        \label{eq:uniqueness-Sigma-difference}
        \|\Sigma_{12,t} - \tilde\Sigma_{12,t}\|^2 +  \|\Sigma_{22,t} - \tilde\Sigma_{22,t}\|^2 \le C \E[h_t^{(K+1)}-\tilde h_t^{(K+1)}]^2.
    \end{equation} 
    By Assumption \ref{assmp:eigenvalue}, there exits some $arepsilon>0$ such that all eigenvalues of $\Sigma_t$ are at least $arepsilon$.
    It then follows from the eigenvalue interlacing theorem (of principal submatrix) that all eigenvalues of $\Sigma_{11,t}$ are at least $arepsilon$. 
    Then we have
    \begin{equation}
        \label{eq:uniqueness-sigma11-bound}
        \|\sigma_{11,t}^{-1}\|^2 = \text{trace}((\sigma_{11,t}^{-1})^\top\sigma_{11,t}^{-1}) = \text{trace}(\Sigma_{11,t}^{-1}) = \sum_{k=0}^K \frac{1}{\lambda_k(\Sigma_{11,t})} \le \frac{K+1}{arepsilon}.
    \end{equation}
    Therefore
    \begin{equation}
        \|\sigma_{21,t}-\tilde\sigma_{21,t}\|^2
        = \|\sigma_{11,t}^{-1} (\Sigma_{12,t} - \tilde\Sigma_{12,t})\|^2 \le \|\sigma_{11,t}^{-1}\|^2 \|\Sigma_{12,t} - \tilde\Sigma_{12,t}\|^2
        \le C \E[h_t^{(K+1)}-\tilde h_t^{(K+1)}]^2. \label{eq:uniqueness-sigma21}
    \end{equation}  
    where the last inequality uses \eqref{eq:uniqueness-sigma11-bound} and \eqref{eq:uniqueness-Sigma-difference}.
    Similarly,
    \begin{align}
        & |\sigma_{21,t}\sigma_{21,t}^\top - \tilde\sigma_{21,t}\tilde\sigma_{21,t}^\top|^2 = |(\sigma_{21,t}-\tilde\sigma_{21,t}) (\sigma_{21,t}+\tilde\sigma_{21,t})^\top|^2 \notag \\
        & = |(\sigma_{21,t}-\tilde\sigma_{21,t}) \sigma_{11,t}^{-1} (\Sigma_{12,t}+\tilde\Sigma_{12,t})|^2
        \le \|\sigma_{21,t}-\tilde\sigma_{21,t}\|^2 \|\sigma_{11,t}^{-1}\|^2 \|\Sigma_{12,t}+\tilde\Sigma_{12,t}\|^2 \notag \\
        & \le C \E[h_t^{(K+1)}-\tilde h_t^{(K+1)}]^2, \label{eq:uniqueness-sigma21-square}
    \end{align}   
    where we have used Lemma \ref{lem:h-moment-bound} to get $\|\Sigma_{12,t}+\tilde\Sigma_{12,t}\|^2 \le C$.
    Also, note that
    \begin{equation*}
        \sigma_{22,t}^2 = \Sigma_{22,t} - \sigma_{21,t}\sigma_{21,t}^\top = \Sigma_{22,t} - \Sigma_{21,t}\Sigma_{11,t}^{-1}\Sigma_{12,t}
    \end{equation*}
    is the Schur complement of the block $\Sigma_{11,t}$ of the matrix $\Sigma_t$, so that its eigenvalues are at least $arepsilon$ as well.
    Therefore $\sigma_{22,t} \ge \sqrt{arepsilon}$ and hence
    \begin{align}
        \|\sigma_{22,t}-\tilde\sigma_{22,t}\|^2 & = \left(\frac{\sigma_{22,t}^2-\tilde\sigma_{22,t}^2}{\sigma_{22,t}+\tilde\sigma_{22,t}}\right)^2 \le \frac{1}{arepsilon} [(\Sigma_{22,t} - \sigma_{21,t}\sigma_{21,t}^\top)-(\tilde\Sigma_{22,t} - \tilde\sigma_{21,t}\tilde\sigma_{21,t}^\top)]^2 \notag \\
        & \le C|\Sigma_{22,t}-\tilde\Sigma_{22,t}|^2 + C|\sigma_{21,t}\sigma_{21,t}^\top - \tilde\sigma_{21,t}\tilde\sigma_{21,t}^\top|^2 \le C \E[h_t^{(K+1)}-\tilde h_t^{(K+1)}]^2, \label{eq:uniqueness-sigma22}
    \end{align}
    where the last inequality uses \eqref{eq:uniqueness-Sigma-difference} and \eqref{eq:uniqueness-sigma21-square}.   
    Now note that
    \begin{align*}
        h_u^{(K+1)}-\tilde{h}_u^{(K+1)} & = -\eta_0\sum_{k=0}^{K}\int_0^u \mathcal{L}^{\prime}(k) g_s^{(k)} \E [\phi(h_s^{(k)})(\phi(h_s^{(K+1)})-\phi(\tilde h_s^{(K+1)}))] \,ds \\
        & \quad + \int_0^u (\sigma_{21,s} - \tilde\sigma_{21,s}, \sigma_{22,s} - \tilde\sigma_{22,s}) \,dW_s.
    \end{align*}
    Therefore
    \begin{align}
        \E[\sup_{u \le t} |h_u^{(K+1)}-\tilde{h}_u^{(K+1)}|^2]
        & \le C \sum_{k=0}^{K} \E\left[\sup_{u \le t} \left|\int_0^u \mathcal{L}^{\prime}(k) g_s^{(k)} \E [\phi(h_s^{(k)})(\phi(h_s^{(K+1)})-\phi(\tilde h_s^{(K+1)}))] \,ds\right|^2\right] \notag \\
        & \quad + C \E\left[\sup_{u \le t} \left| \int_0^u (\sigma_{21,s} - \tilde\sigma_{21,s}, \sigma_{22,s} - \tilde\sigma_{22,s}) \,dW_s \right|^2\right]. \label{eq:uniqueness-pf}
    \end{align}
    Here using Cauchy-Schwarz inequality, we can bound the first term on the right side by
    \begin{align*}
        & C \sum_{k=0}^{K} \E\int_0^t \left|g_s^{(k)} \E [\phi(h_s^{(k)})(\phi(h_s^{(K+1)})-\phi(\tilde h_s^{(K+1)}))]\right|^2ds \\
        & \le C \int_0^t \E [\phi(h_s^{(K+1)})-\phi(\tilde h_s^{(K+1)})]^2\,ds \le C \int_0^t \E[\sup_{u \le s} |h_u^{(K+1)}-\tilde{h}_u^{(K+1)}|^2] \,ds.
    \end{align*}
    Using Doob's maximal inequality, we can bound the second term on the right side of \eqref{eq:uniqueness-pf} by
    \begin{align*}
        & C \E \left| \int_0^t (\sigma_{21,s} - \tilde\sigma_{21,s}, \sigma_{22,s} - \tilde\sigma_{22,s}) \,dW_s \right|^2
        = C \int_0^t [\|\sigma_{21,s} - \tilde\sigma_{21,s}\|^2 + \|\sigma_{22,s} - \tilde\sigma_{22,s}\|^2] \,ds \\
        & \quad \le C \int_0^t \E[h_s^{(K+1)}-\tilde h_s^{(K+1)}]^2 \,ds
        \le C \int_0^t \E[\sup_{u \le s} |h_s^{(K+1)}-\tilde h_s^{(K+1)}|^2] \,ds.
    \end{align*}
    where the first inequality uses \eqref{eq:uniqueness-sigma21} and \eqref{eq:uniqueness-sigma22}.
    Combining above three estimates, we have
    \begin{align*}
        \E[\sup_{u \le t} |h_u^{(K+1)}-\tilde{h}_u^{(K+1)}|^2] \le C \int_0^t \E[\sup_{u \le s} |h_s^{(K+1)}-\tilde h_s^{(K+1)}|^2] \,ds.
    \end{align*}
    It then follows from Gronwall's inequality that
    \begin{equation*}
        \E[\sup_{u \le 1} |h_u^{(K+1)}-\tilde{h}_u^{(K+1)}|^2] = 0.
    \end{equation*}    
    This gives uniqueness of $h_t^{(K+1)}$.
    Since $\phi'$ is bounded, similar arguments as above give uniqueness of $g_t^{(K+1)}$.
    Therefore the statement holds for $K+1$ and this completes the proof by induction.
\end{proof}

Before proving the convergence rate as $L \to \infty$, we will need the following two preparation results.
Recall $t_L:=\lfloor tL \rfloor /L$ and $\tilde{t}_L:=\lceil tL \rceil /L$ are the times corresponding to the discrete step.

\begin{lemma}
    \label{lem:h-fluctuation}
    If $\{h_t^{(k)},g_t^{(k)}, k=0,1,\dotsc,\kappa\}$ is a solution, then
    \begin{equation*}
        \E \|h_t-h_{t_L}\|^2 \le C(t-t_L) \le C/L, \quad \E \|g_t-g_{\tilde{t}_L}\|^2 \le C(\tilde{t}_L-t) \le C/L.
    \end{equation*}
\end{lemma}

\begin{proof}[Proof of Lemma \ref{lem:h-fluctuation}]
    Using Lemma \ref{lem:h-moment-bound} and Lipscthiz property of $\phi$, we can deduce
    \begin{align*}
        \E \|b_t\|^2 \le C, \qquad \|\sigma_t\|^2 = \text{trace}(\sigma_t\sigma_t^\top) = \text{trace}(\Sigma_t) \le C,
    \end{align*}
    and similarly $\E \|c_t\|^2 \le C$ and $\|\theta_t\|^2 \le C$.
    These give the desired result.
\end{proof}

Recall the infinite width limit $h_\ell^L := (h_t^{(0),L},\dotsc,h_t^{(K),L})$ and $g_\ell^L := (g_t^{(0),L},\dotsc,g_t^{(K),L})$.

\begin{lemma}
    \label{lem:h-L-moment-bound}
    $\displaystyle \sup_{L \ge 1} \sup_{\ell=0,\dotsc,L} \E\|h_\ell^L\|^2 < \infty$ and $\displaystyle \sup_{L \ge 1} \sup_{\ell=0,\dotsc,L} \E\|g_\ell^L\|^2 < \infty$.
\end{lemma}

\begin{proof}[Proof of Lemma \ref{lem:h-L-moment-bound}]
    We will prove by induction. 
    By Lemmas \ref{lem:h-L-moment-bound-first} and \ref{lem:g-L-moment-bound-first}, the statement holds for $K=0$. 

    Now suppose the statement holds for $K$, namely  
    \begin{equation*}
        \sup_{L \ge 1} \sup_{\ell=0,\dotsc,L} \sum_{k=0}^K \E [h_\ell^{(k),L}]^2 < \infty, \quad \sup_{L \ge 1} \sup_{\ell=0,\dotsc,L} \sum_{k=0}^K \E [g_\ell^{(k),L}]^2 < \infty.
    \end{equation*}
    We will show that
    \begin{equation*}
        \sup_{L \ge 1} \sup_{\ell=0,\dotsc,L} \E [h_\ell^{(K+1),L}]^2 < \infty, \quad \sup_{L \ge 1} \sup_{\ell=0,\dotsc,L} \E [g_\ell^{(K+1),L}]^2 < \infty.
    \end{equation*}
    From the evolution of $h_\ell^{(K+1),L}$ and independence of $z_\cdot^{(K+1),L}$, we have
    \begin{align*}
        & \E [h_\ell^{(K+1),L}]^2 \\
        & \le 4\E [h_0^{(K+1),L}]^2 + 4\E \left[\sum_{u=1}^\ell \frac{1}{\sqrt{L}} z_u^{(K+1),L} \right]^2 \\
        & \quad + 4\E \left[ \sum_{u=1}^\ell \eta_0 \frac{1}{L} \sum_{k=0}^{K} \mathcal{L}^{\prime}(k,L) g_{u}^{(k),L} \E(\phi(h_{u-1}^{(k),L}) \phi(h_{u-1}^{(K+1),L})) \right]^2 \\
        & + 4\E \left[ \sum_{u=1}^\ell \eta_0 \frac{1}{L^2} \sum_{k=0}^{K} \mathcal{L}^{\prime}(k,L) g_{u}^{(k),L} \E(\phi(h_{u-2}^{(k),L}) \phi(h_{u-2}^{(K+1),L})) \E(\phi^{\prime}(h_{u-1}^{(k),L}) \phi^{\prime}(h_{u-1}^{(K+1),L})) \right]^2 \\
        & \le C + \frac{C}{L}\sum_{u=1}^\ell \E \phi^2(h_{u-1}^{(K+1),L}) + \frac{C\ell}{L^2}\sum_{u=1}^\ell \E \phi^2(h_{u-1}^{(K+1),L}) \\
        & \quad + \frac{C\ell}{L^4}\sum_{u=1}^\ell \left[ \E \phi^2(h_{u-2}^{(K+1),L}) + \E [\phi^{\prime}(h_{u-1}^{(K+1),L})]^2 \right] \\
        & \le C + \frac{C}{L} \sum_{u=0}^{\ell-1} \E [h_{u}^{(K+1),L}]^2.
    \end{align*}
    It then follows from discrete Gronwall's lemma again that
    \begin{equation*}
        \E [h_\ell^{(K+1),L}]^2 \le C e^{C\ell/L}.
    \end{equation*}
    Therefore $\displaystyle \sup_{L \ge 1} \sup_{\ell=0,\dotsc,L} \E[h_\ell^{(K+1),L}]^2 < \infty$.
    A similar argument applied to \eqref{eq:g-tau-squared} gives $\displaystyle \sup_{L \ge 1} \sup_{\ell=0,\dotsc,L} \E[g_\ell^{(K+1),L}]^2 < \infty$ and hence the statement also holds for $K+1$.
    This completes the proof by induction.
\end{proof}

Now we couple $h_\ell^L$ and $g_\ell^L$ with $h_t$ and $g_t$ respectively, and state our result on the convergence rate of $1/L$ as $L \to \infty$.
Denote by $L_t:=\lfloor tL \rfloor$, $L_s:=\lfloor \frac{s}L \rfloor$, $\tilde{L}_t:=\lceil tL \rceil$, and $\tilde{L}_s:=\lceil sL \rceil$  
for $s,t \in [0,1]$.
We can write $h_\ell^L = h_{\ell/L}^{(L)}$ and $g_\ell^L = g_{\ell/L}^{(L)}$, where $h_t^{(L)}$ and $g_t^{(L)}$ are continuous interpolations using the same Brownian motions $W_t$ and $B_t$:
\begin{align*}
    d h_t^{(L)} &= b_{L_t}^{(L)}\,dt + \sigma_{L_t}^{(L)}\,dW_t, \\
    d g_t^{(L)} &= c_{\tilde{L}_t}^{(L)}\,dt + D_{\tilde{L}_t}^{(L)}\theta_{\tilde{L}_t}^{(L)}\,dB_t.
\end{align*}
Here $b_{\ell}^{(L)}=(b_{\ell,k}^{(L)})_{k=0}^K$ and $c_\ell^{(L)}=(c_{\ell,k}^{(L)})_{k=0}^K$ are vectors given by
\begin{align}
    b_{\ell,k}^{(L)} & = -\eta_0 \sum_{i=0}^{k-1} \mathcal{L}^{\prime}(i,L) g_{\ell}^{(i),L} \E(\phi(h_{\ell-1}^{(i),L}) \phi(h_{\ell-1}^{(k),L})) \notag \\
    & \quad - \eta_0 \frac{1}{L} \sum_{i=0}^{k-1} \mathcal{L}^{\prime}(i,L) g_{\ell}^{(i),L} 
    \E(\phi(h_{\ell-2}^{(i),L}) \phi(h_{\ell-2}^{(k),L})) 
    \E(\phi'(h_{\ell-1}^{(i),L}) \phi'(h_{\ell-1}^{(k),L})), \notag \\
    c_{\ell,k}^{(L)} &:= -\eta_0 \phi^{\prime}(h_{\ell-1}^{(k),L}) \sum_{i=0}^{k-1} \mathcal{L}^{\prime}(i,L) \phi(h_{\ell-1}^{(i),L}) \E [g_{\ell}^{(i),L} g_{\ell}^{(k),L}] \notag \\
    & \quad -\eta_0 \frac{1}{L} \phi^{\prime}(h_{\ell-1}^{(k),L}) \sum_{i=0}^{k-1} \mathcal{L}'(i,L) \phi(h_{\ell-1}^{(i),L}) \E(g_{\ell+1}^{(i),L} g_{\ell+1}^{(k),L}) \E[\phi'(h_{\ell}^{(i),L}) \phi'(h_{\ell}^{(k),L})], \label{eq:g-c-term}
\end{align}
$D_\ell^{(L)}$ is a diagonal matrix given by
\begin{align*}
    D_\ell^{(L)} &= \text{diag} \{\phi^{\prime}(h_{\ell-1}^{(0),L}),\dotsc,\phi^{\prime}(h_{\ell-1}^{(K),L})\},
\end{align*}
and $\sigma_\ell^{(L)}$ and $\theta_\ell^{(L)}$ are (the Cholesky decomposition) such that
\begin{align*}
    \sigma_\ell^{(L)} (\sigma_\ell^{(L)})^\top = \Sigma_\ell^{(L)} & := \E[(\phi(h_{\ell-1}^{(0),L}),\dotsc,\phi(h_{\ell-1}^{(K),L}))^\top (\phi(h_{\ell-1}^{(0),L}),\dotsc,\phi(h_{\ell-1}^{(K),L}))], \\
    \theta_\ell^{(L)} (\theta_\ell^{(L)})^\top = \Theta_\ell^{(L)} & := \E[(g_\ell^{(0),L},\dotsc,g_\ell^{(K),L})^\top (g_\ell^{(0),L},\dotsc,g_\ell^{(K),L})].
\end{align*}

The following proposition says that the $L^2$ error decays at a rate of $1/L$ for the coupled difference between $h_\ell^L$ (resp.\ $g_\ell^L$), the finite depth process at discrete step $\ell$, and $h_{\ell/L}$ (resp.\ $g_{\ell/L}$), the corresponding infinite-depth process at time $\ell/L$.

\begin{proposition}
    \label{prop:convergence-L}
    Suppose Assumption \ref{assmp:eigenvalue} holds.
    Then for all $L \ge 1$,
    \begin{equation*}
        \sup_{\ell=0,1,\dotsc,L} \E \|h_\ell^L - h_{\ell/L}\|^2 \le C/L, \quad \sup_{\ell=0,1,\dotsc,L} \E \|g_\ell^L - g_{\ell/L}\|^2 \le C/L.
    \end{equation*}
\end{proposition}


\begin{proof}[Proof of Proposition \ref{prop:convergence-L}]
    We first note that the Lipschitz estimates in \eqref{eq:uniqueness-sigma-Lipschitz}, \eqref{eq:uniqueness-sigma21}, and \eqref{eq:uniqueness-sigma22} still hold when comparing $\sigma_{s_L}$ and $\sigma_{L_s}^{(L)}$, thanks to Assumption \ref{assmp:eigenvalue}.    
    We will again prove by induction.
    By Propositions \ref{prop:convergence-L-h0} and \ref{prop:convergence-L-g0}, the statement holds for $K=0$.

    Now suppose the statement holds for $K$, namely  
    \begin{equation*}
        \sup_{\ell=0,\dotsc,L} \sum_{k=0}^K \E [h_\ell^{(k),L} - h_{\ell/L}^{(k)}]^2 \le C/L, \quad \sup_{\ell=0,\dotsc,L} \sum_{k=0}^K \E [g_\ell^{(k),L} - g_{\ell/L}^{(k)}]^2 \le C/L.
    \end{equation*}
    We will show that
    \begin{equation*}
        \sup_{\ell=0,\dotsc,L} \E [h_\ell^{(K+1),L} - h_{\ell/L}^{(K+1)}]^2 \le C/L, \quad \sup_{\ell=0,\dotsc,L} \E [g_\ell^{(K+1),L} - g_{\ell/L}^{(K+1)}]^2 \le C/L.
    \end{equation*}
    Note that
    \begin{align*}
        \E [h_\ell^{(K+1),L} - h_{\ell/L}^{(K+1)}]^2
        & \le 3\E [h_0^{(K+1),L} - h_0^{(K+1)}]^2 + 3\E \left[\int_0^{\ell/L} (b_{L_s,K+1}^{(L)}-b_{s,K+1})\,ds\right]^2 \\
        & \quad + 3\E \left[\int_0^{\ell/L} (\sigma_{21,L_s}^{(L)} - \sigma_{21,s}, \sigma_{22,L_s}^{(L)} - \sigma_{22,s}) \,dW_s\right]^2.
    \end{align*}
    By the induction assumption, we can get
    \begin{equation*}
        \E [h_0^{(K+1),L} - h_0^{(K+1)}]^2 \le C/L, \qquad \E|b_{s_L,K+1}-b_{s,K+1}|^2 \le C/L.
    \end{equation*}
    By Lemmas \ref{lem:h-moment-bound}, \ref{lem:h-fluctuation}, and \ref{lem:h-L-moment-bound} we have
    \begin{align*}
        & \E \left[\int_0^{\ell/L} (b_{L_s,K+1}^{(L)}-b_{s,K+1})\,ds\right]^2 \\
        & \le C \int_0^{\ell/L} \E|b_{L_s,K+1}^{(L)}-b_{s_L,K+1}|^2\,ds + C \int_0^{\ell/L} \E|b_{s_L,K+1}-b_{s,K+1}|^2\,ds \\
        & = \frac{C}{L} \sum_{u=1}^{\ell} \E [b_{u,K+1}^{(L)} - b_{u /L,K+1}]^2 + \frac{C}{L} \\
        & \le \frac{C}{L} \sum_{u=0}^{\ell-1} \E [h_u^{(K+1),L} - h_{u /L}^{(K+1)}]^2 +\frac{C}{L^2}+ \frac{C}{L},
    \end{align*}
    where the faster-vanishing term $C/L^2$ arises from the last term in $b_{u,K+1}^{(L)}$.
    By Lemmas \ref{lem:h-moment-bound}, \ref{lem:h-fluctuation}, and \ref{lem:h-L-moment-bound}, and Lipschitz property of $\sigma$'s, we have
    \begin{align*}
        & \E \left[\int_0^{\ell/L} (\sigma_{21,L_s}^{(L)} - \sigma_{21,s}, \sigma_{22,L_s}^{(L)} - \sigma_{22,s}) \,dW_s\right]^2 \\
        & = \int_0^{\ell/L} [\|\sigma_{21,L_s}^{(L)} - \sigma_{21,s}\|^2 + \|\sigma_{22,L_s}^{(L)} - \sigma_{22,s}\|^2] \,ds \\
        & \le 2\int_0^{\ell/L} [\|\sigma_{21,L_s}^{(L)} - \sigma_{21,s_L}\|^2 + \|\sigma_{22,L_s}^{(L)} - \sigma_{22,s_L}\|^2] \,ds \\
        & \qquad + 2\int_0^{\ell/L} [\|\sigma_{21,s_L} - \sigma_{21,s}\|^2 + \|\sigma_{22,s_L} - \sigma_{22,s}\|^2] \,ds \\
        & \le \frac{C}{L} \sum_{u=0}^{\ell-1} \E [h_u^{(K+1),L} - h_{u /L}^{(K+1)}]^2 + \frac{C}{L}.
    \end{align*}    
    Combining the above estimates gives
    \begin{align*}
        \E [h_\ell^{(K+1),L} - h_{\ell/L}^{(K+1)}]^2 \le \frac{C}{L} \sum_{u=0}^{\ell-1} \E [h_u^{(K+1),L} - h_{u /L}^{(K+1)}]^2 + \frac{C}{L}.
    \end{align*}
    It then follows from discrete Gronwall's lemma that
    \begin{equation*}
        \E [h_\ell^{(K+1),L} - h_{\ell/L}^{(K+1)}]^2 \le \frac{C}{L} e^{C\ell/L}.
    \end{equation*}
    Therefore $\displaystyle \sup_{\ell=0,1,\dotsc,L} \E [h_\ell^{(K+1),L} - h_{\ell/L}^{(K+1)}]^2 \le C/L$.
    A similar argument applied to \eqref{eq:g-tau-squared} and \eqref{eq:g-c-term} gives $\displaystyle \sup_{\ell=0,1,\dotsc,L} \E [g_\ell^{(K+1),L} - g_{\ell/L}^{(K+1)}]^2 \le C/L$ and hence the statement holds for $K+1$.
    This completes the proof by induction.
\end{proof}

Theorem \ref{thm:feature-learning-dynamics} then follows from Proposition \ref{prop:convergence-L}.

\section{Additional Experiments}\label{app:experiments}

\begin{figure}[t]
    \centering
    \begin{minipage}{0.33\linewidth}
        \centering
        \includegraphics[width=\linewidth]{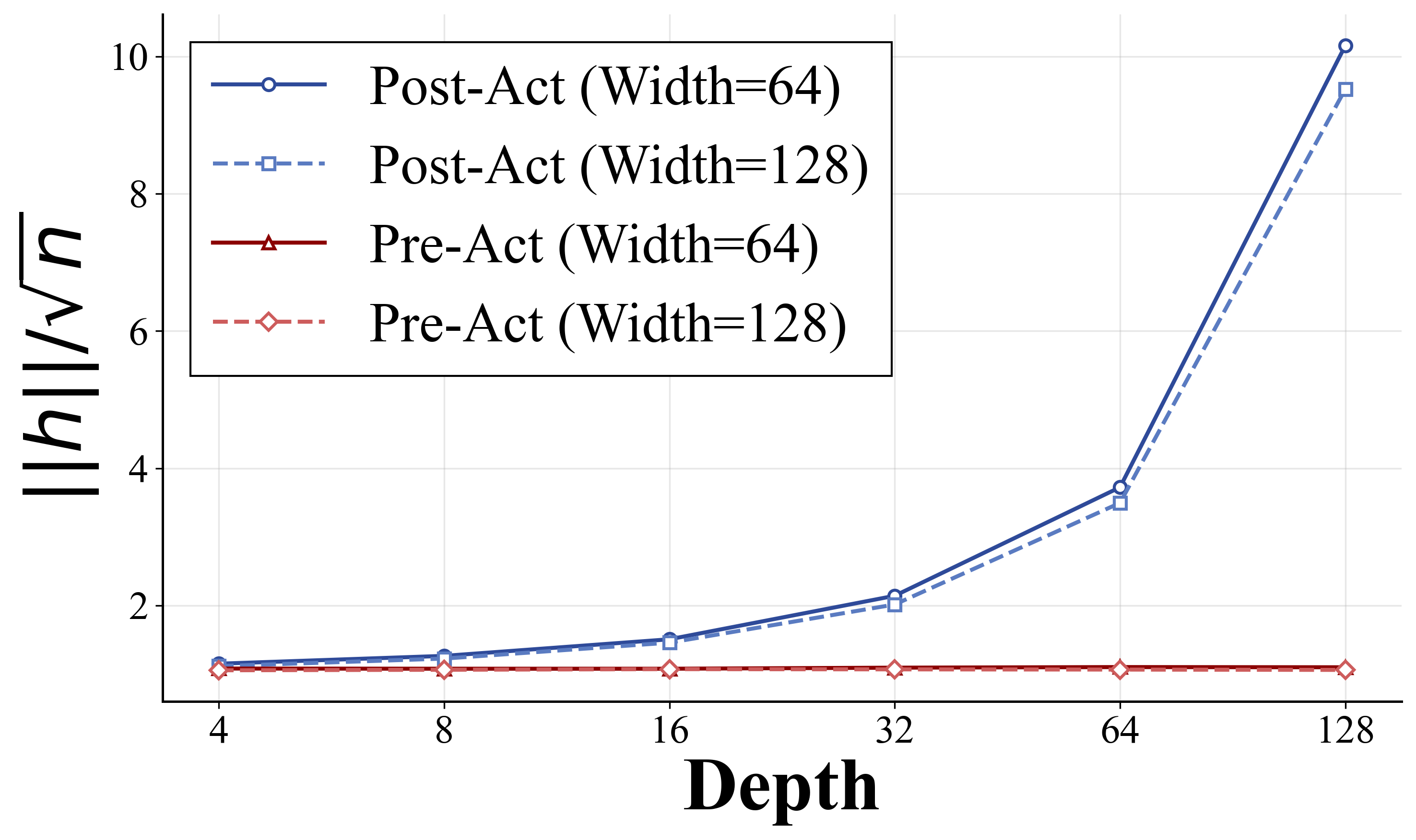}
        \textbf{(a)} Feature Norms vs. Depth
    \end{minipage}%
    \hfill
    \begin{minipage}{0.33\linewidth}
        \centering
        \includegraphics[width=\linewidth]{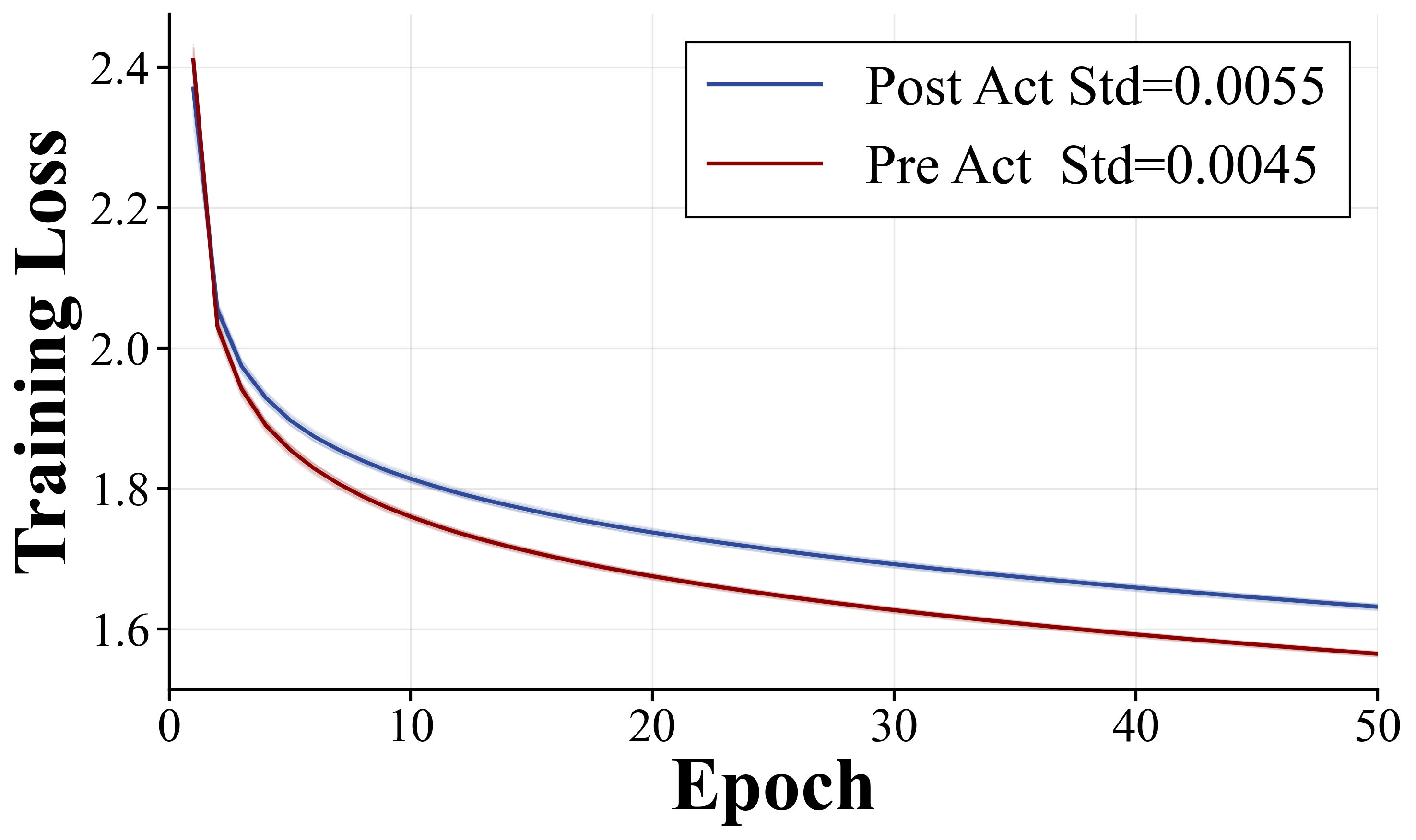}
        (b) Training Loss
    \end{minipage}
    \hfill
    \begin{minipage}{0.33\linewidth}
        \centering
        \includegraphics[width=\linewidth]{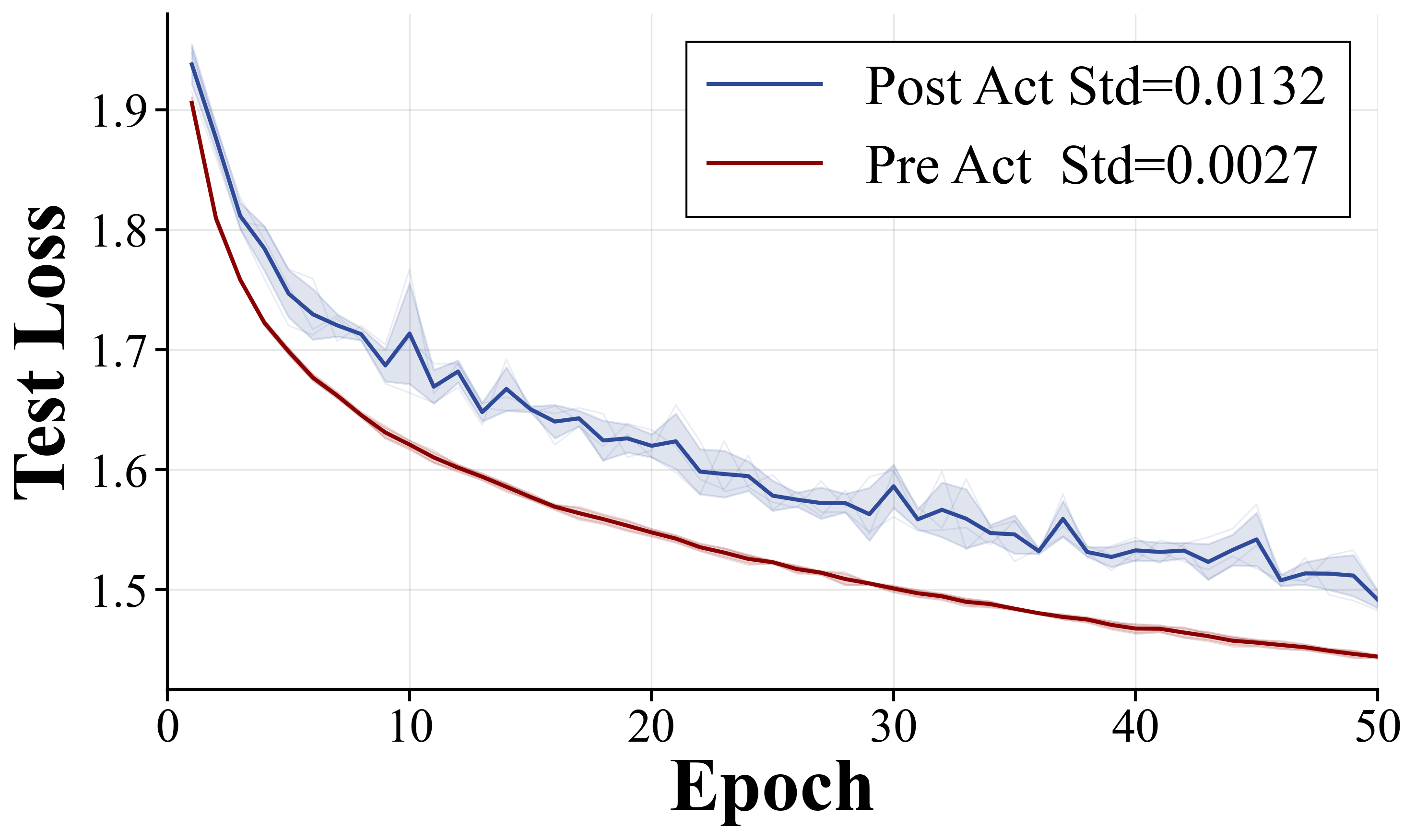}
        (c) Test Loss
    \end{minipage}
    \caption{\textbf{Pre- vs. post-activation ResNets under depth-$\mu$P.}
    (a) Pre-activation keeps feature norms stable across depth, whereas post-activation exhibits rapid growth.
    (b)--(c) For depth-64, width-128 ReLU ResNets trained on CIFAR-10 with SGD and batch size 128, pre-activation yields faster convergence, lower test loss, and smaller variance across runs.}
    \label{fig:pre-act-vs-post-act}
\end{figure}

In this section, we provide supplementary experiments that complement the results reported in the main paper. Unless otherwise specified, our default setup trains the ResNet in \eqref{eq:resnet} using mini-batch SGD with batch size 128 and effective learning rate $\eta_c=0.01$, across widths $n\in\{128,256\}$ and depths $L\in\{2,4,8,16\}$.

\paragraph{Pre- vs. Post-activation ResNets.}
We provide additional experiments supporting the architectural choice used in the main paper. In particular, we compare pre-activation and post-activation ResNets under depth-$\mu$P scaling. As suggested by Proposition~\ref{prop:post-act}, the pre-activation design used in our theoretical analysis is more stable at large depth, whereas post-activation ResNets can exhibit rapid feature-norm growth when the activation satisfies the positive dominance condition. Figure~\ref{fig:pre-act-vs-post-act} supports this behavior and shows that pre-activation also leads to more stable train/test loss curves.

\paragraph{GIA Restoration and Dynamics Alignment.}
\begin{figure*}[t]
    \centering

    \begin{minipage}[t]{0.32\textwidth}
        \centering
        \includegraphics[width=\linewidth]{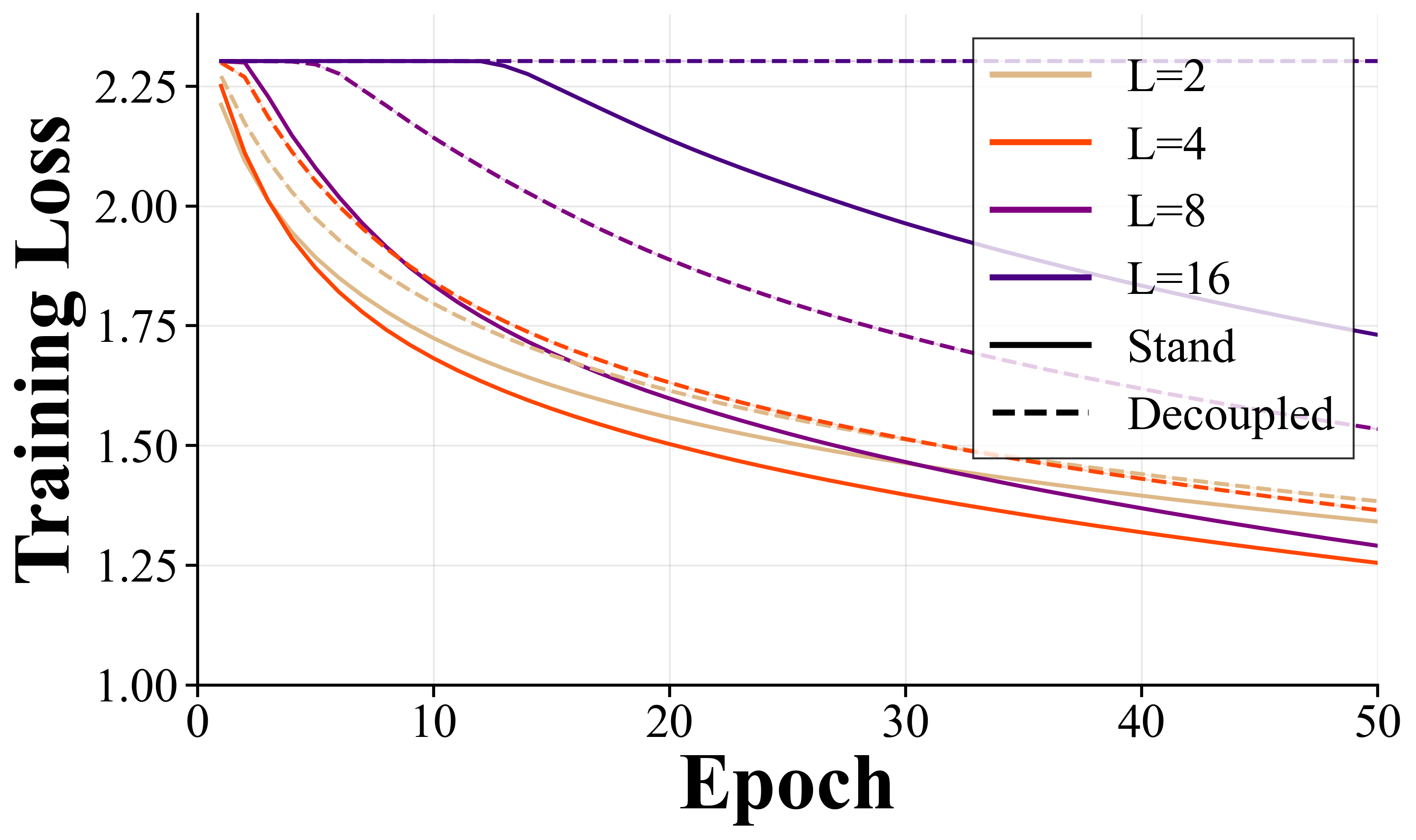}
        \vspace{0.3em}

        \includegraphics[width=\linewidth]{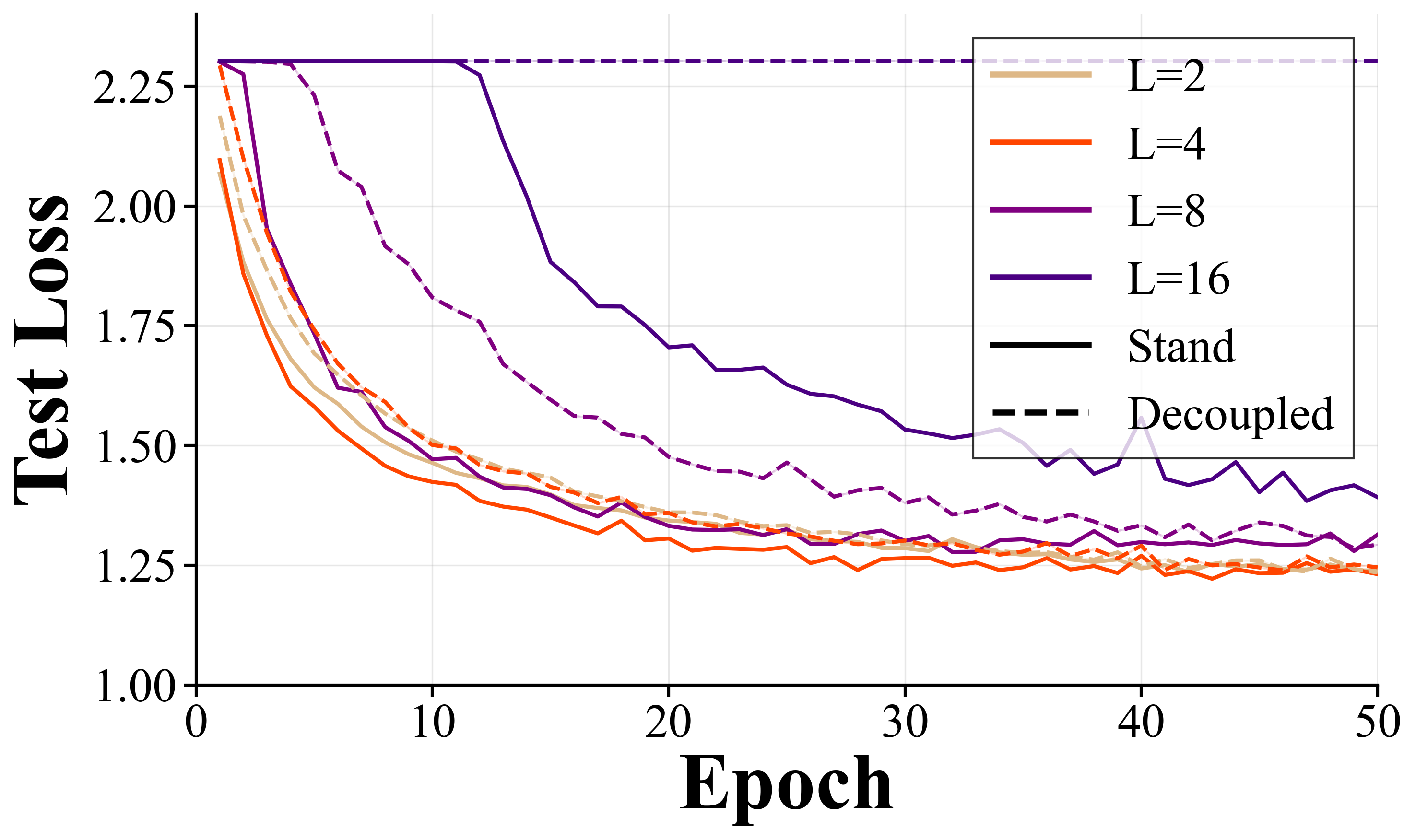}

        \vspace{0.3em}
        \textbf{(a) Vanilla DNN}
    \end{minipage}
    \hfill
    \begin{minipage}[t]{0.32\textwidth}
        \centering
        \includegraphics[width=\linewidth]{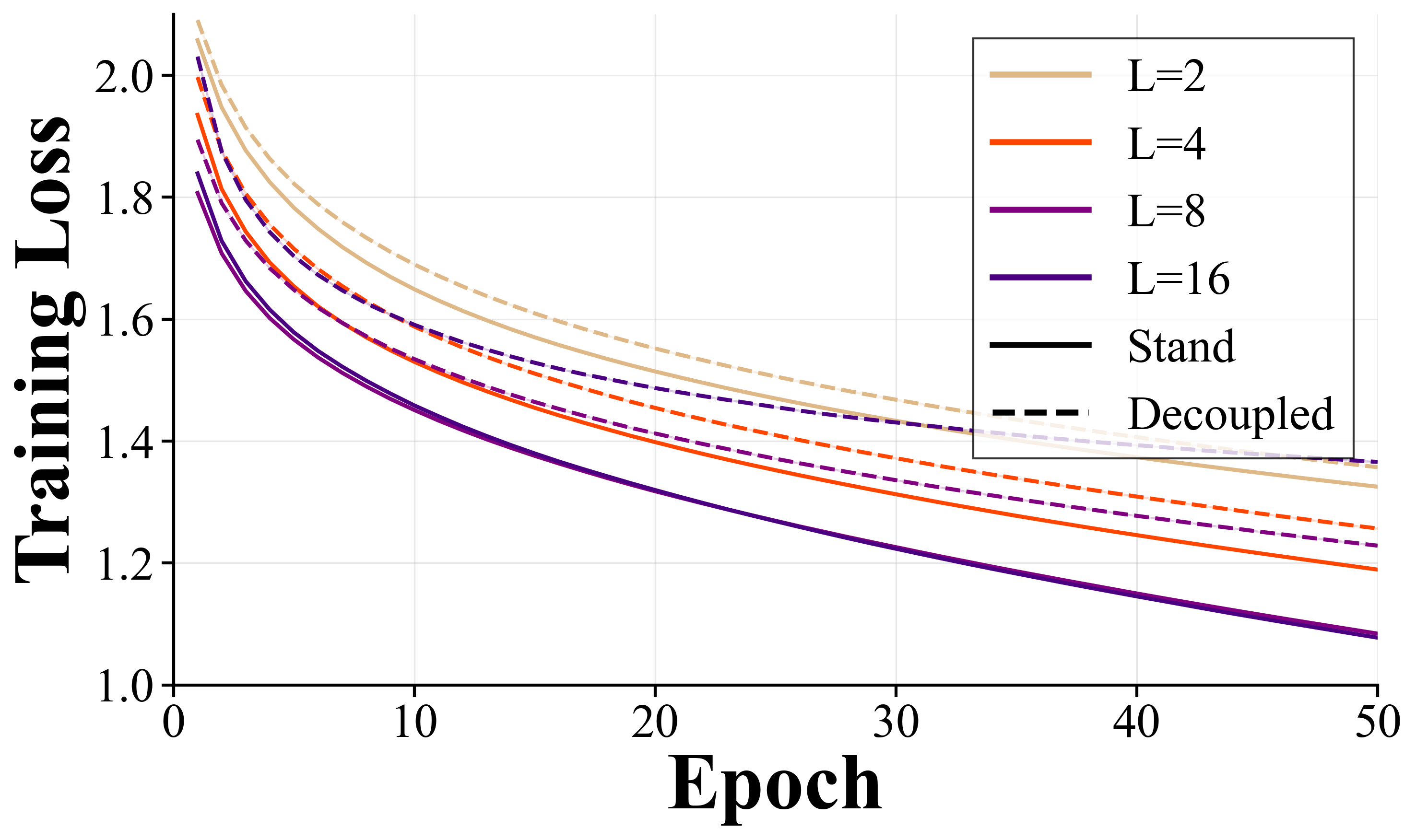}
        \vspace{0.3em}

        \includegraphics[width=\linewidth]{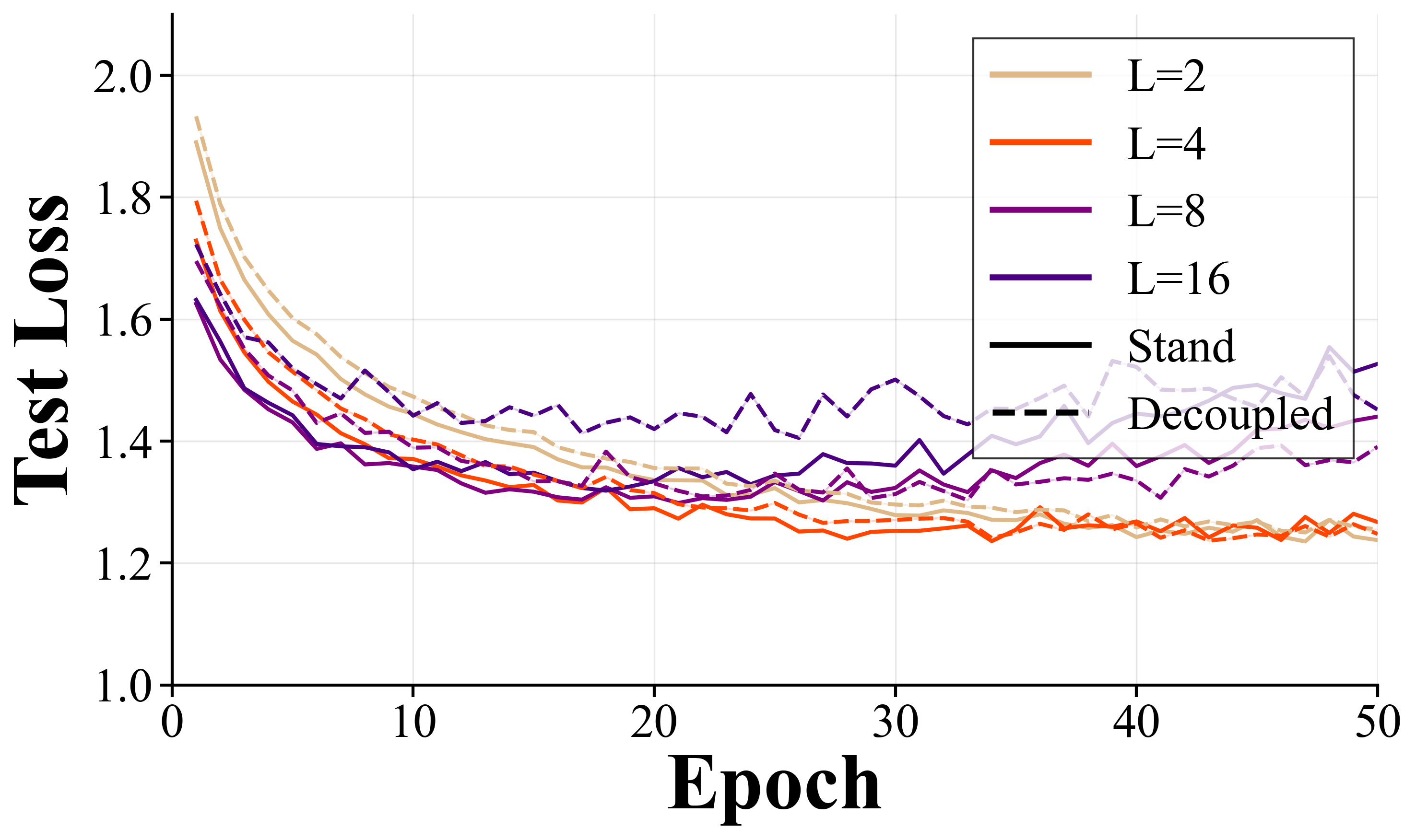}

        \vspace{0.3em}
        \textbf{(b) ResNet under $\mu$P}
    \end{minipage}
    \hfill
    \begin{minipage}[t]{0.32\textwidth}
        \centering
        \includegraphics[width=\linewidth]{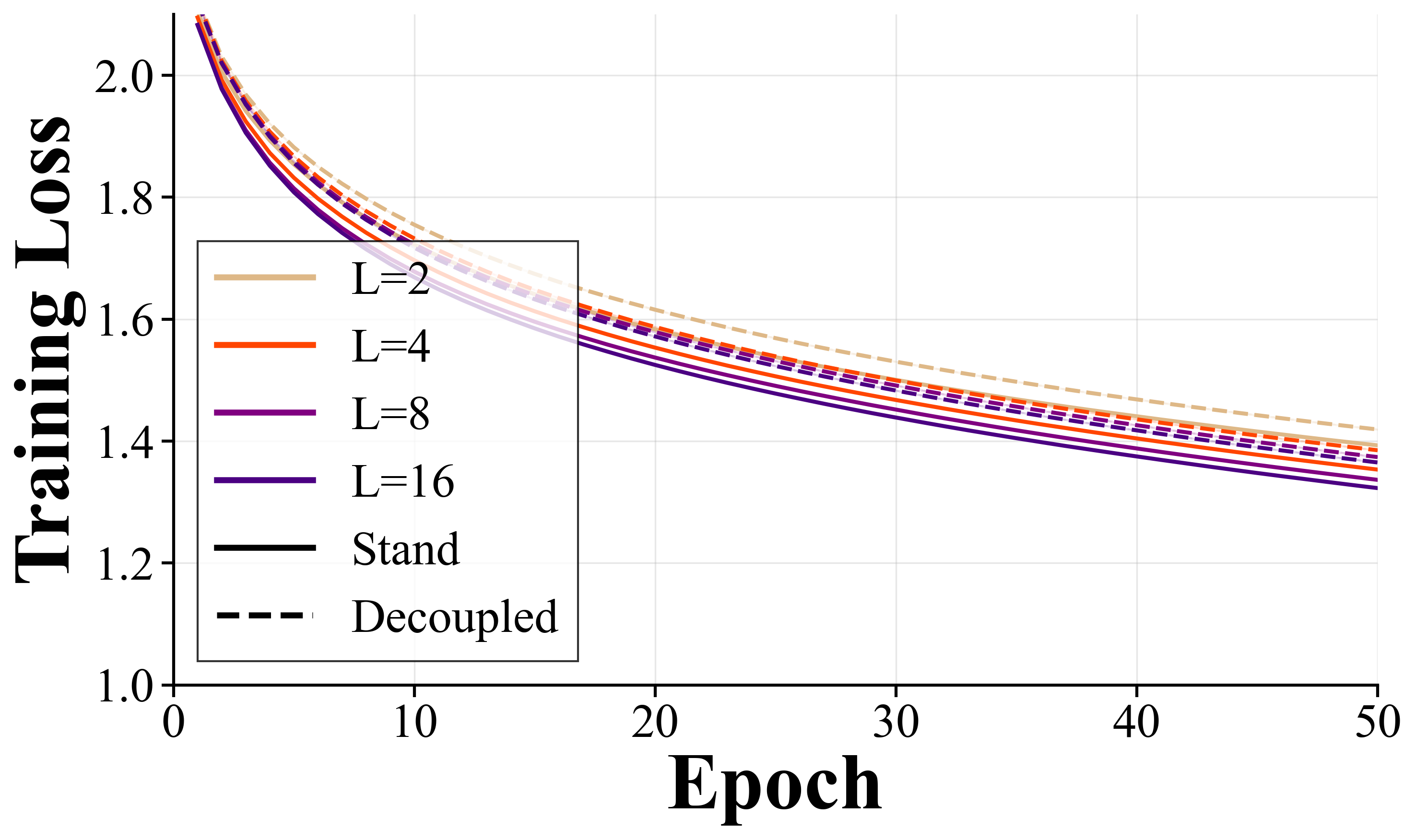}
        \vspace{0.3em}

        \includegraphics[width=\linewidth]{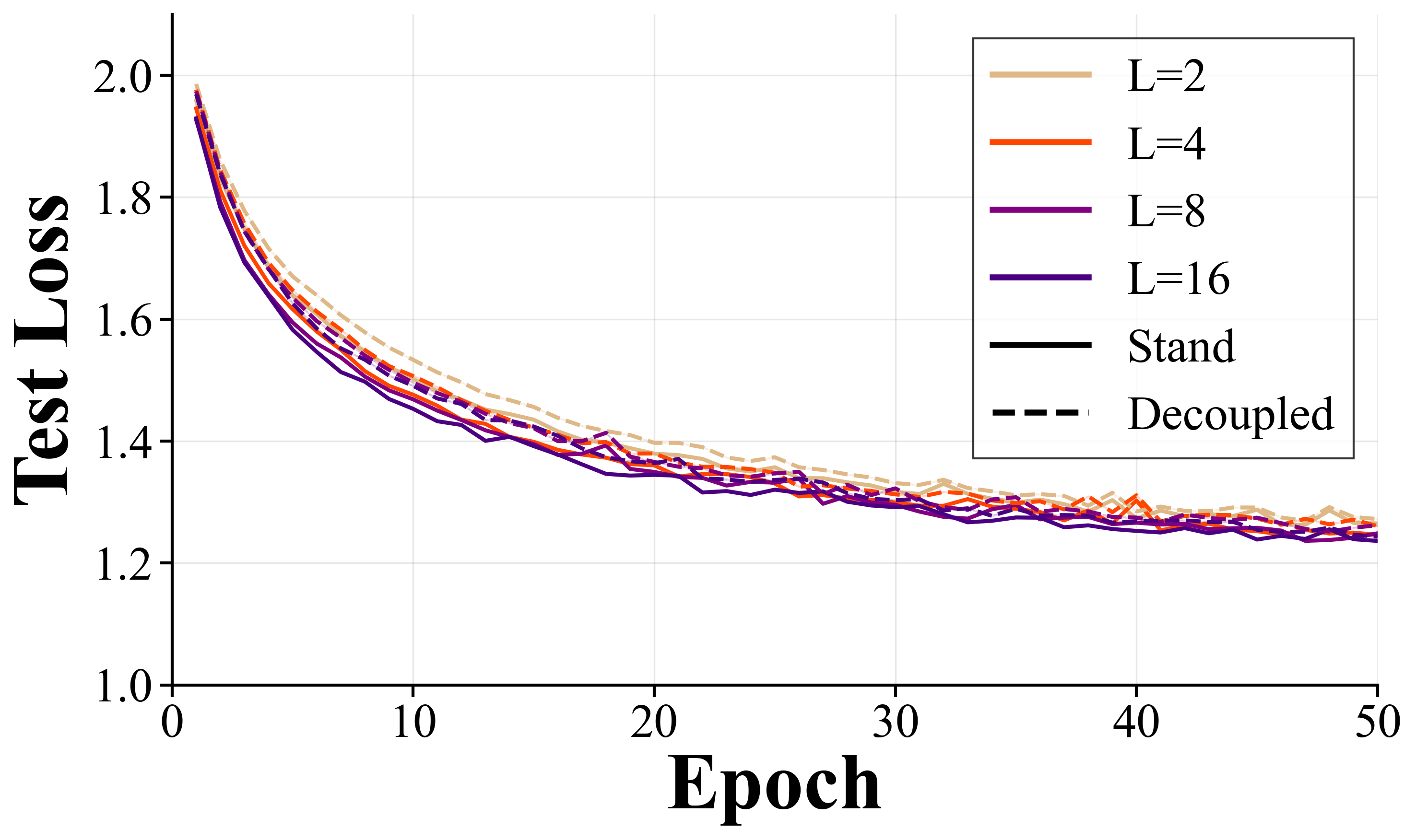}

        \vspace{0.3em}
        \textbf{(c) ResNet under depth-$\mu$P}
    \end{minipage}

    \caption{
    \textbf{Empirical evaluation of GIA restoration at width 256.}
    We repeat Figure~\ref{fig:gia_alignment} at width 256, again comparing standard training with a decoupled backward pass.
    \textbf{(a)} Vanilla DNNs are more stable than at width 128, but the two trajectories remain misaligned.
    \textbf{(b)} ResNets under $\mu$P also exhibit reduced instability without clear alignment.
    \textbf{(c)} Depth-$\mu$P further smooths the dynamics and strengthens the agreement between standard and decoupled trajectories, reinforcing empirical GIA restoration.
    }
    \label{fig:gia_alignment2}
\end{figure*}

We empirically evaluate the GIA restoration predicted by Theorem~\ref{thm:feature-learning-dynamics} and Corollary~\ref{cor:nfd-depth-rate}. For each architecture---vanilla DNN, $\mu$P-ResNet, and depth-$\mu$P ResNet---we compare \emph{standard} training, where the backward pass reuses the forward weights, with a \emph{decoupled} variant, where the backward pass uses an i.i.d.\ copy of the forward weights. Figure~\ref{fig:gia_alignment} reports training and test losses at width~128 across increasing depths. As depth grows, vanilla DNNs suffer from vanishing gradients and the standard and decoupled trajectories remain misaligned; $\mu$P-ResNets also fail to align and become unstable at large depth. In contrast, depth-$\mu$P improves both training and test performance with depth, and the standard and decoupled trajectories move closer to each other, supporting the restoration of GIA in the large-depth regime. Figure~\ref{fig:gia_alignment2} shows the corresponding width-256 results: larger width reduces some instability in vanilla DNNs and $\mu$P-ResNets but does not resolve their misalignment, while further strengthening the trajectory alignment and performance gains under depth-$\mu$P.



\end{document}